%% file: aistats.tex
\documentclass[twoside]{article}

\usepackage[preprint]{aistats2027}
\usepackage[utf8]{inputenc} % allow utf-8 input
\usepackage[T1]{fontenc}    % use 8-bit T1 fonts
\usepackage{hyperref}       % hyperlinks
\usepackage{url}            % simple URL typesetting
\usepackage{booktabs}       % professional-quality tables
\usepackage{amsfonts}       % blackboard math symbols
\usepackage{nicefrac}       % compact symbols for 1/2, etc.
\usepackage{xcolor}         % colors
\usepackage{amsmath,amssymb,amsthm}
\usepackage{mathtools}
\usepackage{bm}
\usepackage{float}
\usepackage{graphicx}
\usepackage{subcaption}
\usepackage{booktabs} % for professional tables
\usepackage{algorithm}
\usepackage{algpseudocode}
\makeatletter
\@addtoreset{ALG@line}{algorithm}
\makeatother
\usepackage{hyperref}
\usepackage{xcolor}
\usepackage{cleveref}
\usepackage{optidef}
\DeclareMathOperator{\prox}{prox}
\newcommand{\R}{\mathbb{R}}
\newcommand{\E}{\mathbb{E}}
\newcommand{\Et}{\mathbb{E}_t}
\newcommand{\inprod}[2]{\langle #1,#2\rangle}
\newcommand{\dist}{\mathrm{dist}}

\newtheorem{theorem}{Theorem}
\newtheorem{proposition}{Proposition}
\newtheorem{lemma}{Lemma}
\newtheorem{corollary}{Corollary}
\newtheorem{remark}{Remark}
\newtheorem{definition}{Definition}
\newtheorem{assumption}{Assumption}
\crefname{assumption}{assumption}{assumptions}
\Crefname{assumption}{Assumption}{Assumptions}

\usepackage{booktabs}
\usepackage{amssymb}  % for \checkmark
\usepackage{pifont}

\usepackage{amsmath}
\usepackage[numbers,sort&compress]{natbib}

\usepackage{amssymb}
\usepackage{mathtools}
\usepackage{amsthm}
\newcommand{\bx}{\mathbf{x}}
\newcommand{\by}{\mathbf{y}}
\newcommand{\bz}{\mathbf{z}}

\newcommand{\bA}{\mathbf{A}}
\newcommand{\bP}{\mathbf{P}}
\newcommand{\bI}{\mathbf{I}}
\newcommand{\bW}{\mathbf{W}}
\newcommand{\bL}{\mathbf{L}}
\newcommand{\bb}{\mathbf{b}}
\newcommand{\ba}{\mathbf{a}}
\newcommand{\br}{\mathbf{r}}
\newcommand{\bv}{\mathbf{v}}
\newcommand{\bg}{\mathbf{g}}
\newcommand{\bw}{\mathbf{w}}
\newcommand{\blambda}{\boldsymbol{\lambda}}
\newcommand{\bmu}{\boldsymbol{\mu}}
\newcommand{\beps}{\boldsymbol{\varepsilon}}
\DeclareMathOperator*{\argmin}{argmin}
\newcommand{\myheight}{4.5cm}
\newcommand{\mywidth}{0.95\linewidth}

\begin{document}

\runningtitle{Asynchronous Gossip Algorithm for Robust and Non-Smooth Convex Decentralized Learning}


\twocolumn[

\aistatstitle{Fast and Efficient Asynchronous Gossip Algorithm for \\ Robust and Non-Smooth Convex Decentralized Learning}

\aistatsauthor{ Anna van Elst \And Olivier Fercoq
\And Igor Colin \And  Stephan Cl\'emen\c{c}on }

\aistatsaddress{LTCI, T\'el\'ecom Paris, Institut Polytechnique de Paris} ]

\begin{abstract}
  Asynchronous primal-dual methods for decentralized non-smooth convex optimization often require each node to maintain $\mathcal{O}(d)$ auxiliary variables, where $d$ is its degree. This dependence on degree increases memory requirements and can amplify the effects of stale information, especially in dense networks. Motivated by the challenge of frugal memory management in decentralized learning, we introduce Goal-PD, an asynchronous gossip-based primal-dual algorithm that maintains only two variables per node, regardless of the node's degree. We establish almost-sure convergence of Goal-PD to a minimizer of the underlying optimization problem, and prove linear convergence when the objective functions are piecewise linear-quadratic. For decentralized mean estimation, we show that pairwise averaging is a special case of Goal-PD, which establishes a direct link between the proposed primal-dual framework and classical gossip. Experiments on synthetic and real datasets over various network topologies, with non-smooth objectives including median estimation, show that Goal-PD converges faster than existing asynchronous baselines while requiring significantly less memory by design.
\end{abstract}

\input{new_files/1_introduction}

\input{new_files/2_background}

\input{new_files/3_algorithm}

\input{new_files/5_exp}

\input{new_files/6_conclusion}

\newpage

\subsection*{AI use statement}

In this work, we used generative AI tools for implementing methods. We have not used generative AI tools for developing the conceptual frameworks, formulating mathematical claims, providing critical ingredients for proving mathematical claims, for assisting in the writing of proofs, for proposing or refining hypotheses and the rest of the required disclosure tasks are not applicable to this work. Additionally, we used generative AI tools for creating or editing software code, for searching for information, for drafting part of the research paper, for text editing, polishing, and formatting. We have reviewed all AI-assisted work. LLM-generated code was verified and tested for correctness by one author and any LLM-generated text has been checked line by line by one author. We take responsibility for the final content of this work, including text, claims, or artifacts produced with the aid of generative AI.

%\subsubsection*{Acknowledgements}
%All acknowledgements go at the end of the paper, including thanks to reviewers who gave useful comments, to colleagues who contributed to the ideas, and to funding agencies and corporate sponsors that provided financial support.
%To preserve the anonymity, please include acknowledgements \emph{only} in the camera-ready papers. The acknowledgements do not count against the 9-page page limit in the camera-ready.

\subsubsection*{References}

\bibliography{references}

%%%%%%%%%%%%%%%%%%%%%%%%%%%%%%%%%%%%%%%%%%%%%%%%%%%%%%%%%%%%
% \section*{Checklist}
% \input{checklist}

\clearpage
\appendix
\thispagestyle{empty}

% Supplementary material: To improve readability, you must use a single-column format for the supplementary material.
\onecolumn
\aistatstitle{Supplementary Materials}

Appendix~\ref{app:proximal-related} details the baseline algorithms for decentralized non-smooth convex optimization discussed in Section~\ref{sec:preliminaries}. Appendix~\ref{app:gossip-purecd} establishes the relationship between Goal-PD and \textsc{Pure-CD}. Appendix~\ref{app:lyapunov-proof} gives the proof of Proposition~\ref{lem:lyapunov}, the one-iteration inequality underlying our analysis. Appendix~\ref{app:as-convergence-proof} contains the proof of Theorem~\ref{thm:as-convergence}, the almost-sure convergence of Goal-PD. Appendix~\ref{app:linear-convergence-proof} provides the proof of Theorem~\ref{thm:gamma1-gossip}, our linear convergence guarantee for piecewise linear-quadratic objectives. Appendix~\ref{app:pairwise-averaging} proves Proposition~\ref{prop:pairwise-averaging}, showing that Goal-PD recovers the classical pairwise-averaging gossip algorithm for mean estimation. Appendix~\ref{app:additional-experiments} provides additional experimental details and results for the synthetic data experiments of Section~\ref{sec:synthetic}, including a description of the network topologies used, experiments on larger networks, and ablations on the margin parameter $\gamma$ and the penalty parameter $\rho$. Appendix~\ref{app:wine-experiments} provides additional experimental details and results for the Wine Quality dataset experiments of Section~\ref{sec:wine}.

\input{new_files/A_appendix}
\input{new_files/B_appendix}
\input{new_files/C_appendix}
\input{new_files/D_appendix}
\input{new_files/E_appendix}
\input{new_files/F_appendix}
\input{new_files/G_appendix}
\input{new_files/H_appendix}

\end{document}

%% file: new_files/1_introduction.tex
\section{Introduction}

Decentralized learning is increasingly important for large-scale machine learning, given the recent development of edge AI systems in particular. In many applications, data are generated and processed directly on connected devices with limited memory resources such as smart sensors and cell phones. Centralized architectures can be undesirable or impractical due to constraints related to communication, energy consumption, privacy, and scalability. Gossip-based methods are particularly attractive in this setting because they rely on lightweight peer-to-peer communication, require no global coordination, and naturally scale with network size~\cite{boyd2006randomized, boyd2011distributed}. However, standard gossip methods generally rest on averaging and are sensitive to corrupted data: even a small fraction of faulty measurements can substantially distort the resulting estimate~\cite{van2025robust, ayadi2017outlier}. The median, by contrast, is a classical robust statistic that is more resistant to outliers and extreme values, making it a natural alternative when seeking robustness against corrupted data~\cite{huber2011robust}.

Medians and, more generally, quantiles can be classically viewed as $M$-estimators, i.e., risk minimizers, when the risk is defined precisely by the (piecewise linear) pinball loss, yielding non-smooth convex optimization problems. Primal-dual methods, including ADMM~\cite{boyd2011distributed}, are well suited to this setting because they separate local non-smooth optimization from consensus constraints. While synchronous methods require network-wide coordination~\cite{wei2012distributed}, asynchronous variants typically introduce \textit{edge-specific variables} to handle delayed communications~\cite{iutzeler2017distributed, bianchi2015coordinate}. Existing methods maintain cached variables for each incident edge. Thus, a node of degree $d$ stores $\mathcal{O}(d)$ variables. Furthermore, these variables are updated only when their corresponding neighbors communicate, which means that the amount of potentially obsolete information increases with the node’s degree, making highly connected networks particularly vulnerable to both memory overhead and outdated information.%Thus, increasing network connectivity leads to both higher memory requirements and greater exposure to stale information.

In this paper, we introduce \textbf{Goal-PD}, a novel \textbf{Go}ssip-based \textbf{a}synchronous \textbf{l}ite \textbf{P}rimal-\textbf{D}ual algorithm, whose memory requirements do not depend on the degree, in stark contrast to existing methods.
The key idea is to replace the cached primal variables for each edge with the freshly received neighbor iterates and to aggregate the dual variables at each node, thereby eliminating the edge-specific dual state while preserving the KKT optimality conditions of the underlying optimization problem. Together, these design choices eliminate the degree-dependent auxiliary state of existing primal-dual methods, resulting in a lightweight scheme with only two variables per node, regardless of the network’s degree distribution. We further show that, through careful parameter selection and a series of non-obvious algebraic simplifications, it is possible to establish a connection between Pure-CD~\cite{alacaoglu2020random} and Goal-PD. This bridge allows us to reuse part of the existing analysis to prove almost-sure convergence of Goal-PD, while requiring a nontrivial extension of the original proof. Notably, this connection enables us to derive, to the best of our knowledge, the first linear convergence rate in the gossip literature for piecewise linear-quadratic functions.

Goal-PD applies to general non-smooth convex objectives of arbitrary dimension, and our analysis covers this general setting. For clarity, we use one-dimensional median estimation with the \textit{pinball loss} as our running example, a problem that is very important in practice and typical of many other robust estimation and learning problems. Our main contributions are as follows:

\noindent$\bullet$~\textbf{Degree-independent state.} Goal-PD requires only two variables per node, regardless of its degree, which eliminates the $\mathcal{O}(d)$ auxiliary state and reduces exposure to outdated information from neighbors.

\smallskip
\noindent$\bullet$~\textbf{Convergence guarantees.} We prove almost-sure convergence to a minimizer of the underlying optimization problem for general non-smooth convex functions, and linear convergence when the functions are piecewise linear-quadratic, a class that includes the pinball loss used for median estimation.

\smallskip
\noindent$\bullet$~\textbf{Gossip connection.} We show that pairwise averaging for decentralized mean estimation is recovered as a special case, linking classical gossip with the proposed primal-dual framework.

\smallskip
\noindent$\bullet$~\textbf{Empirical evidence.} We evaluate Goal-PD on synthetic and real-world datasets with various network topologies, using non-smooth objectives, including the estimation of the median and quantiles. Our experiments show faster convergence than existing asynchronous baseline methods, while requiring significantly less memory.

\smallskip
This article is structured as follows. Section~\ref{sec:preliminaries} reviews related work on non-smooth decentralized optimization and quantile estimation. In Section~\ref{sec:goalpd}, we present the Goal-PD method and provide our convergence analysis. Finally, Section~\ref{sec:experiments} empirically validates Goal-PD on diverse non-smooth convex tasks.

%% file: new_files/2_background.tex
\section{Preliminaries}
\label{sec:preliminaries}

This section reviews decentralized optimization methods from the literature, against which Goal-PD compares favorably (Section~\ref{sec:experiments}), and revisits the optimization-based characterization of median and quantile estimation via the pinball loss.

\subsection{Decentralized Convex Optimization}
\label{sec:back-opt}
We begin by reviewing the main methods of decentralized optimization for non-smooth convex objectives. Here and throughout, we denote scalars by lowercase letters \(x \in \mathbb{R}\), vectors by boldface lowercase letters \(\mathbf{x} \in \mathbb{R}^n\), and matrices by boldface uppercase letters \(\mathbf{X} \in \mathbb{R}^{m \times n}\). We set $[n]:=\{ 1, \dots, n \}$ and write $\{\mathbf{e}_k:\; k\in [n]\}$ for $\mathbb{R}^n$'s canonical basis, \(\mathbb{I}\{{\mathcal{A}}\}\in\{0,1\}\) for the indicator function of any event \(\mathcal{A}\), \(\mathbf{M}^{\top}\) for the transpose of any matrix \(\mathbf{M}\), \(|F|\) for the cardinality of any finite set $F$, \(\mathbf{I}_n\) for the identity matrix in \(\mathbb{R}^{n \times n}\), \(\mathbf{1}_n\) for the vector in \(\mathbb{R}^n\) whose coordinates are all equal to $1$. We write \(\langle \cdot, \cdot \rangle\) and \(\|\cdot\|\) for the Euclidean inner product and norm on vectors and, when applied to matrices, for the Frobenius inner product and norm. The set of edges defining the communication network is denoted by $E\subset [n]^2$.

\textbf{Problem Formulation.} Consider a distributed optimization problem on a connected network of $n\geq 1$ agents, where each agent $i\in [n]$ has a closed, proper, convex, and non-smooth local objective $f_i:\mathbb{R}^p \to \mathbb{R}$, where $p\geq 1$ is the dimension of the problem. The agents’ goal is to collaboratively solve the following minimization problem:
\begin{equation}
\label{eq:1}
\argmin_{\bx \in  \mathbb{R}^p}  \sum_{i=1}^n f_i(\bx)   \enspace.
\end{equation}
Gossip methods address this problem by alternating local updates with neighbor averaging \cite{nedic2009distributed, duchi2011dual}. 
These approaches are typically divided into two broad families, depending on how they handle the non-smoothness of the objectives: primal-based subgradient methods and proximal-based splitting methods.

\begin{table}[t]
\centering
\caption{Comparison of Goal-PD with state-of-the-art gossip methods for non-smooth convex decentralized optimization. Goal-PD's memory requirement is independent of the node degree $d$. All methods admit almost-sure convergence guarantees. The linear convergence rate holds for Goal-PD when the objective functions are PLQ (piecewise linear-quadratic). }
\label{tab:comparison}
\setlength{\tabcolsep}{2pt}
\begin{tabular}{lccc}
\toprule
\textbf{Method} & \textbf{Memory} & \textbf{ Conv.} & \textbf{Rate} \\
\midrule
SubGD~\cite{nedic2009distributed} & $\mathcal{O}(1)$ & \checkmark & Sublinear \\
AsyncADMM~\cite{iutzeler2013asynchronous} & $\mathcal{O}(d)$ & \checkmark & --- \\
DAPD~\cite{bianchi2015coordinate} & $\mathcal{O}(d)$ & \checkmark & --- \\
\textbf{Goal-PD (ours)} & $\mathcal{O}(1)$ & \checkmark & Linear (PLQ) \\
\bottomrule
\end{tabular}
\end{table}

\textbf{Subgradient Descent.} Each node performs local (sub)gradient updates and averages its estimate through pairwise communication. The method converges at a rate $\mathcal{O}(1/\sqrt{t})$, where $t\geq 1$ is the number of iterations \cite{nedic2009distributed}. However, it is sensitive to step-size selection, especially for non-smooth objectives, and often converges slowly in practice.

\textbf{Proximal Methods.} Proximal methods are well-suited for non-smooth objectives. A canonical example is the proximal point algorithm (PPA), which can be viewed as an implicit gradient descent.  Distributed variants such as P-EXTRA have been proposed in \cite{shi2015proximal} for synchronous decentralized settings, but they do not naturally extend to asynchronous settings due to the non-linearity of the proximal operator. In contrast, decentralized primal-dual methods are better suited to this setting, as they decouple local optimization from consensus constraints and hence support asynchronous gossip implementations; see \cite{shi2014linear, yang2022survey, boyd2011distributed}.

\textbf{Primal-Dual Methods.} ADMM can be used to derive a synchronous decentralized gossip algorithm for non-smooth convex objectives, as proposed in \cite{shi2014linear, iutzeler2013asynchronous}, and was proved to converge at a rate of $\mathcal{O}(1/t)$ for general non-smooth convex functions \cite{he20121}. However, asynchronous methods are generally more practical, as they avoid global coordination, tolerate communication delays, and scale well with large networks. An asynchronous decentralized ADMM, which we refer to as AsyncADMM, was proposed in \cite{iutzeler2013asynchronous}, using randomized coordinate descent. Its main limitation is that each node $k\in [n]$ must store $2d_k$ auxiliary variables associated with its incident edges, where $d_k$ is the degree of node $k$: this memory requirement becomes significant when the graph is dense or when nodes have limited memory capacity. Another asynchronous ADMM method has been proposed in \cite{wei2013o1kconvergenceasynchronousdistributed}. However, when applied to our problem, its edge-based reformulation does not converge. In addition, an asynchronous decentralized primal-dual algorithm, referred to as DAPD, has been proposed in \cite{bianchi2015coordinate}, using randomized coordinate descent on the deterministic algorithm of \citet{vu2013splitting} and \citet{condat2013primal}. DAPD was shown to converge almost surely, but it also incurs significant memory overhead at each node as it requires edge-specific variables. We describe these baseline algorithms in more detail in Appendix~\ref{app:proximal-related}. Another primal-dual method of particular interest to us is Pure-CD~\cite{alacaoglu2020random}, which was not designed for the gossip framework; it is asynchronous and offers almost-sure convergence, as well as a linear convergence rate for piecewise linear-quadratic functions.

\subsection{Median and Quantile Estimation}

\label{sec:quantile}

Let us now review the optimization-based characterization of the median and quantiles using the pinball loss \cite{wasserman2013all}, which defines the M-estimation problem that will serve as our running example of a non-smooth convex optimization problem throughout this article.

For $\alpha \in (0,1)$, the \textit{pinball loss} is defined as $L_{\alpha}(z) := (\alpha - \mathbb{I}\{{z \leq 0\}}) z$, where $\alpha$ is the target quantile level and $\alpha=1/2$ corresponds to median estimation. Taking $f_i(x) = L_{\alpha}(a_i - x)$ in Problem~\eqref{eq:1}, where $a_1, a_2, \ldots, a_n$ are the observed data points, recovers the task of estimating their $\alpha$-quantile. We adopt the notation and results from \cite{iutzeler2017distributed}, defining $f_a^\beta: x \mapsto L_{\alpha}(a - x)/(1-\alpha)$ with $\beta = \alpha/(1-\alpha)$. For $a \in \mathbb{R}$ and $\beta>0$, $f_a^\beta$ is convex and continuous but non-differentiable. Recall that the proximal operator of a function $f$ with parameter $\gamma>0$ is $\prox_{\gamma,f}(z) := \argmin_{w \in \mathbb{R}}\left\{f(w)+\|w-z\|^2/(2\gamma)\right\}$. Then $f_a^\beta$ admits the following closed-form expression:
\[
\prox_{\gamma,f_a^\beta}(z) = \begin{cases}z+\gamma \beta, & \text { if } z<a-\gamma \beta \\
z-\gamma, & \text { if } z>a+\gamma \\
a, & \text { if } z \in[a-\gamma \beta , a+\gamma] \, \end{cases}
\enspace,
\]
which naturally leads us to employ proximal methods.

%% file: new_files/3_algorithm.tex
\section{Goal-PD}
\label{sec:goalpd}

This section presents the problem formulation and framework, describes the proposed asynchronous gossip algorithm for non-smooth convex functions, and establishes its convergence guarantees.

\subsection{Problem Formulation and Framework}
\label{sec:setup}

Let $G=([n],E)$ be a connected, undirected communication graph with $n$ nodes and $m=|E|$ edges. Each node $i\in[n]$ holds a private, closed, proper, convex, non-smooth local objective $f_i:\mathbb{R}^p\to\mathbb{R}$, and we want to solve, using communication along edges,
\[
\min_{\bx_1,\dots,\bx_n\in\mathbb{R}^p}\ \sum_{i=1}^n f_i(\bx_i)
\qquad\text{s.t.}\qquad \bA\bx=0 \enspace,
\]
where $\bA\in\mathbb{R}^{m\times n}$ is the oriented incidence matrix of $G$: fixing an arbitrary orientation of each edge $e\in E$, writing $e=(k,l)$ for its head/tail nodes $k,l\in[n]$, $\bA$ is defined by $(\bA\bx)_e = \bx_k-\bx_l$. This recovers Problem~\eqref{eq:1} at consensus. Introducing an edge-indexed dual variable $\blambda\in\mathbb{R}^{m\times p}$ for the constraint $\bA\bx=0$, this problem admits the saddle-point formulation
\begin{equation}
\label{eq:2}
\min_{\bx\in\mathbb{R}^{n\times p}}\ \max_{\blambda\in\mathbb{R}^{m\times p}}\ \sum_{i=1}^n f_i(\bx_i) - \inprod{\blambda}{\bA\bx} \enspace.
\end{equation}

With $g:=\sum_{i\in[n]}f_i$, the KKT operator is
\begin{equation}
\label{eq:F}
F(\bx,\blambda) := \big(\bA\bx,\ \bA^\top\blambda-\partial g(\bx)\big) \enspace,
\end{equation}
and $\mathcal{Z}^\star := F^{-1}(0) = \{(\bx^\star,\blambda^\star): \bA\bx^\star=0,\ \bA^\top\blambda^\star\in\partial g(\bx^\star)\}$ is assumed nonempty, i.e., there exists a primal-dual solution to Problem~\eqref{eq:2}. In particular, for any KKT pair $(\bx^\star,\blambda^\star)\in\mathcal{Z}^\star$, $\bx^\star$ is at consensus and the common value $\bx^\star\in\mathbb{R}^p$ minimizes $\sum_{i\in[n]}f_i$. For notational simplicity, we use $\bx^\star$ to denote both the consensus matrix (with identical rows) and its common row vector value. We write $\bmu:=\bA^\top\blambda\in\mathbb{R}^{n\times p}$ for the aggregated node dual variable associated with $\blambda$: from the orientation of the edges, defining $\varepsilon_{e,k}:=+1$, $\varepsilon_{e,l}:=-1$ for $e=(k,l)\in E$, and $\varepsilon_{e,i}:=0$ for $i\notin e$, and, for $i\in[n]$, letting $N_i:=\{e\in E: i\in e\}$ denote the edges adjacent to $i$, its components are given node-wise by $\bmu_i=\sum_{e\in N_i}\varepsilon_{e,i}\,\blambda_e$; this eliminates the need to store all $\blambda_e$ individually. As is usual in randomized gossip algorithms \citep{boyd2006randomized}, we consider an edge sampling law $(p_e)_{e\in E}$, $p_e>0$, $\sum_{e\in E}p_e=1$, i.e., at each iteration a single edge $e=(k,l)\sim(p_e)_{e\in E}$ is drawn.

\subsection{Algorithm and Convergence Analysis}

We now present Goal-PD, which addresses a fundamental limitation of existing gossip algorithms for non-smooth convex optimization: their memory overhead and exposure to potentially outdated information. As outlined in \Cref{alg:goal-pd}, Goal-PD maintains only two variables per node, $\bx_i$, the primal variable representing the node's local estimate of the solution, and $\bmu_i$, an aggregated dual variable. Importantly, the memory requirement does not grow with the node's degree, eliminating the need for edge-specific variables.

\begin{algorithm}[H]
\caption{\textsc{Goal-PD}: Gossip-based Asynchronous Lite Primal-Dual algorithm}
\label{alg:goal-pd}
\begin{algorithmic}[1]
\State \textbf{Input:} local functions
$(f_i)_{i\in[n]}$; node step sizes
$(\sigma_i)_{i\in[n]}$; scalar $\rho>0$
\State \textbf{Init:} $\bx_i^0\in\mathbb{R}^p$ arbitrary, $\bmu_i^0=0$, for all $i\in[n]$
\For{$t=1,2,\dots$}
    \State draw $e=(k,l)$ with probability $p_e>0$
    \State $\tilde{\bx}_i\gets\prox_{\sigma_i,f_i}\big(\bx_i+\sigma_i\bmu_i\big)$, for $i\in\{k,l\}$
    \State $\bar{\bx}_e\gets\dfrac{1}{2}\big(\tilde{\bx}_k+\tilde{\bx}_l\big)$
    \State $\bmu_i\gets\bmu_i+\rho\big(\bar{\bx}_e-\tilde{\bx}_i\big)$, for $i\in\{k,l\}$
    \State $\bx_i\gets\bar{\bx}_e$, for $i\in\{k,l\}$
\EndFor
\end{algorithmic}
\end{algorithm}

\Cref{alg:goal-pd} proceeds as follows. At each iteration, Goal-PD samples a single edge $e=(k,l)$, activating only its two endpoint nodes while all others remain idle. Each active node $i\in\{k,l\}$ first applies a proximal update to its local objective using its current state $(\bx_i,\bmu_i)$, producing $\tilde{\bx}_i$ (for the definition of the proximal operator, we refer to Section~\ref{sec:quantile}). The two endpoints then average their updated points to form $\bar{\bx}_e$, which each endpoint uses to update its own primal and dual variables. This asynchronous, edge-based update requires each node to store, at any iteration $t$, only the pair $(\bx_i,\bmu_i)\in\mathbb{R}^p\times\mathbb{R}^p$, regardless of its degree.

We first establish that Goal-PD is closely related to Pure-CD~\citep{alacaoglu2020random} specialized to the gossip setting, up to a scalar parameter $\gamma\in(0,1]$, which can be seen as a mixing weight in the gossip averaging. To derive this connection, we swap the primal and dual roles in the saddle-point formulation, so that edges, rather than nodes, are sampled, and, with a judicious choice of parameters, we reduce Pure-CD's step sizes to a node-specific $\sigma_i$ and graph-independent scalars $\rho$ and $\gamma$ (see Appendix~\ref{app:gossip-purecd} for the full derivation), where we recall $\gamma$ is a mixing weight. Goal-PD is recovered exactly by setting $\gamma=1$, which yields the most direct form of the algorithm: for plain averaging gossip with mixing weight $\gamma$, taking $\gamma=1$ leads to the fastest mixing. However, Pure-CD only provides convergence guarantees for $\gamma<1$. This connection allows us to reuse part of the theoretical analysis of Pure-CD. In particular, their one-iteration inequality continues to hold in our setting after an appropriate reformulation (this is exactly Proposition~\ref{lem:lyapunov} below). However, setting $\gamma=1$ eliminates the strict contraction that underlies the proofs of their almost-sure convergence and linear convergence (for PLQ functions) results. We therefore develop a different convergence argument to establish analogous guarantees for Goal-PD.

Although $\bmu_i^t\in\mathbb{R}^{p}$ is the variable actually stored at iteration $t$ and node $i$, we will work with the edge dual variable $\blambda^t\in\mathbb{R}^{m\times p}$ instead, defined by $\blambda^0:=0$ and $\blambda_e^{t+1}:=\blambda_e^t-\rho(\tilde{\bx}_k-\tilde{\bx}_l)/2$ for the sampled edge $e=(k,l)$, so that $\bmu^t=\bA^\top\blambda^t$. Define $\bz^t=(\bx^t,\blambda^t)$, with $\bx^t\in\mathbb{R}^{n\times p}$, as the state associated with \Cref{alg:goal-pd}. Denoting $\tilde{\bz}^{t+1}:=(\tilde{\bx}^{t+1},\tilde{\blambda}^{t+1})$, we have
\[
\begin{cases}
\tilde{\bx}_i^{t+1}:=\prox_{\sigma_i,f_i}\big(\bx_i^t+\sigma_i\bmu_i^t\big), & i\in[n],\\[4pt]
\tilde{\blambda}_e^{t+1}:=\blambda_e^t-\dfrac{\rho}{2}\big(\tilde{\bx}_k^{t+1}-\tilde{\bx}_l^{t+1}\big), & e=(k,l)\in E.
\end{cases}
\]
We condition on the randomness of the sampling process using the natural filtration $\mathcal F_t:=\sigma(e^1,\dots,e^t)$ generated by the edges sampled up to iteration $t$, and write $\Et[\cdot]:=\E[\cdot\mid\mathcal F_t]$ for the corresponding conditional expectation. As is standard in the primal-dual literature~\citep{iutzeler2013asynchronous,bianchi2015coordinate,alacaoglu2020random}, we consider a nonnegative Lyapunov function $V(\bz)$ and show it decreases in expectation, namely $\Et\big[V(\bz^{t+1}-\bz^\star)\big]\ \le\ V(\bz^t-\bz^\star)-K$, where $\bz^\star$ denotes a solution and $K$ a nonnegative term. First, by adapting \citep[Lemma~1]{alacaoglu2020random}, we show that, for an appropriate choice of the step size $\sigma_i$ and Lyapunov function $V$, the following one-iteration inequality holds.

\begin{proposition}[One-Iteration Inequality]
\label{lem:lyapunov}
Define $p_{\min}:=\min_{e\in E}p_e>0$ and $\pi_i:=\sum_{e\in N_i}p_e$, where $\pi_i$ is the probability that node $i$ is touched at a given iteration. Let $\hat{\bz}=(\hat{\bx},\hat{\blambda})$ be any consensus point, i.e., $\hat{\bx}_k=\hat{\bx}_l$ for every $(k,l)\in E$. For any distribution $(p_e)_{e\in E}$, any $\rho>0$, and the choice $\sigma_i=p_{\min}/(\rho\pi_i)$, the following inequality holds for every $t\ge0$:
\[
\Et\big[V(\bz^{t+1}-\hat{\bz})\big]\ \le\ V(\bz^t-\hat{\bz})-K(\tilde{\bx}^{t+1},\bx^t,\hat{\bz}) \enspace,
\]
where the Lyapunov function $V$, for $\bz=(\bx,\blambda)$, is
\[
V(\bx,\blambda):=\frac{p_{\min}}{\rho}\sum_{e\in E}\frac{\|\blambda_e\|^2}{p_e}+\frac\rho2\|\bx\|^2,
\]
and the additional term $K$ is given by
\[
K(\tilde{\bx}^{t+1},\bx^t,\hat{\bz}):=\Phi(\tilde{\bx}^{t+1}-\bx^t)+p_{\min}D_g(\tilde{\bx}^{t+1};\hat{\bz}),
\]
with the Bregman gap of $g$ defined, for $\bz=(\bx,\blambda)$ and $\bmu:=\bA^\top\blambda$, by $D_g(\bx';\bz):=g(\bx')-g(\bx)-\sum_{i\in[n]}\inprod{\bmu_i}{\bx_i'-\bx_i}$, and the auxiliary residual function $\Phi(\bx):=(\rho/2)\sum_{i\in[n]}\pi_i\|\bx_i\|^2$.
\end{proposition}

We provide a self-contained proof of Proposition~\ref{lem:lyapunov} in Appendix~\ref{app:lyapunov-proof} based on our own reformulation. In contrast to \citep[Lemma~1]{alacaoglu2020random}, where $\Phi$ depends on $\blambda$ and is not guaranteed to remain positive definite at $\gamma=1$, we therefore eliminate this dependence. Indeed, when $\gamma=1$, the strict margin induced by $\gamma<1$ vanishes on edges with sampling probability $p_e=p_{\min}$. Importantly, this dependence on $\blambda$ is not needed for the almost-sure convergence proof.

We next specialize Proposition~\ref{lem:lyapunov} to a KKT pair, for which the Bregman gap is guaranteed to be nonnegative, thereby establishing expected Lyapunov descent.

\begin{corollary}[Expected Lyapunov Descent]
\label{cor:lyapunov-descent}
Let $\hat{\bz}=\bz^\star=(\bx^\star,\blambda^\star)\in\mathcal Z^\star$ be a KKT pair, and write $\bmu^\star:=\bA^\top\blambda^\star$. Since $\bmu^\star\in\partial g(\bx^\star)$, the Bregman gap reduces to the Bregman divergence of $g$ generated by $\bmu^\star$,
\[
D_g(\bx';\bz^\star)=g(\bx')-g(\bx^\star)-\inprod{\bmu^\star}{\bx'-\bx^\star}\ \ge\ 0 \qquad\forall \bx',
\]
by convexity of $g$. Moreover, $\Phi\ge0$, and hence $K(\tilde{\bx}^{t+1},\bx^t,\bz^\star)\ge0$. Proposition~\ref{lem:lyapunov} therefore implies that, along the iterates of \Cref{alg:goal-pd}, $V(\bz^t-\bz^\star)$ is nonincreasing in conditional expectation.
\end{corollary}

\input{new_files/4_theory}

%% file: new_files/4_theory.tex
Before presenting the main theoretical guarantees, we first introduce a lemma that is central to our analysis. The lemma establishes that a perturbed version of classical pairwise gossip still converges, provided that the perturbations are square-summable. We use standard notation from the gossip literature~\citep{boyd2006randomized}. Let $\delta(\bx):=\|\bP\bx\|$, where $\bP:=\bI-\mathbf1_n\mathbf1_n^\top/n$ is the orthogonal projection onto the complement of the consensus subspace $\mathcal C:=\{\bx\in\mathbb R^{n\times p}:\bx_i=\bx_j\ \forall i,j\in[n]\}$, so that $\delta(\bx)$ measures how close the nodes are to consensus on average. For $e=(k,l)\in E$, let $\bb_e\in\mathbb R^n$ denote the vector with $+1$ in coordinate $k$, $-1$ in coordinate $l$, and $0$ elsewhere, so that the pairwise-averaging operator $\bW_e:=\bI-\bb_e\bb_e^\top/2$ performs the consensus step of \Cref{alg:goal-pd} on edge $e$, i.e., $(\bW_e\bx)_k=(\bW_e\bx)_l=(\bx_k+\bx_l)/2$ and $(\bW_e\bx)_i=\bx_i$ for $i\notin\{k,l\}$. Finally, $\bL_p:=\sum_{e\in E}p_e\,\bb_e\bb_e^\top$ denotes the sampling-weighted graph Laplacian, with associated spectral-gap constant $\kappa_p:=\lambda_2(\bL_p)/2\in(0,1]$. Remarkably, writing $\br_i^{t+1}:=\tilde{\bx}_i^{t+1}-\bx_i^t$ for the node-wise prox residual, the primal update of \Cref{alg:goal-pd} can be viewed as a perturbed pairwise-averaging gossip update, $\bx^{t+1}=\bW_{e^{t+1}}\bx^t+\beps^{t+1}$, where, for $e^{t+1}=(k,l)$, $\beps^{t+1}$ equals $(\br_k^{t+1}+\br_l^{t+1})/2$ at nodes $k,l$ and $0$ elsewhere. The following lemma characterizes the behavior of pairwise averaging gossip under summable perturbations.

\begin{lemma}[Pairwise averaging gossip with perturbation]
\label{lem:gossip-goalpd}
Suppose that, for every $t\ge0$, $\bx^{t+1}=\bW_{e^{t+1}}\bx^t+\beps^{t+1}$, where the edges $e^{t+1}$ are sampled independently from $(p_e)_{e\in E}$, with $(\bx^t)$ adapted to $(\mathcal F_t)$, and suppose there is an $\mathcal F_t$-measurable nonnegative sequence $(\eta_t)_{t\ge0}$ with $\|\beps^{t+1}\|\le\eta_t$ surely. Then, for every $t\ge0$,
\[
\Et\big[\delta(\bx^{t+1})^2\big]\ \le\ \Big(1-\frac{\kappa_p}2\Big)\delta(\bx^t)^2+\Big(1+\frac2{\kappa_p}\Big)\eta_t^2 .
\]
In particular, if $\sum_{t=0}^\infty\eta_t^2<\infty$ a.s., then, as $t\to\infty$, $\delta(\bx^t)=\|\bP\bx^t\|\to0$ a.s.
\end{lemma}
\begin{proof}[Proof Sketch]
The detailed proof is provided in Appendix~\ref{app:as-convergence-proof}. Projecting the update onto $\mathcal C^\perp$ and applying Young's inequality isolates the perturbation term $\beps^{t+1}$ from the pairwise-averaging contraction, for which a standard spectral-gap argument~\citep{boyd2006randomized} gives the stated one-step inequality. The a.s.\ convergence then follows by applying the Robbins--Siegmund Theorem~\citep{robbins1971convergence} to this inequality.
\end{proof}

We next establish the almost-sure convergence of the Goal-PD iterates by combining \Cref{cor:lyapunov-descent} and \Cref{lem:gossip-goalpd}, as formalized in the following theorem.

\begin{theorem}[Almost-Sure Convergence]
\label{thm:as-convergence}
For the choice $\sigma_i=p_{\min}/(\rho\pi_i)$
as in Proposition~\ref{lem:lyapunov}, for any $\rho>0$, the iterates of \Cref{alg:goal-pd} satisfy, as $t\to\infty$,
\[
\bx^t\to\bx^\star,\qquad\bmu^t\to\bmu^\star\qquad\text{a.s.}
\]
for some KKT pair $(\bx^\star,\blambda^\star)\in\mathcal Z^\star$, where $\bmu^\star=\bA^\top\blambda^\star$. In particular, for all nodes $i\in[n]$, $\bx_i^t\to\bx^\star$ a.s.
\end{theorem}
\begin{proof}[Proof Sketch]
The detailed proof is provided in Appendix~\ref{app:as-convergence-proof}. We first show that the disagreement $\delta(\bx^t)$ converges to zero almost surely, and hence that $\bx^t$ reaches consensus. The key steps are to establish, using Corollary~\ref{cor:lyapunov-descent} and \citep[Proposition~2.3]{combettes2015stochastic}, the summability of the prox residual $\br^{t+1}$, and then to invoke \Cref{lem:gossip-goalpd} with the perturbation $\beps^{t+1}$ equal to $(\br_k^{t+1}+\br_l^{t+1})/2$ at nodes $k,l$ and $0$ elsewhere, for $e^{t+1}=(k,l)$. We next show that every cluster point is a KKT pair. Since the functions $f_i$ are closed, proper, and convex, \citep[Theorem~24.4]{rockafellar1997convex} allows us to pass the prox optimality condition to the limit. Finally, the boundedness of $\bz^t$, together with the fact that all its cluster points belong to the set of KKT pairs, allows us to apply \citep[Proposition~2.3]{combettes2015stochastic} and conclude that $\bz^t=(\bx^t,\blambda^t)$ converges almost surely to some $\bz^\star=(\bx^\star,\blambda^\star)\in\mathcal Z^\star$, and hence $\bmu^t=\bA^\top\blambda^t\to\bA^\top\blambda^\star=\bmu^\star$.
\end{proof}

\begin{remark}[On the choice of step size $\sigma_i$]
We note that, under uniform edge sampling ($p_e=1/m$ for every $e\in E$), the step-size rule reduces to $\sigma_i=1/(\rho d_i)$, where $d_i:=|N_i|$ denotes the degree of node $i$. In addition, for the classical asynchronous edge sampling law $p_e=(1/n)\cdot (1/d_i+1/d_j)$, the same step size is recovered whenever the graph is regular. Empirically, we found $\sigma_i=1/(\rho d_i)$ to be a reasonable choice in general: unlike the theoretical rule $\sigma_i=p_{\min}/(\rho\pi_i)$, it requires no extra computation, since evaluating $p_{\min}$ would otherwise require a full gossip computation, and it was also found to perform slightly better on the median estimation problem (see \Cref{sec:experiments}).
\end{remark}

We now turn to the convergence rate of Goal-PD. In the special case where the functions are piecewise linear-quadratic (PLQ), we show that Goal-PD converges linearly. We first recall the definition of PLQ functions, following~\citep{rockafellar1998variational}.

\begin{definition}[Piecewise linear-quadratic functions]
\label{def:plq}
A function $f:\mathbb R^p\to\mathbb R$ is called piecewise linear-quadratic if $\operatorname{dom} f$ can be represented as the union of finitely many polyhedral sets, relative to each of which $f(\bx)$ is given by an expression of the form $\frac12\inprod{\bx}{\bA\bx}+\inprod{\ba}{\bx}+\alpha$ for some scalar $\alpha\in\mathbb R$, vector $\ba\in\mathbb R^n$, and symmetric matrix $\bA\in\mathbb R^{n\times n}$.
\end{definition}

We note that the pinball, Huber, $\ell_1$-norm, and hinge losses are all standard examples of PLQ functions~\citep{rockafellar1998variational}. We now state the assumption on the objective functions required for our linear convergence result.

\begin{assumption}[PLQ Objectives]
\label{ass:plq}
Every $f_i$, $i\in[n]$, is a piecewise linear-quadratic function.
\end{assumption}

Since the solution set $\mathcal Z^\star$ is not necessarily a singleton, as KKT pairs need not be unique, the convergence rate that we state below is relative to the KKT pair closest to each iterate, in the norm induced by $V$: for $\bz^t=(\bx^t,\blambda^t)$, we write $\bz^t_\star:=P^{V}_{\mathcal Z^\star}(\bz^t):=\arg\min_{\mathbf{u}\in\mathcal Z^\star}V(\bz^t-\mathbf{u})$ for the closest KKT pair to $\bz^t$ in this norm, which is well defined (exists and is unique) since $\mathcal Z^\star$ is closed convex and $V$ is a positive-definite quadratic. Under Assumption~\ref{ass:plq}, we obtain the following linear convergence result for Goal-PD.

\begin{theorem}[Linear Convergence for PLQ Functions]
\label{thm:gamma1-gossip}
Let \Cref{ass:plq} hold, suppose the iterates of \Cref{alg:goal-pd} remain in a bounded set, and let $V$ and the choice $\sigma_i=p_{\min}/(\rho\pi_i)$ be as in Proposition~\ref{lem:lyapunov}. For \emph{every} $\rho>0$, \Cref{alg:goal-pd}
satisfies, for every $t\ge0$,
\[
\E\big[V(\bz^t-\bz^t_\star)\big]\ \le\ (1-c)^t\,\Big(V(\bz^0-\bz^0_\star)+\theta\,\delta(\bx^0)^2\Big),
\]
with $\bz^t_\star:=P^{V}_{\mathcal Z^\star}(\bz^t)=(\bx^t_\star,\blambda^t_\star)$ as above, and $\theta>0$, $c\in(0,1)$ specific constants defined in Appendix~\ref{app:linear-convergence-proof}. In particular, $\E\|\bx^t-\bx^t_\star\|^2\to0$ geometrically fast.
\end{theorem}
\begin{proof}[Proof Sketch]
The detailed proof is provided in Appendix~\ref{app:linear-convergence-proof}. The proof proceeds in three steps. We first derive an upper bound on the distance to $\mathcal Z^\star$ using a regularity property enjoyed by PLQ functions. We then establish auxiliary bounds that exploit properties of the norm induced by $V$. Finally, and most importantly, we combine these bounds with the one-step disagreement contraction of \Cref{lem:gossip-goalpd} to establish a one-step joint contraction, in which, for a well-chosen $\theta$, we bound $\Et[S_{t+1}+\theta\delta(\bx^{t+1})^2]$ in terms of $S_t$ and $\delta(\bx^t)^2$ in expectation, where $S_t$ is the $V$-term and $\delta(\bx^t)^2$ is the disagreement term, and then iterate this recursion to conclude geometric convergence.
\end{proof}

We conclude by showing that, for decentralized mean estimation, Goal-PD recovers the classical pairwise-averaging gossip algorithm~\citep{boyd2006randomized}, establishing a direct connection with our primal-dual framework. Notably, this equivalence holds for any $\sigma_i>0$, rather than only for the specific choice considered above. To the best of our knowledge, such a connection is not available for prior gossip algorithms designed for non-smooth objectives, highlighting the relevance and efficiency of Goal-PD as a primal-dual gossip framework.

\begin{proposition}[Pairwise Averaging]
\label{prop:pairwise-averaging}
Consider the mean-estimation instance $f_i(\bx):=\|\bx-\ba_i\|^2/2$ for observations $\ba_1,\dots,\ba_n\in\mathbb R^p$, $i\in[n]$, with $\bx_i^0=\ba_i$. For any $\sigma_i>0$, $i\in[n]$, and $\rho=1$, \Cref{alg:goal-pd} satisfies, for every $t\ge0$ and every sampled edge $e=(k,l)$, $\tilde{\bx}_i^{t+1}=\bx_i^t$ for $i\in\{k,l\}$, and $\bx_k^{t+1}=\bx_l^{t+1}=(\bx_k^t+\bx_l^t)/2$, i.e., Goal-PD reduces exactly to the classical pairwise-averaging gossip algorithm.
\end{proposition}
\begin{proof}[Proof Sketch]
The proof proceeds by induction on the invariant $\bmu_i^t=\bx_i^t-\ba_i$, which holds at initialization and is preserved by the update. Since $f_i$ is quadratic, its proximal operator is affine, enabling the required cancellation. Although the choice of $\sigma_i$ is arbitrary for this cancellation, setting $\rho=1$ is essential. The full argument is given in Appendix~\ref{app:pairwise-averaging}.
\end{proof}

%% file: new_files/5_exp.tex
\begin{figure*}[t]
    \centering
    \begin{subfigure}{0.32\textwidth}
        \centering
        \includegraphics[height=\myheight, width=\mywidth]{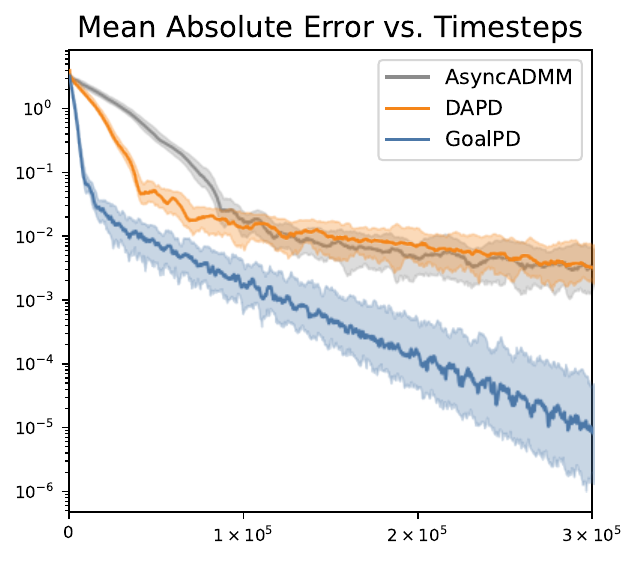}
        \caption{Median estimation}
        \label{subfig:median}
    \end{subfigure}\hfill
    \begin{subfigure}{0.32\textwidth}
        \centering
        \includegraphics[height=\myheight, width=\mywidth]{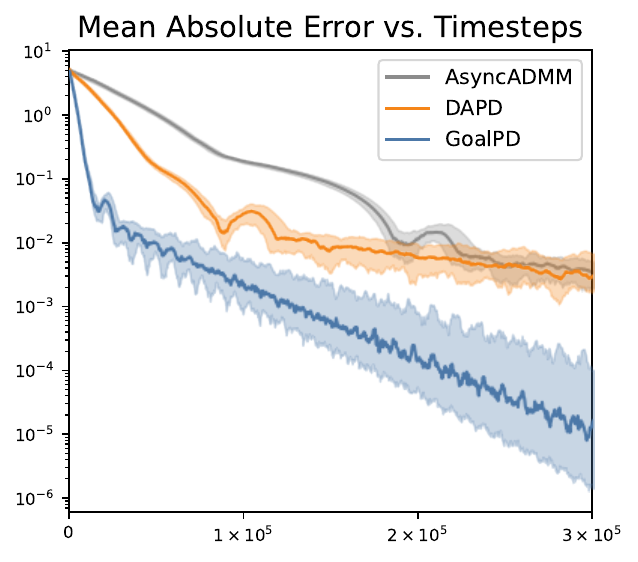}
        \caption{Quantile estimation}
        \label{subfig:quantile}
    \end{subfigure}\hfill
    \begin{subfigure}{0.32\textwidth}
        \centering
        \includegraphics[height=\myheight, width=\mywidth]{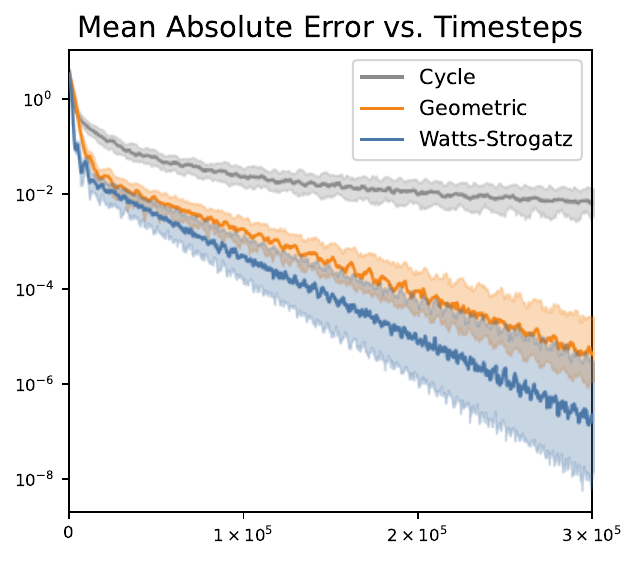}
        \caption{Comparison of topologies}
        \label{subfig:topology}
    \end{subfigure}
    \caption{Convergence of Goal-PD on a geometric graph with $n=101$ nodes, under contaminated Gaussian data. Plots (a) and (b) show that Goal-PD converges faster than AsyncADMM and DAPD for median and quantile ($\alpha=0.3$) estimation, respectively. Plot (c) illustrates the convergence of Goal-PD for median estimation across three network topologies, also with $n=101$ nodes. All plots report the mean absolute error versus timesteps,  curves show the median over $100$ independent trials, with shaded regions indicating the interquartile range.}
    \label{fig:first_fig}
\end{figure*}

\section{Experiments}
\label{sec:experiments}

We empirically evaluate Goal-PD on both synthetic and real datasets, comparing it against AsyncADMM~\citep{iutzeler2013asynchronous} and DAPD~\citep{bianchi2015coordinate}, the two asynchronous baselines described in Appendix~\ref{app:proximal-related}. Our implementation is provided in the supplementary material.

\subsection{Synthetic Data}
\label{sec:synthetic}

\textbf{Setup.} Unless otherwise specified, experiments use a network of $n=101$ nodes arranged in a geometric graph with $507$ edges. Each node observes data from a contaminated Gaussian distribution: $80\%$ from $\mathcal N(10,3^2)$ and $20\%$ from $\mathcal N(30,5^2)$; the sample size is odd to ensure a unique median. All results report the absolute error to the ground-truth statistic, averaged across nodes, versus the number of timesteps; curves show the median over $100$ independent trials, with shaded regions indicating the interquartile range. For every method, we take $\rho=1$ and uniform edge sampling, and data assignments are randomly shuffled across nodes. For Goal-PD, we additionally take step size $\sigma_i=1/d_i$.

\textbf{\Cref{fig:first_fig}.} Figures~\ref{subfig:median} and~\ref{subfig:quantile} show that Goal-PD outperforms existing methods on median and quantile estimation and exhibits linear convergence. Figure~\ref{subfig:topology} investigates the impact of network topology on Goal-PD across networks of varying connectivity, measured by the constant $\kappa_p$ from \Cref{lem:gossip-goalpd}: Watts--Strogatz ($|E|=202$, $\kappa_p=2.28\times10^{-3}$), geometric ($|E|=507$, $\kappa_p=4.34\times10^{-4}$), and cycle ($|E|=101$, $\kappa_p=3.83\times10^{-5}$). Convergence is slowest on the cycle, while the geometric and Watts--Strogatz graphs converge significantly faster, confirming that graph connectivity drives the convergence rate of Goal-PD.

\begin{table}[t]
\centering
\caption{Effect on convergence of two different step size choices: the practical step size $\hat\sigma_i:=1/d_i$ and the theoretical step size $\sigma_i^{\mathrm{th}}:=p_{\min}/\pi_i$, on median estimation, under the classical degree-weighted edge sampling law, across network topologies with varying degree heterogeneity. $R_\sigma$ is the average ratio of $\sigma_i^{\mathrm{th}}/\hat\sigma_i$ over all nodes, and $R_\varepsilon$ is the corresponding ratio of median absolute errors, averaged over $100$ trials.}
\label{tab:stepsize}
\begin{tabular*}{\linewidth}{@{\extracolsep{\fill}}lccc@{}}
\toprule
\textbf{Network} & $R_\sigma$ & $R_\varepsilon$ & \textbf{Winner} \\
\midrule
Barab\'asi--Albert & 0.140 & 15.536 & $\hat\sigma_i$ \\
Geometric                  & 0.533 &  5.167 & $\hat\sigma_i$ \\
Watts--Strogatz            & 0.608 &  2.420 & $\hat\sigma_i$ \\
4-Regular                  & 1.000 &  1.000 & Tie \\
Star                       & 1.000 &  1.000 & Tie \\
\bottomrule
\end{tabular*}
\end{table}

\textbf{Choice of $\sigma_i$ under non-uniform edge sampling.} While our convergence guarantees are established for the theoretical step size $\sigma_i^{\mathrm{th}}:=p_{\min}/\pi_i$ (\Cref{lem:lyapunov}), we argue that, in practice, $\hat\sigma_i:=1/d_i$ remains a good choice under the classical degree-weighted edge sampling law $p_e=(1/n)\cdot(1/d_i+1/d_j)$, especially since the theoretical step size can be quite conservative when the sampling probabilities $p_e$ are highly heterogeneous. A comparison of these two choices is reported in \Cref{tab:stepsize}. Across all tested topologies, $\hat\sigma_i$ performs at least as well as $\sigma_i^{\mathrm{th}}$, with the largest improvements observed on graphs with greater degree heterogeneity. This is consistent with the smaller step sizes induced by the theoretical rule in these settings, which can slow convergence. On the 4-regular and star graphs, the two choices coincide and consequently yield identical performance. Appendix~\ref{app:additional-experiments} provides more details on the experiment setup as well as additional experiments.

\begin{figure*}[t]
    \centering
    \begin{subfigure}{0.32\textwidth}
        \centering
        \includegraphics[height=\myheight, width=\mywidth]{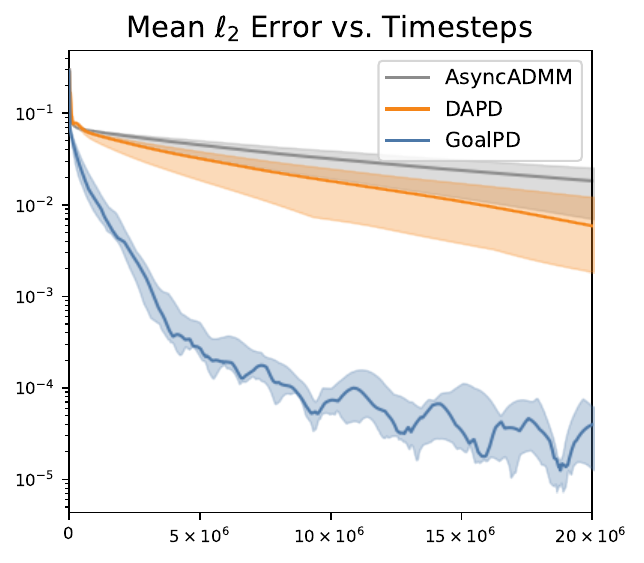}
        \caption{Coordinatewise median}
        \label{subfig:median2}
    \end{subfigure}\hfill
    \begin{subfigure}{0.32\textwidth}
        \centering
        \includegraphics[height=\myheight, width=\mywidth]{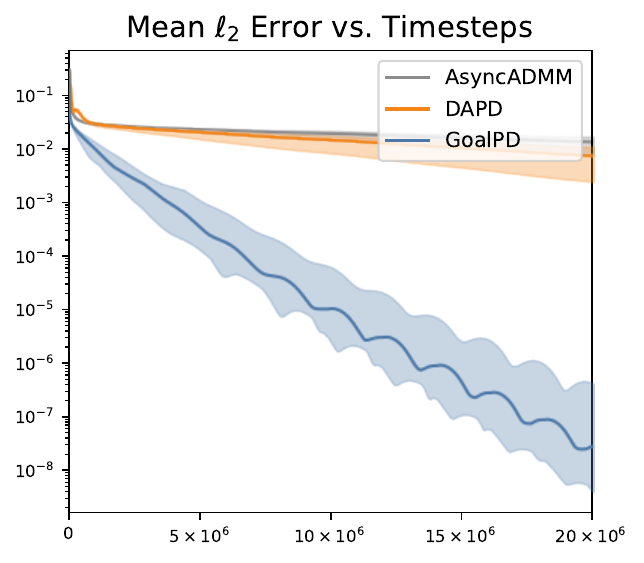}
        \caption{Geometric median}
        \label{subfig:quantile2}
    \end{subfigure}\hfill
    \begin{subfigure}{0.32\textwidth}
        \centering
        \includegraphics[height=\myheight, width=\mywidth]{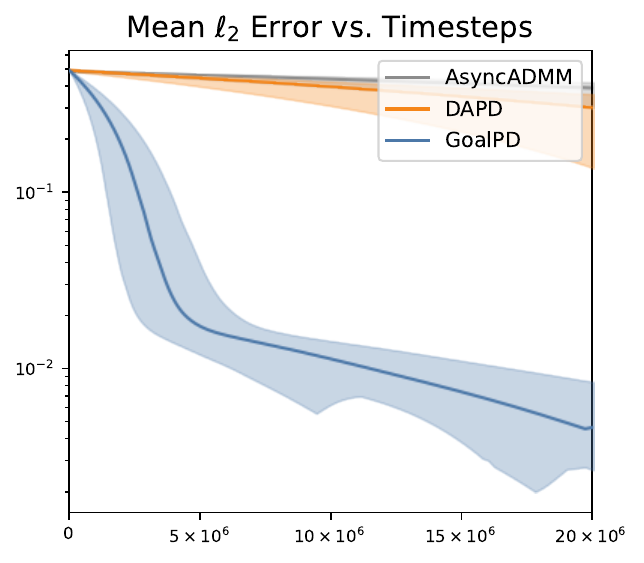}
        \caption{Median regression}
        \label{subfig:topology2}
    \end{subfigure}
    \caption{Convergence of Goal-PD on multidimensional non-smooth convex problems on the Wine Quality dataset, on a geometric graph with average degree $33$, with $6497$ nodes and $105{,}706$ edges. All plots show that Goal-PD outperforms the existing gossip baselines, namely AsyncADMM and DAPD, on coordinatewise median (a), geometric median (b), and median regression (c). All plots report the mean $\ell_2$ error versus timesteps; curves show the median over $10$ independent trials, with shaded regions indicating the interquartile range.}
    \label{fig:second_fig}
\end{figure*}

\subsection{Wine Quality Dataset}
\label{sec:wine}

We further evaluate Goal-PD on the UCI Wine Quality dataset~\citep{winequality186}, which combines $4898$ white and $1599$ red wine samples ($6497$ samples in total), each described by $11$ physicochemical features together with a label. We assign one sample to each node of a geometric graph with average degree $33$, with $6497$ nodes and $105{,}706$ edges, and consider three multidimensional non-smooth convex tasks over these physicochemical features. All results report the $\ell_2$ distance to the ground-truth statistic, averaged across nodes and normalized by the dimension, versus the number of timesteps; curves show the median over $10$ independent trials, with shaded regions indicating the interquartile range. Regarding the choice of the step size $\rho$, we test $10$ equally spaced values of $\rho$ in the interval $[0.2, 5.0]$. The three non-smooth problems can be described as follows. Writing $\ba_i\in\mathbb R^{11}$ for the feature vector held by node $i$, the \emph{coordinatewise median} takes $f_i(\bx)=\sum_{k=1}^{11}L_{1/2}((\ba_i)_k-x_k)$ in Problem~\eqref{eq:1}, which is separable across coordinates: since each term is a univariate pinball loss, its proximal operator decomposes into $11$ independent copies of the closed-form univariate median prox recalled in \Cref{sec:quantile}. The \emph{geometric median}~\citep{minsker2015geometric} instead takes $f_i(\bx)=\|\bx-\ba_i\|_2$, minimizing the sum of Euclidean distances to the data points; although this loss does not separate across coordinates, its proximal operator also admits a closed-form solution, derived in \Cref{app:geometric-median-prox}. Finally, \emph{median regression}~\citep{koenker2001quantile} extends the pinball loss with a regression component: writing $b_i\in\mathbb R$ for node $i$'s label, we take $f_i(\bx) = L_\alpha\big(b_i-\inprod{\ba_i}{\bw}-c\big)$ with $\alpha=1/2$, so that $\bx=(\bw^\top,c)^\top\in\mathbb R^{12}$ is now a shared regression coefficient. Further details on median regression are provided in Appendix~\ref{app:median-regression-prox}.

The features in the wine dataset are measured in heterogeneous units, with standard deviations ranging from $3 \times 10^{-3}$ to $56$. To prevent features with large numerical ranges from dominating the Euclidean norm in the geometric median objective, we scale each feature by its standard deviation. Nodes can precompute these standard deviations via gossip averaging of the first two moments. This scaling not only renders the estimator invariant to changes in measurement units, but also simplifies the selection of the step size, $\rho$. While the optimal $\rho$ on raw data is highly scale-dependent, standardization makes $\rho = 1$ a robust default. Figure~\ref{fig:second_fig} reports the resulting convergence curves: across all three tasks, Goal-PD significantly outperforms both AsyncADMM and DAPD. Appendix~\ref{app:wine-experiments} contains further details on the experimental setup, together with additional results.

%% file: new_files/6_conclusion.tex
\section{Conclusion}
\label{sec:conclusion}

We introduced Goal-PD, an asynchronous primal-dual gossip algorithm for non-smooth decentralized convex optimization. Unlike existing methods requiring $\mathcal{O}(d)$ auxiliary variables per node, Goal-PD achieves a degree-independent memory footprint. Theoretically, we established almost-sure convergence for general non-smooth convex functions and proved linear convergence for piecewise linear-quadratic objectives, including the pinball loss. Empirically, Goal-PD outperforms state-of-the-art asynchronous baselines across diverse network topologies and heterogeneous data distributions, delivering faster convergence with substantially lower memory overhead. Future work will aim to extend this scalable framework to more general robust learning problems and non-convex objectives.

%% file: new_files/A_appendix.tex
\section{Baselines}
\label{app:proximal-related}
In this section, we describe the baseline algorithms for decentralized convex non-smooth optimization mentioned in Section~\ref{sec:preliminaries}.

\subsection{Distributed Subgradient Descent from \texorpdfstring{\citet{nedic2009distributed}}{Nedic and Ozdaglar (2009)}}

\begin{algorithm}[H]
    \caption{Subgradient Descent }
    \label{alg:gd}
    \begin{algorithmic}[1]
        \State \textbf{Input:} Initial vectors $\ba_1, \ldots, \ba_n$; step size $\rho > 0$, subgradient $\bg_k$ for each local loss $f_k$.
        \State \textbf{Initialization:} For all nodes $k=1, \ldots, n$: $\bx_k \gets \ba_k$, $\bmu_k \gets 0$.
        \For{$t=0, 1, \ldots$}
            \State For all $k$, update:
            $\bx_k \gets \bx_k - \dfrac{\rho}{\sqrt{t+1}} \bg_k$.
            \State Draw $e = (i,j) \in E$ with probability $p_e$ and compute average: $\bx_i^+, \bx_j^+ \gets  \frac{1}{2}(\bx_i + \bx_j)$ .
        \EndFor
    \end{algorithmic}
\end{algorithm}

\subsection{AsyncADMM from \texorpdfstring{\citet{iutzeler2013asynchronous}}{Iutzeler et al. (2013)}}

\begin{algorithm}[H]
    \caption{AsyncADMM}
    \label{alg:asyncadmm}
    \begin{algorithmic}[1]
        \State \textbf{Input:} Initial vectors $\ba_1, \ldots, \ba_n$; step size $\rho > 0$.
        \State \textbf{Initialization:} For all agents $k=1, \ldots, n$: $\bx_k \gets \ba_k$, $\blambda_{kl} \gets 0$, $\bar{\bx}_{kl} \gets \ba_k$ for all $l \in \mathcal{N}_k$.
        \For{$t=0, 1, \ldots$}
            \State Draw $(i,j) \in E$ with probability $p_{ij}$.
            \State Agents $k \in \{i,j\}$ update their primal variable: $\quad \bx_k \gets \operatorname{prox}_{\frac{f_k}{\rho d_k}}\left (\frac{1}{d_k}(\sum_{l \in \mathcal{N}_k} \bar{\bx}_{kl} -\frac{1}{\rho} \blambda_{kl})\right)$
            \State Compute average: $\bar{\bx} \gets \frac{\bx_i + \bx_j}{2}$ .
            \State Update: $\blambda_{ij} \gets \blambda_{ij} + \rho(\bx_i -  \bar{\bx} )$ ; $\blambda_{ji} \gets \blambda_{ji} + \rho(\bx_j -  \bar{\bx} )$.
            \State Update $ \bar{\bx}_{ij}, \bar{\bx}_{ji} \gets \bar{\bx}$.
        \EndFor
    \end{algorithmic}
\end{algorithm}

\subsection{DAPD with \texorpdfstring{$f=0$}{f=0} from \texorpdfstring{\citet{bianchi2015coordinate}}{Bianchi et al. (2015)}}

\begin{algorithm}[H]
    \caption{DAPD (adapted\textit{ }for\textit{ $f=0$} using $\tau = \rho/2$ and $\rho \mapsto 1/\rho$ ) }
    \label{alg:dapd}
    \begin{algorithmic}[1]
        \State \textbf{Input:} Initial vectors $\ba_1, \ldots, \ba_n$; step size $\rho > 0$.
        \State \textbf{Initialization:} For all agents $k=1, \ldots, n$: $\bx_k \gets \ba_k$, $\blambda_{kl} \gets 0$, $\bar{\bx}_{kl} \gets \ba_k$ for all $l \in \mathcal{N}_k$.
        \For{$t=0, 1, \ldots$}
            \State Draw $(i,j) \in E$ with probability $p_{ij}$.
            \State Update dual variables: \\$\blambda_{ij} \gets \frac{1}{2}(\blambda_{ij}-\blambda_{ji}) + \frac{\rho}{2}(\bx_i -  \bx_j )$ ;      $\blambda_{ji} \gets -\blambda_{ij}$.
            \State Agents $k \in \{i,j\}$ update their primal variable: \\$\bx_k \gets \operatorname{prox}_{\frac{f_k}{\rho d_k}}\left (\frac{1}{2}\bx_k {+} \frac{1}{2d_k}\sum_{l \in \mathcal{N}_k} (\bar{\bx}_{kl} {-} \frac{1}{\rho}\blambda_{kl})\right)$.
            \State Update neighbor values: $\bar{\bx}_{ij} \gets \bx_j$, $\bar{\bx}_{ji} \gets \bx_i$.
        \EndFor
    \end{algorithmic}
\end{algorithm}

%% file: new_files/B_appendix.tex
\section{Relationship between Goal-PD and Pure-CD}
\label{app:gossip-purecd}

We first recall the generic \textsc{Pure-CD} problem formulation and algorithm~\citep{alacaoglu2020random}. We then perform two main changes to specialize it to the gossip setting.

\textsc{Pure-CD} solves saddle-point problems of the form
\[
\min_{\bx\in\mathcal X}\max_{\by\in\mathcal Y}\ f(\bx)+g(\bx)+\inprod{\bA\bx}{\by}-h^*(\by) \enspace,
\]
where $f,g:\mathcal X\to\mathbb R\cup\{+\infty\}$ and $h:\mathcal Y\to\mathbb R\cup\{+\infty\}$ are proper, closed, convex, $f$ has coordinatewise Lipschitz gradients, $\bA$ is a linear operator, and $\mathcal X=\prod_{i=1}^n\mathcal X_i$, $\mathcal Y=\prod_{j=1}^m\mathcal Y_j$ are block-separable. Writing $I(j):=\{i:\bA_{j,i}\ne0\}$ and $J(i):=\{j:\bA_{j,i}\ne0\}$ for the row/column sparsity patterns of $\bA$, and given diagonal step-size matrices $\theta,\tau,\sigma\succ0$ with $\theta_j:=\pi_j/\underline p$, $\pi_j:=\sum_{i\in I(j)}p_i$, $\underline p:=\min_ip_i$ (and $\tau$ set according to their step-size rule), \textsc{Pure-CD} iterates, for $k=0,1,\dots$,
\[
\bar{\by}_{k+1}=\prox_{\sigma,h^*}\big(\by_k+\sigma \bA\bx_k\big) \enspace, \qquad
\bar{\bx}_{k+1}=\prox_{\tau,g}\big(\bx_k-\tau(\nabla f(\bar{\bx}_k)+\bA^\top\bar{\by}_{k+1})\big) \enspace,
\]
then draws a coordinate index $i_{k+1}\in\{1,\dots,n\}$ with $\mathbb P(i_{k+1}=i)=p_i$, and updates
\[
\bx_{k+1}^{i_{k+1}}=\bar{\bx}_{k+1}^{i_{k+1}}, \qquad \bx_{k+1}^j=\bx_k^j\ \forall j\ne i_{k+1} \enspace,
\]
\[
\by_{k+1}^j=\bar{\by}_{k+1}^j+\sigma_j\theta_j\big(\bA(\bx_{k+1}-\bx_k)\big)_j\ \forall j\in J(i_{k+1}), \qquad \by_{k+1}^j=\by_k^j\ \forall j\notin J(i_{k+1}) \enspace.
\]

We show that Goal-PD (\Cref{alg:goal-pd}) is recovered, up to a single scalar margin $\gamma\in(0,1]$, from a gossip-based, edge-sampled specialization of \textsc{Pure-CD}~\citep{alacaoglu2020random}.

\textbf{Primal/dual role swap.} The saddle-point formulation~\eqref{eq:2} of the main text already matches \textsc{Pure-CD}'s template, with $\bx$ as the primal (sampled) block. Sampling $\bx$ directly would activate a node and \emph{all} of its incident edges at once, not a single communication link. Negating and swapping the order of the arguments of the Lagrangian preserves the same saddle point while swapping which variable is sampled, giving the identification
\[
f\equiv0, \qquad \tilde g\equiv0, \qquad \tilde{\bA}=\bA^\top, \qquad \tilde h^*(\bx)=g(\bx)=\sum_i f_i(\bx_i) \enspace.
\]
The sampled blocks are now the edges, while, with $\tilde g\equiv0$, the proximal step is performed at the nodes on the separable function $\tilde h^*=g$.

\textbf{Simplifying the step-size parameters.} \textsc{Pure-CD}'s generic rule involves node weights $\theta_i=\pi_i/p_{\min}$, with $\pi_i:=\sum_{e\in N_i}p_e$, and edge step sizes $\tau_e$ that otherwise depend on the graph and the sampling law. Choosing $\sigma_i:=p_{\min}/(\rho\,\pi_i)$ for a single free scalar $\rho>0$ makes $\pi_i\sigma_i\equiv p_{\min}/\rho$ and $\sigma_i\theta_i\equiv1/\rho$ at every node, collapsing the step-size condition to the single graph-free threshold $\tau_e<\rho/2$. Setting $\tau_e\equiv\gamma\rho/2$ then makes the $x$-update coefficient $\sigma_i\theta_i\tau_e=\gamma/2$ constant as well, leaving only the two free scalars $\rho,\gamma$ in the algorithm.

Under this identification, drawing $e=(k,l)$ with probability $p_e>0$, the update becomes, for $i\in\{k,l\}$,
\begin{equation}
\label{eq:gossip-purecd}
\tilde{\bx}_i\gets\prox_{\sigma_i,f_i}\big(\bx_i+\sigma_i\bmu_i\big), \qquad
\bar{\bx}_e\gets\dfrac{1}{2}\big(\tilde{\bx}_k+\tilde{\bx}_l\big), \qquad
\bmu_i\gets\bmu_i+\gamma\rho\big(\bar{\bx}_e-\tilde{\bx}_i\big), \qquad
\bx_i\gets\tilde{\bx}_i+\gamma\big(\bar{\bx}_e-\tilde{\bx}_i\big) \enspace.
\end{equation}

The margin $\gamma\in(0,1)$ sets how tightly the edge step size approaches the threshold required by \textsc{Pure-CD}'s step-size rule for its almost-sure convergence guarantee; $\gamma$ must stay strictly below $1$ for that guarantee to apply. At the boundary $\gamma=1$, the last two updates in \eqref{eq:gossip-purecd} reduce exactly to
\[
\bmu_i\gets\bmu_i+\rho\big(\bar{\bx}_e-\tilde{\bx}_i\big), \qquad
\bx_i\gets\bar{\bx}_e, \qquad i\in\{k,l\} \enspace,
\]
which is exactly the update of Goal-PD (\Cref{alg:goal-pd}). Goal-PD is simply the ($\gamma=1$) case of \eqref{eq:gossip-purecd}. It is the most direct version of the algorithm, although it falls just outside the region covered by \textsc{Pure-CD}'s convergence guarantee.

%% file: new_files/C_appendix.tex
\section{Proof of Proposition~\ref{lem:lyapunov}}
\label{app:lyapunov-proof}

This section gives a self-contained proof of Proposition~\ref{lem:lyapunov} (One-Iteration Inequality).

We first bound the Bregman gap $D_g(\tilde{\bx}^{t+1};\hat{\bz})$ using the optimality condition of the proximal operator together with the convexity of each $f_i$. Since $\tilde{\bx}_i^{t+1}=\prox_{\sigma_i,f_i}(\bx_i^t+\sigma_i\bmu_i^t)$, the prox optimality condition $\sigma_i^{-1}(\bx_i^t+\sigma_i\bmu_i^t-\tilde{\bx}_i^{t+1})\in\partial f_i(\tilde{\bx}_i^{t+1})$ combined with the subgradient inequality for the convex $f_i$ and the three-point identity $2\inprod{\ba}{\bb}=\|\ba\|^2+\|\bb\|^2-\|\ba-\bb\|^2$, applied with $\ba:=\bx_i^t-\tilde{\bx}_i^{t+1}$ and $\bb:=\bx_i'-\tilde{\bx}_i^{t+1}$, gives, for every $\bx_i'\in\R^p$,
\[
f_i(\bx_i')\ge f_i(\tilde{\bx}_i^{t+1})+\inprod{\bmu_i^t}{\bx_i'-\tilde{\bx}_i^{t+1}}
+\frac1{2\sigma_i}\Big(\|\bx_i^t-\tilde{\bx}_i^{t+1}\|^2+\|\bx_i'-\tilde{\bx}_i^{t+1}\|^2-\|\bx_i'-\bx_i^t\|^2\Big).
\]
Summing over $i\in[n]$ at $\bx_i'=\hat{\bx}_i$, using $g=\sum_if_i$, $1/\sigma_i=\rho\pi_i/p_{\min}$, subtracting $\sum_i\inprod{\hat{\bmu}_i}{\tilde{\bx}_i^{t+1}-\hat{\bx}_i}$ from both sides, then multiplying by $p_{\min}$, and using the definition of $D_g$:
\begin{align}
p_{\min}D_g(\tilde{\bx}^{t+1};\hat{\bz})&\le p_{\min}\sum_{i\in[n]}\inprod{\bmu_i^t-\hat{\bmu}_i}{\tilde{\bx}_i^{t+1}-\hat{\bx}_i} \notag\\
&\quad-\frac\rho2\sum_{i\in[n]}\pi_i\|\bx_i^t-\tilde{\bx}_i^{t+1}\|^2-\frac\rho2\sum_{i\in[n]}\pi_i\|\hat{\bx}_i-\tilde{\bx}_i^{t+1}\|^2+\frac\rho2\sum_{i\in[n]}\pi_i\|\hat{\bx}_i-\bx_i^t\|^2. \label{eq:proof-node}
\end{align}

We will now rewrite the cross term containing $\bmu_i^t-\hat{\bmu}_i$ using the edge dual variable $\blambda$. Recall that $\hat{\bmu}_i=\sum_{e\in N_i}\varepsilon_{e,i}\hat{\blambda}_e$, so $\bmu_i^t-\hat{\bmu}_i=\sum_{e\in N_i}\varepsilon_{e,i}(\blambda_e^t-\hat{\blambda}_e)$. Swapping the order of summation (each edge $e=(k,l)$ contributes once through $N_k$ and once through $N_l$, with opposite sign) gives
\[
\sum_{i\in[n]}\inprod{\bmu_i^t-\hat{\bmu}_i}{\tilde{\bx}_i^{t+1}-\hat{\bx}_i}
=\sum_{e=(k,l)\in E}\inprod{\blambda_e^t-\hat{\blambda}_e}{(\tilde{\bx}_k^{t+1}-\hat{\bx}_k)-(\tilde{\bx}_l^{t+1}-\hat{\bx}_l)} .
\]
Write $\tilde{\bv}_e:=\tilde{\bx}_k^{t+1}-\tilde{\bx}_l^{t+1}$ for $e=(k,l)\in E$, and by definition of $\tilde{\blambda}_e^{t+1}$, we have $\tilde{\blambda}_e^{t+1}-\blambda_e^t=-\tilde{\bv}_e\rho/2$. Since $\hat{\bz}$ is at consensus, $\hat{\bx}_k=\hat{\bx}_l$ for every edge $e=(k,l)$, the bracket collapses to $\tilde{\bv}_e$, and using $\tilde{\bv}_e=-2(\tilde{\blambda}_e^{t+1}-\blambda_e^t)/\rho$, we obtain
\begin{equation}
\label{eq:proof-cross}
p_{\min}\sum_{i\in[n]}\inprod{\bmu_i^t-\hat{\bmu}_i}{\tilde{\bx}_i^{t+1}-\hat{\bx}_i}
=\frac{2p_{\min}}\rho\sum_{e\in E}\inprod{\tilde{\blambda}_e^{t+1}-\blambda_e^t}{\hat{\blambda}_e-\blambda_e^t}.
\end{equation}

Next, we bound the expectation of the potential function $\Et[V(\bz^{t+1}-\hat{\bz})]$, treating its $x$- and $\lambda$-parts in turn. By \Cref{alg:goal-pd}, when the sampled edge is $e^{t+1}=(k,l)$ (which happens with probability $p_e$), $\bx_k^{t+1}=\bx_l^{t+1}=\bar{\bx}_e^{t+1}$ and $\bx_i^{t+1}=\bx_i^t$ for every $i\notin\{k,l\}$, so
\[
\Et\big[\|\bx^{t+1}-\hat{\bx}\|^2\big]=\sum_{e=(k,l)\in E}p_e\Big(\|\bar{\bx}_e^{t+1}-\hat{\bx}_k\|^2
+\|\bar{\bx}_e^{t+1}-\hat{\bx}_l\|^2+\sum_{i\notin\{k,l\}}\|\bx_i^t-\hat{\bx}_i\|^2\Big).
\]
Since $\hat{\bz}$ is at consensus, $\hat{\bx}_k=\hat{\bx}_l$, the first part simplifies to $2\|\bar{\bx}_e^{t+1}-\hat{\bx}_k\|^2$, and applying the parallelogram identity $\|\ba+\bb\|^2=2\|\ba\|^2+2\|\bb\|^2-\|\ba-\bb\|^2$ with $\ba:=\tilde{\bx}_k^{t+1}-\hat{\bx}_k$ and $\bb:=\tilde{\bx}_l^{t+1}-\hat{\bx}_l$,
\[
2\|\bar{\bx}_e^{t+1}-\hat{\bx}_k\|^2=2\big\|\tfrac{\ba+\bb}2\big\|^2
=\|\tilde{\bx}_k^{t+1}-\hat{\bx}_k\|^2+\|\tilde{\bx}_l^{t+1}-\hat{\bx}_l\|^2-\tfrac12\|\tilde{\bv}_e\|^2 .
\]
Substituting back and using $\sum_{e\in N_i}p_e=\pi_i$ and $\sum_{e\notin N_i}p_e=1-\pi_i$,
\[
\Et\big[\|\bx^{t+1}-\hat{\bx}\|^2\big]=\sum_{i\in[n]}\pi_i\|\tilde{\bx}_i^{t+1}-\hat{\bx}_i\|^2
-\tfrac12\sum_{e\in E}p_e\|\tilde{\bv}_e\|^2+\|\bx^t-\hat{\bx}\|^2-\sum_{i\in[n]}\pi_i\|\bx_i^t-\hat{\bx}_i\|^2 .
\]
Multiplying by $\rho/2$ and using $\|\tilde{\bv}_e\|^2=4\|\tilde{\blambda}_e^{t+1}-\blambda_e^t\|^2/\rho^2$,
\begin{equation}
\label{eq:proof-x}
\frac\rho2\Et\big[\|\bx^{t+1}-\hat{\bx}\|^2\big]=\frac\rho2\sum_{i\in[n]}\pi_i\big(\|\tilde{\bx}_i^{t+1}-\hat{\bx}_i\|^2-\|\bx_i^t-\hat{\bx}_i\|^2\big)
-\frac1\rho\sum_{e\in E}p_e\|\tilde{\blambda}_e^{t+1}-\blambda_e^t\|^2+\frac\rho2\|\bx^t-\hat{\bx}\|^2 .
\end{equation}

We now turn to the $\lambda$-part. Since $\blambda_e^{t+1}=\tilde{\blambda}_e^{t+1}$ with probability $p_e$ and $\blambda_e^{t+1}=\blambda_e^t$ otherwise, $\Et[\|\blambda_e^{t+1}-\hat{\blambda}_e\|^2]=p_e\|\tilde{\blambda}_e^{t+1}-\hat{\blambda}_e\|^2+(1-p_e)\|\blambda_e^t-\hat{\blambda}_e\|^2$. Expanding $\tilde{\blambda}_e^{t+1}-\hat{\blambda}_e=(\tilde{\blambda}_e^{t+1}-\blambda_e^t)+(\blambda_e^t-\hat{\blambda}_e)$, multiplying by $p_{\min}/(\rho p_e)$, and summing over $e\in E$,
\begin{align}
\frac{p_{\min}}\rho\sum_{e\in E}\frac{\Et[\|\blambda_e^{t+1}-\hat{\blambda}_e\|^2]}{p_e}
&=\frac{p_{\min}}\rho\sum_{e\in E}\|\tilde{\blambda}_e^{t+1}-\blambda_e^t\|^2
+\frac{2p_{\min}}\rho\sum_{e\in E}\inprod{\tilde{\blambda}_e^{t+1}-\blambda_e^t}{\blambda_e^t-\hat{\blambda}_e} \notag\\
&\quad+\frac{p_{\min}}\rho\sum_{e\in E}\frac{\|\blambda_e^t-\hat{\blambda}_e\|^2}{p_e} . \label{eq:proof-lam}
\end{align}

We now combine \eqref{eq:proof-x} and \eqref{eq:proof-lam} to form $\Et[V(\bz^{t+1}-\hat{\bz})]$:
\begin{align}
\Et[V(\bz^{t+1}-\hat{\bz})]&=V(\bz^t-\hat{\bz})+\frac\rho2\sum_{i\in[n]}\pi_i\|\tilde{\bx}_i^{t+1}-\hat{\bx}_i\|^2-\frac\rho2\sum_{i\in[n]}\pi_i\|\bx_i^t-\hat{\bx}_i\|^2 \notag\\
&\quad-\frac1\rho\sum_{e\in E}(p_e-p_{\min})\|\tilde{\blambda}_e^{t+1}-\blambda_e^t\|^2
-\frac{2p_{\min}}\rho\sum_{e\in E}\inprod{\tilde{\blambda}_e^{t+1}-\blambda_e^t}{\hat{\blambda}_e-\blambda_e^t}. \label{eq:proof-combine}
\end{align}
Finally, we combine \eqref{eq:proof-node} and \eqref{eq:proof-combine}: by \eqref{eq:proof-cross}, the last term of \eqref{eq:proof-combine} is the negative of the first term on the right-hand side of \eqref{eq:proof-node}, so they cancel, as do the $\pm\sum_i\pi_i\|\tilde{\bx}_i^{t+1}-\hat{\bx}_i\|^2$ and $\pm\sum_i\pi_i\|\bx_i^t-\hat{\bx}_i\|^2$ terms, and we obtain
\[
p_{\min}D_g(\tilde{\bx}^{t+1};\hat{\bz})+\Et[V(\bz^{t+1}-\hat{\bz})]\ \le\ V(\bz^t-\hat{\bz})
-\frac1\rho\sum_{e\in E}(p_e-p_{\min})\|\tilde{\blambda}_e^{t+1}-\blambda_e^t\|^2-\frac\rho2\sum_{i\in[n]}\pi_i\|\tilde{\bx}_i^{t+1}-\bx_i^t\|^2 .
\]
Since $p_e\ge p_{\min}$ for every edge, $\frac1\rho\sum_{e\in E}(p_e-p_{\min})\|\tilde{\blambda}_e^{t+1}-\blambda_e^t\|^2\ge0$, so dropping this term gives
\[
\Et[V(\bz^{t+1}-\hat{\bz})]\ \le\ V(\bz^t-\hat{\bz})-\frac\rho2\sum_{i\in[n]}\pi_i\|\tilde{\bx}_i^{t+1}-\bx_i^t\|^2-p_{\min}D_g(\tilde{\bx}^{t+1};\hat{\bz}) ,
\]
which is exactly the claim.

%% file: new_files/D_appendix.tex
\section{Proof of Theorem~\ref{thm:as-convergence}}
\label{app:as-convergence-proof}

In this section, the goal is to prove Theorem~\ref{thm:as-convergence}, and we start by providing useful notation and lemmas. Throughout this appendix, $\sigma_i=p_{\min}/(\rho\pi_i)$ as in Proposition~\ref{lem:lyapunov}. As in Corollary~\ref{cor:lyapunov-descent}, we write $\bz^\star=(\bx^\star,\blambda^\star)\in\mathcal Z^\star$ and $\bmu^\star:=\bA^\top\blambda^\star$.

\subsection{Additional notation}

For $\bx\in\mathbb R^{n\times p}$ and $e=(k,l)\in E$, write $\bv_e:=\bx_k-\bx_l$ for the edge-$e$ disagreement. Since $G$ is connected, $\bx$ is at consensus if and only if $\bv_e=0$ for every $e\in E$. We use the notation $\mathcal C$, $\bP$, $\delta$, $\bb_e$, $\bW_e$, $\bL_p$ and $\kappa_p$ introduced before \Cref{lem:gossip-goalpd}, and recall the following standard facts. Each $\bW_e=\bI-\bb_e\bb_e^\top/2$ is a symmetric projection ($\bW_e^2=\bW_e=\bW_e^\top$) with $\bP\bW_e=\bW_e\bP$, and
\[
\langle \bx,\bL_p\bx\rangle=\sum_{e\in E}p_e\|\bv_e\|^2\ \ge0 .
\]
As is standard in gossip analysis~\citep{boyd2006randomized}, since $G$ is connected and $p_e>0$ for every edge, $\bL_p$ has one-dimensional null space $\mathrm{span}(\mathbf1_n)$ and second-smallest eigenvalue $\lambda_2(\bL_p)\in(0,2]$, so that $\kappa_p=\lambda_2(\bL_p)/2\in(0,1]$.

\subsection{Preliminary lemma}

We first present a useful lemma, which derives from Proposition~\ref{lem:lyapunov} using \citep[Proposition~2.3]{combettes2015stochastic}, as is standard in the literature.

\begin{lemma}[Boundedness and convergence]
\label{lem:residual-goalpd}
The sequence $(\bz^t)$ is bounded a.s.\ and, for every $\bz^\star=(\bx^\star,\blambda^\star)\in\mathcal Z^\star$, $\sum_{t=0}^\infty D_g(\tilde{\bx}^{t+1};\bz^\star)<\infty$ a.s. Moreover, with $\pi_{\min}:=\min_{i\in[n]}\pi_i>0$ (positive since $G$ is connected and $p_e>0$ for every edge),
\[
\sum_{t=0}^\infty\|\tilde{\bx}^{t+1}-\bx^t\|^2\ \le\ \frac2{\rho\,\pi_{\min}}\sum_{t=0}^\infty\Phi(\tilde{\bx}^{t+1}-\bx^t)<\infty
\qquad\text{a.s.},
\]
in particular, $\tilde{\bx}^{t+1}-\bx^t\to0$ a.s.\ as $t\to\infty$.
\end{lemma}

\begin{proof}
We argue similarly to~\citep{alacaoglu2020random}, using Proposition 2.3 of
\citep{combettes2015stochastic} in the Hilbert space $\mathbb R^{n\times p}\times\mathbb R^{m\times p}$ endowed with the norm $\bz\mapsto V(\bz)^{1/2}$ (a norm since $p_e>0$ for every $e\in E$), applied to the inequality of \Cref{lem:lyapunov} with
target set $\mathcal Z^\star$, the sequence denoted $(x_n)$ in \citep{combettes2015stochastic} taken to be $(\bz^t)$, and
$\vartheta_t(\bz):=K(\tilde{\bx}^{t+1},\bx^t,\bz)\ge0$ for $\bz\in\mathcal Z^\star$ by
\Cref{cor:lyapunov-descent}, to obtain boundedness of $(\bz^t)$
and the stated summability. The
residual bound follows from $\Phi(\tilde{\bx}^{t+1}-\bx^t)=\rho\sum_i\pi_i\|
\tilde{\bx}_i^{t+1}-\bx_i^t\|^2/2\ge\rho\,\pi_{\min}\|\tilde{\bx}^{t+1}-\bx^t\|^2/2$,
summed over $t$.
\end{proof}

\subsection{Proof of \texorpdfstring{\Cref{lem:gossip-goalpd}}{Lemma 1}}

We prove \Cref{lem:gossip-goalpd}, which shows that inexact pairwise averaging with a square-summable perturbation
drives the disagreement $\delta(\bx^t)$ to $0$, corresponding to the
consensus step in \Cref{alg:goal-pd}.

\begin{proof}[Proof of \Cref{lem:gossip-goalpd}]
Since $\bP\bW_e=\bW_e\bP$ for every $e\in E$, writing $\by^t:=\bP\bx^t$ gives
$\by^{t+1}=\bW_{e^{t+1}}\by^t+\bP\beps^{t+1}$. The goal is to bound
$\Et\|\by^{t+1}\|^2$ in terms of $\|\by^t\|^2$. Since $\kappa_p>0$, applying
Young's inequality $2\inprod{\ba}{\bb}\le(\kappa_p/2)\|\ba\|^2+(2/\kappa_p)\|\bb\|^2$ with
$\ba=\bW_{e^{t+1}}\by^t$, $\bb=\bP\beps^{t+1}$ gives
\[
\|\by^{t+1}\|^2\ \le\ \Big(1+\frac{\kappa_p}2\Big)\|\bW_{e^{t+1}}\by^t\|^2
+\Big(1+\frac2{\kappa_p}\Big)\|\bP\beps^{t+1}\|^2 .
\]
Since $\by^t\in\mathcal C^\perp$ and each matrix $\bW_e$ is a symmetric
projection, we obtain the standard gossip inequality~\citep{boyd2006randomized}:
\[
\Et\|\bW_{e^{t+1}}\by^t\|^2=\sum_{e\in E}p_e\,(\by^t)^\top \bW_e\by^t
=(\by^t)^\top\Big(\bI-\tfrac12\bL_p\Big)\by^t
\ \le\ (1-\kappa_p)\|\by^t\|^2,
\]
using $(\by^t)^\top \bL_p\by^t\ge\lambda_2(\bL_p)\|\by^t\|^2$. Since $\bP$ is
an orthogonal projection, $\|\bP\beps^{t+1}\|\le\|\beps^{t+1}\|\le\eta_t$
surely, and $\eta_t$ is $\mathcal F_t$-measurable, so taking $\Et$ of both
sides,
\[
\Et\|\by^{t+1}\|^2\ \le\ \Big(1+\frac{\kappa_p}2\Big)(1-\kappa_p)\|\by^t\|^2
+\Big(1+\frac2{\kappa_p}\Big)\eta_t^2
\ \le\ \Big(1-\frac{\kappa_p}2\Big)\|\by^t\|^2+\Big(1+\frac2{\kappa_p}\Big)\eta_t^2 ,
\]
and, since $\|\by^t\|=\delta(\bx^t)$, this is exactly the stated inequality.

Now suppose, in addition, that $\sum_{t=0}^\infty\eta_t^2<\infty$ a.s. The perturbation $(1+2/\kappa_p)\eta_t^2$ is $\mathcal F_t$-measurable and a.s.\
summable by hypothesis, and $\kappa_p/2>0$ is a fixed constant, so
the Robbins--Siegmund Theorem~\citep{robbins1971convergence} gives a.s.\ convergence of $\|\by^t\|^2$ and
$\sum_t\|\by^t\|^2<\infty$ a.s. We therefore derive $\|\by^t\|=\delta(\bx^t)\to0$ a.s., which finishes the proof.
\end{proof}

\subsection{Proof of Theorem~\ref{thm:as-convergence}}

\begin{proof}
We first show that $\delta(\bx^t)\to0$ a.s., i.e.\ that $\bx^t$ converges to
consensus almost surely. Set $\br^{t+1}:=\tilde{\bx}^{t+1}-\bx^t$; by
\Cref{lem:residual-goalpd}, $(\bz^t)$ is bounded a.s.\ and
$\sum_t\|\br^{t+1}\|^2<\infty$ a.s.
For the sampled edge $e^{t+1}=(k,l)$, \Cref{alg:goal-pd} gives
$\bx^{t+1}=\bW_{e^{t+1}}\bx^t+\beps^{t+1}$, where $\beps^{t+1}$ equals
$(\br_k^{t+1}+\br_l^{t+1})/2$ at nodes $k,l$ and $0$ elsewhere, so
$\|\beps^{t+1}\|^2=\|\br_k^{t+1}+\br_l^{t+1}\|^2/2\le\|\br^{t+1}\|^2$ surely,
for every possible edge draw. Thus $\|\beps^{t+1}\|\le\|\br^{t+1}\|=:\eta_t$,
where $\eta_t$ is $\mathcal F_t$-measurable and
$\sum_t\eta_t^2<\infty$ a.s. \Cref{lem:gossip-goalpd} therefore gives
$\delta(\bx^t)\to0$ a.s.

\smallskip\noindent Next, we will show that any cluster point $\hat{\bz}=(\hat{\bx},\hat{\blambda})$ is a
KKT pair, i.e.\ $\hat{\bz}\in\mathcal Z^\star$; write $\hat{\bmu}:=\bA^\top\hat{\blambda}$.
Since $(\bz^t)$ is bounded a.s., it admits, a.s., at least one convergent
subsequence $\bz^{t_j}=(\bx^{t_j},\blambda^{t_j})\to\hat{\bz}=(\hat{\bx},\hat{\blambda})$.
Since $\delta$ is continuous and $\delta(\bx^{t_j})\to0$ (shown above),
$\delta(\hat{\bx})=0$, i.e.,
$\hat{\bv}_e:=\hat{\bx}_k-\hat{\bx}_l=0$ for $e=(k,l)\in E$.
The prox optimality condition used in the proof of \Cref{lem:lyapunov},
$\sigma_i^{-1}(\bx_i^t+\sigma_i\bmu_i^t-\tilde{\bx}_i^{t+1})\in\partial
f_i(\tilde{\bx}_i^{t+1})$, rewrites as $\bmu_i^t-\br_i^{t+1}/\sigma_i\in\partial
f_i(\tilde{\bx}_i^{t+1})$, $i\in[n]$. Along $t_j$, $\br^{t_j+1}\to0$ while
$\bmu^{t_j}\to\hat{\bmu}$ (the aggregate of
$\hat{\blambda}$) and $\tilde{\bx}^{t_j+1}=\bx^{t_j}+\br^{t_j+1}\to\hat{\bx}$. Since
$\sigma_i$ is a fixed constant and $f_i$ is a closed proper convex function, we can apply \citep[Theorem~24.4]{rockafellar1997convex}, which gives
$\hat{\bmu}_i\in\partial f_i(\hat{\bx}_i)$ for every $i\in[n]$, i.e.\
$\hat{\bmu}\in\partial g(\hat{\bx})$. Combining $\hat{\bv}_e=0$ for every $e\in E$
with $\hat{\bmu}\in\partial g(\hat{\bx})$, we obtain $\hat{\bz}=(\hat{\bx},\hat{\blambda})\in\mathcal Z^\star$.
As the subsequence was an arbitrary convergent one, almost surely, the set
of cluster points of $(\bz^t)$ is included in $\mathcal Z^\star$.

\smallskip\noindent It remains to prove almost-sure convergence of $(\bz^t)$.
Since $(\bz^t)$ is bounded a.s.\ and the set of its cluster points is included
in $\mathcal Z^\star$ a.s., \citep[Proposition~2.3(vi)]{combettes2015stochastic} shows that $(\bz^t)$
converges a.s.\ to a random variable $\bz^\star=(\bx^\star,\blambda^\star)\in\mathcal Z^\star$. In particular, $\bx^t\to\bx^\star$ a.s., and since $\bx^\star$ is at consensus, $\bx_i^t\to\bx^\star$ a.s.\ for every $i\in[n]$ (identifying $\bx^\star$ with its common value). Since $\bmu=\bA^\top\blambda$
is continuous, $\bmu^t=\bA^\top\blambda^t\to \bA^\top\blambda^\star=\bmu^\star$ a.s.\ as well,
which is exactly the claim.

\end{proof}

%% file: new_files/E_appendix.tex
\section{Proof of Theorem~\ref{thm:gamma1-gossip}}
\label{app:linear-convergence-proof}

This appendix proves Theorem~\ref{thm:gamma1-gossip}. Throughout, we work in the $\bz=(\bx,\blambda)$ coordinates of Section~3.1, on which $V$, $\Phi$, $D_g$, and Proposition~\ref{lem:lyapunov} are defined, and we reuse the operator $F$ and the solution set $\mathcal Z^\star$ exactly as stated before it. Recall $p_{\max}:=\max_ep_e$ and $\pi_{\max}:=\max_i\pi_i$, and let $\Delta:=\max_{i\in[n]}d_i$ be the maximum node degree.

The proof proceeds in three steps. We first derive an upper bound on the distance to $\mathcal Z^\star$ using the metric subregularity of $F$, since the $f_i$ are PLQ functions (\Cref{app:metric-subreg}). We then establish auxiliary bounds that exploit properties of the norm induced by $V$ (\Cref{app:aux-bounds-V}). Finally, and most importantly, we combine these bounds to establish a one-step joint contraction (\Cref{app:one-step-joint}), in which, for a well-chosen $\theta$, we bound $\Et[S_{t+1}+\theta\delta_{t+1}^2]$ in terms of $S_t$ and $\delta_t^2$ in expectation, where $S_t$ is the $V$-term and $\delta_t^2$ is the disagreement term, and then iterate this recursion to conclude geometric convergence.

\subsection{Upper bound using metric subregularity}
\label{app:metric-subreg}

Following~\citep{alacaoglu2020random}, we derive a useful upper bound using the metric subregularity of $F$. We first define the distance to a set, then state the metric subregularity assumption, which, under \Cref{ass:plq}, reduces to the iterates remaining in a bounded set (see \Cref{rem:subreg-plq} and \citep[Lemma~4.4]{latafat2019new}). For $\bz=(\bx,\blambda)$ and a nonempty set $S$ in the same product space, $\dist(\bz,S):=\inf_{\mathbf{u}\in S}\|\bz-\mathbf{u}\|$ for the usual Euclidean norm. The assumption is given below, and we use it throughout this section.

\begin{assumption}[Metric Subregularity]
\label{ass:subreg}
There exist a set $\mathcal N\supseteq\mathcal Z^\star$ and a constant $L_F>0$ such that
\[
\dist(\bz,\mathcal Z^\star)\ \le\ L_F\,\dist\big(0,F(\bz)\big)\qquad\text{for all }\bz\in\mathcal N,
\]
and the formal iterates of \Cref{alg:goal-pd} satisfy $\tilde{\bz}^{t+1}\in\mathcal N$ for every $t\ge0$, almost surely.
\end{assumption}

\begin{remark}
\label{rem:subreg-plq}
Under \Cref{ass:plq}, each $\partial f_i$ is a polyhedral multifunction \citep[Proposition~12.30]{rockafellar1998variational}, hence so is $F$, since $\bA$ is linear. Applying the Corollary in \citep{robinson2009some}, we obtain that there are $\kappa,\varepsilon>0$ such that $\dist(\bz,\mathcal Z^\star)\le\kappa\,\dist(0,F(\bz))$ whenever $\dist(0,F(\bz))\le\varepsilon$. Consequently, for any bounded $\mathcal N$ with $\sup_{\bz\in\mathcal N}\dist(\bz,\mathcal Z^\star)\le R$, the inequality in \Cref{ass:subreg} holds on $\mathcal N$ with $L_F=\max\{\kappa,R/\varepsilon\}$: indeed, if $\dist(0,F(\bz))>\varepsilon$, then $\dist(\bz,\mathcal Z^\star)\le R<(R/\varepsilon)\,\dist(0,F(\bz))$. Under \Cref{ass:plq}, the assumption therefore reduces to the iterates remaining in a fixed bounded set.
\end{remark}

We will first make explicit the relationship between $\tilde{\bz}^{t+1}$ and $\bz^t$ using the following operators $H$, $M$, and $\mathcal A$, defined on the same product space as $\bz$:
\[
H(\bz):=\big(\tau^{-1}\blambda+\bA\bx,\ \sigma^{-1}\bx\big), \qquad M(\bz):=\big(\bA\bx,\ -\bA^\top\blambda\big), \qquad \mathcal A(\bz):=\big(0,\ \partial g(\bx)\big),
\]
with $\tau:=\rho/2$ and $(\sigma^{-1}\bx)_i:=\sigma_i^{-1}\bx_i$. The desired identity is given in the following lemma.

\begin{lemma}
\label{lem:splitting-identity}
Fix $\rho>0$ and $t\ge0$, and let $\tilde{\bz}^{t+1}=(\tilde{\bx}^{t+1},\tilde{\blambda}^{t+1})$ be the formal update of \Cref{alg:goal-pd} from $\bz^t$. Then
\[
\tilde{\bz}^{t+1}\ =\ (H+\mathcal A)^{-1}(H-M)\bz^t ,
\]
equivalently, $\tilde{\blambda}^{t+1}=\blambda^t-\tau \bA\tilde{\bx}^{t+1}$ (the $\lambda$-block) and $\sigma^{-1}(\bx^t-\tilde{\bx}^{t+1})+\bA^\top\blambda^t\in\partial g(\tilde{\bx}^{t+1})$ (the $x$-block).
\end{lemma}

\begin{proof}
Write $\bmu^t:=\bA^\top\blambda^t$. The formal update of \Cref{alg:goal-pd} is, at every node $i$ and edge $e=(k,l)$,
\[
\tilde{\bx}_i^{t+1}=\prox_{\sigma_i,f_i}(\bx_i^t+\sigma_i\bmu_i^t), \qquad \tilde{\blambda}_e^{t+1}=\blambda_e^t-\tau(\tilde{\bx}_k^{t+1}-\tilde{\bx}_l^{t+1}),
\]
i.e., in vector form, $\tilde{\blambda}^{t+1}=\blambda^t-\tau \bA\tilde{\bx}^{t+1}$. Since $(H-M)(\bz)=\big(\tau^{-1}\blambda,\ \sigma^{-1}\bx+\bA^\top\blambda\big)$, we check these are exactly the two block components of $(H-M)\bz^t-H\tilde{\bz}^{t+1}\in\mathcal A(\tilde{\bz}^{t+1})$.

\emph{$\lambda$-block (first).} $(H-M)\bz^t$ has $\lambda$-component $\tau^{-1}\blambda^t$, and $H\tilde{\bz}^{t+1}$ has $\lambda$-component $\tau^{-1}\tilde{\blambda}^{t+1}+\bA\tilde{\bx}^{t+1}$; since $\mathcal A$'s $\lambda$-component is identically $0$, the identity's $\lambda$-block reads $\tau^{-1}(\blambda^t-\tilde{\blambda}^{t+1})-\bA\tilde{\bx}^{t+1}=0$, i.e.\ $\tilde{\blambda}^{t+1}=\blambda^t-\tau \bA\tilde{\bx}^{t+1}$.

\emph{$x$-block (second).} $(H-M)\bz^t$ has $x$-component $\sigma^{-1}\bx^t+\bmu^t$, and $H\tilde{\bz}^{t+1}$ has $x$-component $\sigma^{-1}\tilde{\bx}^{t+1}$; since $\mathcal A$'s $x$-component is $\partial g$, the identity's $x$-block reads $\sigma^{-1}(\bx^t-\tilde{\bx}^{t+1})+\bmu^t\in\partial g(\tilde{\bx}^{t+1})$, exactly the prox optimality condition $\sigma_i^{-1}[(\bx_i^t+\sigma_i\bmu_i^t)-\tilde{\bx}_i^{t+1}]\in\partial f_i(\tilde{\bx}_i^{t+1})$ at every node.

This completes the proof of the lemma.
\end{proof}

Using \Cref{lem:splitting-identity} combined with \Cref{ass:subreg}, we can now derive an upper bound on $\dist(\tilde{\bz}^{t+1},\mathcal Z^\star)$.

\begin{lemma}
\label{lem:distance-bound}
Fix $\rho>0$ and $t\ge0$. Under Assumption~\ref{ass:subreg} (applicable since $\tilde{\bz}^{t+1}\in\mathcal N$ for every $t$),
\[
\dist(\tilde{\bz}^{t+1},\mathcal Z^\star)\ \le\ L_F\,\|H-M\|\,\|\bz^t-\tilde{\bz}^{t+1}\| .
\]
\end{lemma}

\begin{proof}
Write $\bw:=(H-M)(\bz^t-\tilde{\bz}^{t+1})=(\bw_1,\bw_2)$. By the $\lambda$-block of \Cref{lem:splitting-identity}, $\bw_1=\tau^{-1}(\blambda^t-\tilde{\blambda}^{t+1})=\bA\tilde{\bx}^{t+1}$, exactly the first component of $F(\tilde{\bz}^{t+1})=(\bA\tilde{\bx}^{t+1},\ \bA^\top\tilde{\blambda}^{t+1}-\partial g(\tilde{\bx}^{t+1}))$. For the $x$-component, $\bw_2=\sigma^{-1}(\bx^t-\tilde{\bx}^{t+1})+\bA^\top(\blambda^t-\tilde{\blambda}^{t+1})$; substituting $\bA^\top\blambda^t=\bA^\top\tilde{\blambda}^{t+1}+\tau \bA^\top \bA\tilde{\bx}^{t+1}$ (from the $\lambda$-block) into the $x$-block condition $\sigma^{-1}(\bx^t-\tilde{\bx}^{t+1})+\bA^\top\blambda^t\in\partial g(\tilde{\bx}^{t+1})$ gives $\bA^\top\tilde{\blambda}^{t+1}+\bw_2\in\partial g(\tilde{\bx}^{t+1})$, i.e., $-\bw_2\in \bA^\top\tilde{\blambda}^{t+1}-\partial g(\tilde{\bx}^{t+1})$, the second component of $F(\tilde{\bz}^{t+1})$. Hence $(\bw_1,-\bw_2)\in F(\tilde{\bz}^{t+1})$, and $\|(\bw_1,-\bw_2)\|=\|(\bw_1,\bw_2)\|=\|\bw\|$.

Since $(\bw_1,-\bw_2)$ is one particular element of $F(\tilde{\bz}^{t+1})$,
\[
\dist\big(0,F(\tilde{\bz}^{t+1})\big)\ \le\ \|(\bw_1,-\bw_2)\|\ =\ \|\bw\|\ =\ \big\|(H-M)(\bz^t-\tilde{\bz}^{t+1})\big\|\ \le\ \|H-M\|\,\|\bz^t-\tilde{\bz}^{t+1}\| .
\]
Using \Cref{ass:subreg} at $\bz=\tilde{\bz}^{t+1}$ gives $\dist(\tilde{\bz}^{t+1},\mathcal Z^\star)\le L_F\,\dist\big(0,F(\tilde{\bz}^{t+1})\big)$, which, combined with the previous inequality, finishes the proof.
\end{proof}

We note that $\|H-M\|$ can be further bounded using standard results on matrix norms, as well as spectral graph theory. Splitting $H-M$ into the block-diagonal part $D(\bz):=(\tau^{-1}\blambda,\sigma^{-1}\bx)$ and the strictly triangular part $N(\bz):=(0,\bA^\top\blambda)$, standard operator-norm properties give $\|D\|=\max\{\tau^{-1},\max_i\sigma_i^{-1}\}$ and $\|N\|=\|\bA\|$, so
\begin{equation}
\label{eq:H-M-bound}
\|H-M\|\ \le\ \|D\|+\|N\|\ =\ \max\Big\{\tau^{-1},\ \frac{\rho\pi_{\max}}{p_{\min}}\Big\}+\|\bA\| ,
\end{equation}
using $\max_i\sigma_i^{-1}=\rho\pi_{\max}/p_{\min}$. Moreover, $\bA^\top \bA$ is the combinatorial graph Laplacian $\bL:=\mathrm{Deg}-\mathrm{Adj}$ of $G$, where $\mathrm{Deg}:=\operatorname{diag}(d_1,\dots,d_n)$ is the degree matrix and $\mathrm{Adj}$ is the $(0,1)$ adjacency matrix, since $(\bA^\top \bA)_{ii}=d_i$ and $(\bA^\top \bA)_{ij}=-1$ if $(i,j)\in E$ and $0$ otherwise. So $\|\bA\|=\lambda_{\max}(\bA^\top \bA)^{1/2}$, and by the Anderson--Morley bound~\citep{anderson1985eigenvalues,zhang2011laplacian}, $\lambda_{\max}(\bA^\top \bA)\le\max_{(i,j)\in E}(d_i+d_j)\le2\Delta$, hence $\|\bA\|\le\sqrt{2\Delta}$.

\subsection{Auxiliary bounds with \texorpdfstring{$V$}{V}}
\label{app:aux-bounds-V}

Here, we provide useful bounds with $V$ similar to those in~\citep{alacaoglu2020random}. Recall $\bz^t_\star:=P^{V}_{\mathcal Z^\star}(\bz^t):=\arg\min_{\mathbf{u}\in\mathcal Z^\star}V(\bz^t-\mathbf{u})$, the closest KKT pair to $\bz^t$ in the norm induced by $V$, where $P^V_{\mathcal Z^\star}$ denotes the Hilbert-space projection onto the closed convex set $\mathcal Z^\star$, induced by the inner product $\langle\cdot,\cdot\rangle_V$ underlying $V$. We note that the metric projection onto a nonempty closed convex set in a Hilbert space is $1$-Lipschitz, since its defining optimality condition $\langle \mathbf{u}-P^V_{\mathcal Z^\star}(\mathbf{u}),\mathbf{v}-P^V_{\mathcal Z^\star}(\mathbf{u})\rangle_V\le0$ for all $\mathbf{v}\in\mathcal Z^\star$, applied at $\mathbf{u}=\ba$ with $\mathbf{v}=P^V_{\mathcal Z^\star}(\bb)$ and at $\mathbf{u}=\bb$ with $\mathbf{v}=P^V_{\mathcal Z^\star}(\ba)$ and added, gives $\|P^V_{\mathcal Z^\star}(\ba)-P^V_{\mathcal Z^\star}(\bb)\|_V^2\le\langle \ba-\bb,P^V_{\mathcal Z^\star}(\ba)-P^V_{\mathcal Z^\star}(\bb)\rangle_V\le\|\ba-\bb\|_V\|P^V_{\mathcal Z^\star}(\ba)-P^V_{\mathcal Z^\star}(\bb)\|_V$ by Cauchy--Schwarz. Using the fact that $V(\cdot)^{1/2}$ is a norm, we first provide an upper bound on $V(\bz^t-\bz^t_\star)$, obtained using Cauchy--Schwarz and the triangle inequality applied twice.

\begin{lemma}
\label{lem:Dp0}
We denote by $\|\cdot\|_V$ the norm induced by $V$, i.e., $\|\bz\|_V:=V(\bz)^{1/2}$. For every $t\ge0$,
\begin{equation}
\label{eq:triangle}
\|\bz^t-\bz^t_\star\|_V\ \le\ 2\,\|\bz^t-\tilde{\bz}^{t+1}\|_V+\big\|\tilde{\bz}^{t+1}-P^{V}_{\mathcal Z^\star}(\tilde{\bz}^{t+1})\big\|_V.
\end{equation}
\end{lemma}

\begin{proof}
Since $\|\cdot\|_V$ is a norm, the triangle inequality applied twice, through $\tilde{\bz}^{t+1}$ and through $P^V_{\mathcal Z^\star}(\tilde{\bz}^{t+1})$, gives
\[
\|\bz^t-\bz^t_\star\|_V\ \le\ \|\bz^t-\tilde{\bz}^{t+1}\|_V+\|\tilde{\bz}^{t+1}-\bz^t_\star\|_V
\]
\[
\le\ \|\bz^t-\tilde{\bz}^{t+1}\|_V+\big\|\tilde{\bz}^{t+1}-P^V_{\mathcal Z^\star}(\tilde{\bz}^{t+1})\big\|_V+\big\|P^V_{\mathcal Z^\star}(\tilde{\bz}^{t+1})-\bz^t_\star\big\|_V.
\]
Since $\bz^t_\star=P^V_{\mathcal Z^\star}(\bz^t)$, $1$-Lipschitzness of the projection gives
\[
\big\|P^V_{\mathcal Z^\star}(\tilde{\bz}^{t+1})-\bz^t_\star\big\|_V=\big\|P^V_{\mathcal Z^\star}(\tilde{\bz}^{t+1})-P^V_{\mathcal Z^\star}(\bz^t)\big\|_V\ \le\ \|\tilde{\bz}^{t+1}-\bz^t\|_V=\|\bz^t-\tilde{\bz}^{t+1}\|_V,
\]
which, substituted above, gives the claim.
\end{proof}

We next show that $V$ is equivalent to the squared Euclidean norm, with explicit constants.

\begin{lemma}
\label{lem:cV-CV}
For every $\rho>0$ and every $\bz=(\bx,\blambda)$,
\begin{equation}
\label{eq:cV-CV}
m_V\|\bz\|^2\le V(\bx,\blambda)\le M_V\|\bz\|^2, \qquad
m_V:=\min\Big\{\frac{p_{\min}}{\rho\,p_{\max}},\ \frac\rho2\Big\}, \qquad
M_V:=\max\Big\{\frac1{\rho},\ \frac\rho2\Big\} .
\end{equation}
\end{lemma}

\begin{proof}
By definition, $m_V\le p_{\min}/(\rho p_{\max})\le p_{\min}/(\rho p_e)\le1/\rho\le M_V$ for every $e\in E$ (using $p_{\min}\le p_e\le p_{\max}$), and $m_V\le\rho/2\le M_V$. Summing $m_V\|\blambda_e\|^2\le\frac{p_{\min}}{\rho p_e}\|\blambda_e\|^2\le M_V\|\blambda_e\|^2$ over $e\in E$ and adding $m_V\|\bx\|^2\le\frac\rho2\|\bx\|^2\le M_V\|\bx\|^2$ gives $m_V\|\bz\|^2\le V(\bx,\blambda)\le M_V\|\bz\|^2$.
\end{proof}

\subsection{One-step joint contraction}
\label{app:one-step-joint}

Our goal is to establish a one-step joint contraction bound (\Cref{lem:joint-contraction}), for which we first need a few additional bounds. We reuse the disagreement notation of Appendix~\ref{app:as-convergence-proof}: $\delta(\bx):=\|\bP\bx\|$, the pairwise-averaging operator $\bW_e$, the sampling-weighted Laplacian $\bL_p$, and the spectral-gap constant $\kappa_p:=\lambda_2(\bL_p)/2\in(0,1]$.

From what was used in the proof of Theorem~\ref{thm:as-convergence} and established in Lemma~\ref{lem:gossip-goalpd} (Appendix~\ref{app:as-convergence-proof}), applied with $\eta_t:=\|\tilde{\bx}^{t+1}-\bx^t\|$, we first have the following useful bound on the disagreement term, for every $t\ge0$:
\begin{equation}
\label{eq:disagreement-bound}
\Et\big[\delta(\bx^{t+1})^2\big]\ \le\ \Big(1-\frac{\kappa_p}2\Big)\delta(\bx^t)^2+\Big(1+\frac2{\kappa_p}\Big)\|\tilde{\bx}^{t+1}-\bx^t\|^2 .
\end{equation}

We next bound $\|\bz^t-\tilde{\bz}^{t+1}\|^2$ in terms of the disagreement $\delta(\bx^t)^2$ and the term $Q_t:=\Phi(\tilde{\bx}^{t+1}-\bx^t)$.

\begin{lemma}
\label{lem:zbar-disagreement}
For every $t\ge0$,
\[
\|\bz^t-\tilde{\bz}^{t+1}\|^2\ \le\ c_A\,\delta(\bx^t)^2+c_B\,\Phi(\tilde{\bx}^{t+1}-\bx^t),
\qquad
c_A:=\frac{\rho^2\|\bA\|^2}2,\qquad
c_B:=\frac{\rho^2\|\bA\|^2+2}{\rho\,\pi_{\min}} .
\]
\end{lemma}

\begin{proof}
Write $Q_t:=\Phi(\tilde{\bx}^{t+1}-\bx^t)=\frac\rho2\sum_{i\in[n]}\pi_i\|\tilde{\bx}_i^{t+1}-\bx_i^t\|^2\ge\frac{\rho\,\pi_{\min}}2\|\tilde{\bx}^{t+1}-\bx^t\|^2$, hence
\begin{equation}
\label{eq:r-via-Q}
\|\bx^t-\tilde{\bx}^{t+1}\|^2\ \le\ \frac2{\rho\pi_{\min}}Q_t .
\end{equation}
By the identity $\blambda^t-\tilde{\blambda}^{t+1}=\frac\rho2\bA\tilde{\bx}^{t+1}$ (established in the proof of Lemma~\ref{lem:splitting-identity}), $\tilde{\bx}^{t+1}=\bx^t+(\tilde{\bx}^{t+1}-\bx^t)$, and $\|\bA\bx^t\|\le\|\bA\|\delta(\bx^t)$, we get $\|\blambda^t-\tilde{\blambda}^{t+1}\|\le\frac\rho2\big(\|\bA\|\delta(\bx^t)+\|\bA\|\|\tilde{\bx}^{t+1}-\bx^t\|\big)$. Squaring with $(a+b)^2\le2a^2+2b^2$ and substituting \eqref{eq:r-via-Q} gives $\|\blambda^t-\tilde{\blambda}^{t+1}\|^2\le\frac{\rho^2\|\bA\|^2}2\delta(\bx^t)^2+\frac{\rho\|\bA\|^2}{\pi_{\min}}Q_t$. Adding \eqref{eq:r-via-Q} and collecting the coefficient of $Q_t$ gives the claim.
\end{proof}

Combining the previous bounds with \Cref{lem:distance-bound} and \Cref{lem:Dp0}, we now establish the joint one-step contraction.

\begin{lemma}
\label{lem:joint-contraction}
Let \Cref{ass:subreg} hold, and set $C_1:=1+2/\kappa_p$, $C_2:=2C_1/(\rho\pi_{\min})$, and $C_3:=M_V\big(2+L_F\|H-M\|\big)^2$, with $M_V$ as in \Cref{lem:cV-CV} and $\|H-M\|$ as in \eqref{eq:H-M-bound}. For $t\ge0$, with $\bz^t_\star$ as in \Cref{lem:Dp0}, write
\[
S_t:=V(\bz^t-\bz^t_\star), \qquad \delta_t^2:=\delta(\bx^t)^2, \qquad
Q_t:=\Phi(\tilde{\bx}^{t+1}-\bx^t) ,
\]
and let $c_A,c_B$ be as in \Cref{lem:zbar-disagreement}. Define $\theta:=\frac{4c_A}{4c_AC_2+\kappa_pc_B}>0$. Then, for every $t\ge0$,
\[
\Et\big[S_{t+1}+\theta\delta_{t+1}^2\big]\ \le\
S_t\Big[1-\frac{1-\theta C_2}{C_3c_B}\Big]+\delta_t^2\Big[\theta\Big(1-\frac{\kappa_p}2\Big)+\frac{(1-\theta C_2)c_A}{c_B}\Big] .
\]
\end{lemma}

\begin{proof}
Applying Proposition~\ref{lem:lyapunov} at the (consensus, $\mathcal F_t$-measurable) reference point $\hat{\bz}=\bz^t_\star$ gives $p_{\min}D_g(\tilde{\bx}^{t+1};\bz^t_\star)+\Et\big[V(\bz^{t+1}-\bz^t_\star)\big]\le V(\bz^t-\bz^t_\star)-Q_t$. Drop $D_g(\tilde{\bx}^{t+1};\bz^t_\star)\ge0$ (Corollary~\ref{cor:lyapunov-descent}), then use optimality of $\bz^{t+1}_\star$ over $\mathcal Z^\star\ni \bz^t_\star$, $V(\bz^{t+1}-\bz^{t+1}_\star)\le V(\bz^{t+1}-\bz^t_\star)$ surely, to get
\begin{equation}
\label{eq:S-recursion}
\Et[S_{t+1}]\ \le\ S_t-Q_t .
\end{equation}
Combining \eqref{eq:disagreement-bound} and \eqref{eq:r-via-Q},
\begin{equation}
\label{eq:delta-recursion}
\Et[\delta_{t+1}^2]\ \le\ \Big(1-\frac{\kappa_p}2\Big)\delta_t^2+C_1\|\tilde{\bx}^{t+1}-\bx^t\|^2\ \le\ \Big(1-\frac{\kappa_p}2\Big)\delta_t^2+C_2Q_t .
\end{equation}
Combining \eqref{eq:S-recursion} and \eqref{eq:delta-recursion} with $\theta$, we obtain
\begin{equation}
\label{eq:joint-recursion}
\Et[S_{t+1}+\theta\delta_{t+1}^2]\ \le\ S_t+\theta\Big(1-\frac{\kappa_p}2\Big)\delta_t^2-(1-\theta C_2)Q_t .
\end{equation}
It remains to prove a lower bound on $Q_t$. By \Cref{lem:Dp0},
\begin{equation}
\label{eq:St-triangle}
S_t^{1/2}\ =\ V(\bz^t-\bz^t_\star)^{1/2}
\ \le\ 2\,V(\bz^t-\tilde{\bz}^{t+1})^{1/2}+V\big(\tilde{\bz}^{t+1}-P_{\mathcal Z^\star}^{V}(\tilde{\bz}^{t+1})\big)^{1/2} ,
\end{equation}
where we recall $P_{\mathcal Z^\star}^{V}(\mathbf{u}):=\arg\min_{\mathbf{v}\in\mathcal Z^\star}V(\mathbf{u}-\mathbf{v})$. By optimality of $P_{\mathcal Z^\star}^{V}(\tilde{\bz}^{t+1})$ against the Euclidean projection $\mathbf{u}=P_{\mathcal Z^\star}(\tilde{\bz}^{t+1})\in\mathcal Z^\star$, where $P_{\mathcal Z^\star}(\mathbf{u}):=\arg\min_{\mathbf{v}\in\mathcal Z^\star}\|\mathbf{u}-\mathbf{v}\|^2$,
\[
V\big(\tilde{\bz}^{t+1}-P_{\mathcal Z^\star}^{V}(\tilde{\bz}^{t+1})\big)\ \le\
V\big(\tilde{\bz}^{t+1}-P_{\mathcal Z^\star}(\tilde{\bz}^{t+1})\big)\ \le\
M_V\,\dist(\tilde{\bz}^{t+1},\mathcal Z^\star)^2 ,
\]
using $\|\mathbf{u}-P_{\mathcal Z^\star}(\mathbf{u})\|=\dist(\mathbf{u},\mathcal Z^\star)$ and, for the last inequality, $V(\cdot)\le M_V\|\cdot\|^2$ (\Cref{lem:cV-CV}). Also, directly from this bound, $V(\bz^t-\tilde{\bz}^{t+1})\le M_V\|\bz^t-\tilde{\bz}^{t+1}\|^2$, and by \Cref{lem:distance-bound}, $\dist(\tilde{\bz}^{t+1},\mathcal Z^\star)\le L_F\|H-M\|\,\|\bz^t-\tilde{\bz}^{t+1}\|$. Substituting both bounds into \eqref{eq:St-triangle},
\[
S_t^{1/2}\ \le\ 2\sqrt{M_V}\,\|\bz^t-\tilde{\bz}^{t+1}\|+\sqrt{M_V}\,L_F\|H-M\|\,\|\bz^t-\tilde{\bz}^{t+1}\|
\ =\ \big(2+L_F\|H-M\|\big)\sqrt{M_V}\,\|\bz^t-\tilde{\bz}^{t+1}\| .
\]
Finally, by \Cref{lem:zbar-disagreement}, $\|\bz^t-\tilde{\bz}^{t+1}\|^2\le c_A\delta_t^2+c_BQ_t$, so
\begin{equation}
\label{eq:S-via-delta-Q}
S_t\ \le\ C_3\big(c_A\delta_t^2+c_BQ_t\big),\qquad\text{i.e.}\qquad Q_t\ \ge\ \frac{S_t-C_3c_A\delta_t^2}{C_3c_B} .
\end{equation}
From the definition of $\theta$, we have $4c_A=\theta(4c_AC_2+\kappa_pc_B)$, and thus
\[
1-\theta C_2\ =\ \frac{\kappa_pc_B}{4c_AC_2+\kappa_pc_B}\ \in\ (0,1) ,
\]
since $4c_AC_2>0$ and $\kappa_pc_B>0$ for every $\rho>0$. Combining \eqref{eq:joint-recursion} with the $Q_t$ bound \eqref{eq:S-via-delta-Q}, since $1-\theta C_2>0$,
\[
\Et[S_{t+1}+\theta\delta_{t+1}^2]\ \le\ S_t\Big[1-\frac{1-\theta C_2}{C_3c_B}\Big]+\delta_t^2\Big[\theta\Big(1-\frac{\kappa_p}2\Big)+\frac{(1-\theta C_2)c_A}{c_B}\Big] ,
\]
which is the claim.
\end{proof}

\subsection{Completing the proof of Theorem~\ref{thm:gamma1-gossip}}

\begin{proof}
We will now prove that, by choosing the right constant $c$ and using \Cref{lem:joint-contraction}, we can show geometric contraction. Recall $\theta:=4c_A/D$, where $D:=4c_AC_2+\kappa_pc_B$ (\Cref{lem:joint-contraction}), so $\theta C_2=4c_AC_2/D$ and $1-\theta C_2=\kappa_pc_B/D$; hence
\begin{equation}
\label{eq:theta-ratio}
\frac{(1-\theta C_2)c_A}{\theta c_B}\ =\ \frac{\kappa_pc_B/D\cdot c_A}{(4c_A/D)\cdot c_B}
\ =\ \frac{\kappa_pc_Bc_A}{4c_Ac_B}\ =\ \frac{\kappa_p}4 .
\end{equation}
Using \Cref{lem:joint-contraction} with \eqref{eq:theta-ratio}, we obtain
\begin{equation}
\label{eq:main-bound}
\Et[S_{t+1}+\theta\delta_{t+1}^2]\ \le\ S_t\Big(1-\frac{1-\theta C_2}{C_3c_B}\Big)+\delta_t^2\,\theta\Big(1-\frac{\kappa_p}4\Big) .
\end{equation}
Set the constant
\[
c\ :=\ \min\Big\{\frac{\kappa_p}4,\ \frac{1-\theta C_2}{C_3c_B}\Big\}\ \in\ (0,1) .
\]
By definition of the min, both coefficients in the bound \eqref{eq:main-bound} are $\le1-c$, so, with $U_t:=S_t+\theta\delta_t^2$,
\[
\Et[U_{t+1}]\ \le\ (1-c)S_t+(1-c)\theta\delta_t^2\ =\ (1-c)U_t .
\]
Taking full expectation and iterating over $t$, $\E[U_t]\le(1-c)^tU_0$. Since $S_t\le U_t$ ($\theta,\delta_t^2\ge0$),
\[
\E[S_t]\ \le\ \E[U_t]\ \le\ (1-c)^tU_0\ =\ (1-c)^t\big(V(\bz^0-\bz^0_\star)+\theta\,\delta(\bx^0)^2\big) ,
\]
i.e., $\E\big[V(\bz^t-\bz^t_\star)\big]\le(1-c)^t\big(V(\bz^0-\bz^0_\star)+\theta\,\delta(\bx^0)^2\big)$. In particular, since $m_V\|\bx^t-\bx^t_\star\|^2\le S_t$ (\Cref{lem:cV-CV}), $\bx^t-\bx^t_\star\to0$ in expectation, geometrically fast, which completes the proof.
\end{proof}

%% file: new_files/F_appendix.tex
\section{Proof of Proposition~\ref{prop:pairwise-averaging}}
\label{app:pairwise-averaging}

We prove the claim by induction on $t\ge0$, showing that the invariant $\bmu_i^t=\bx_i^t-\ba_i$ holds for every $i\in[n]$.

\textbf{Base case.} At $t=0$, $\bmu_i^0=0$ by \Cref{alg:goal-pd}, and $\bx_i^0=\ba_i$ by assumption, so $\bmu_i^0=\bx_i^0-\ba_i=0$.

\textbf{Inductive step.} Assume $\bmu_i^t=\bx_i^t-\ba_i$ for every $i\in[n]$, and let $e=e^{t+1}=(k,l)$ be the edge sampled at iteration $t+1$. For $i\in\{k,l\}$, since $f_i(\bx)=\|\bx-\ba_i\|^2/2$, the proximal operator admits the closed form $\prox_{\sigma_i,f_i}(\bv)=(\bv+\sigma_i\ba_i)/(1+\sigma_i)$. Using the invariant,
\[
\tilde{\bx}_i^{t+1}\ =\ \prox_{\sigma_i,f_i}\big(\bx_i^t+\sigma_i\bmu_i^t\big)
\ =\ \frac{\bx_i^t+\sigma_i(\bx_i^t-\ba_i)+\sigma_i\ba_i}{1+\sigma_i}
\ =\ \frac{(1+\sigma_i)\bx_i^t}{1+\sigma_i}\ =\ \bx_i^t \enspace,
\]
for every $\sigma_i>0$. Hence $\bar{\bx}_e^{t+1}=\frac12(\tilde{\bx}_k^{t+1}+\tilde{\bx}_l^{t+1})=\frac12(\bx_k^t+\bx_l^t)$, and, since \Cref{alg:goal-pd} sets $\bx_i^{t+1}=\bar{\bx}_e^{t+1}$ for $i\in\{k,l\}$, we obtain $\bx_k^{t+1}=\bx_l^{t+1}=\frac12(\bx_k^t+\bx_l^t)$, the classical pairwise-averaging update. Nodes $i\notin\{k,l\}$ remain idle, i.e., $\bx_i^{t+1}=\bx_i^t$ and $\bmu_i^{t+1}=\bmu_i^t$, so the invariant is trivially preserved for them.

It remains to verify the invariant for $i\in\{k,l\}$. Using $\rho=1$,
\[
\bmu_i^{t+1}\ =\ \bmu_i^t+\rho\big(\bar{\bx}_e^{t+1}-\tilde{\bx}_i^{t+1}\big)
\ =\ (\bx_i^t-\ba_i)+\big(\bar{\bx}_e^{t+1}-\bx_i^t\big)
\ =\ \bar{\bx}_e^{t+1}-\ba_i\ =\ \bx_i^{t+1}-\ba_i \enspace,
\]
which closes the induction. This completes the proof.

%% file: new_files/G_appendix.tex
\section{Additional Experiments for \texorpdfstring{Section~\ref{sec:synthetic}}{Section 4.1}}
\label{app:additional-experiments}

This section provides additional experimental details and results for the synthetic data experiments of \Cref{sec:synthetic}.

\subsection{Network Topologies}
\label{app:topologies}

Our experiments use five graph models: geometric, Watts--Strogatz, star, regular and Barab\'asi--Albert. The three random models are illustrated in \Cref{fig:graph-topology} for $n=101$ nodes. They differ in how they place edges, which gives them different degree distributions, distances and connectivity. For more information on graphs and their models, we refer the reader to~\citet{newman2018networks}.

\begin{figure}[H]
    \centering
    \includegraphics[width=\textwidth]{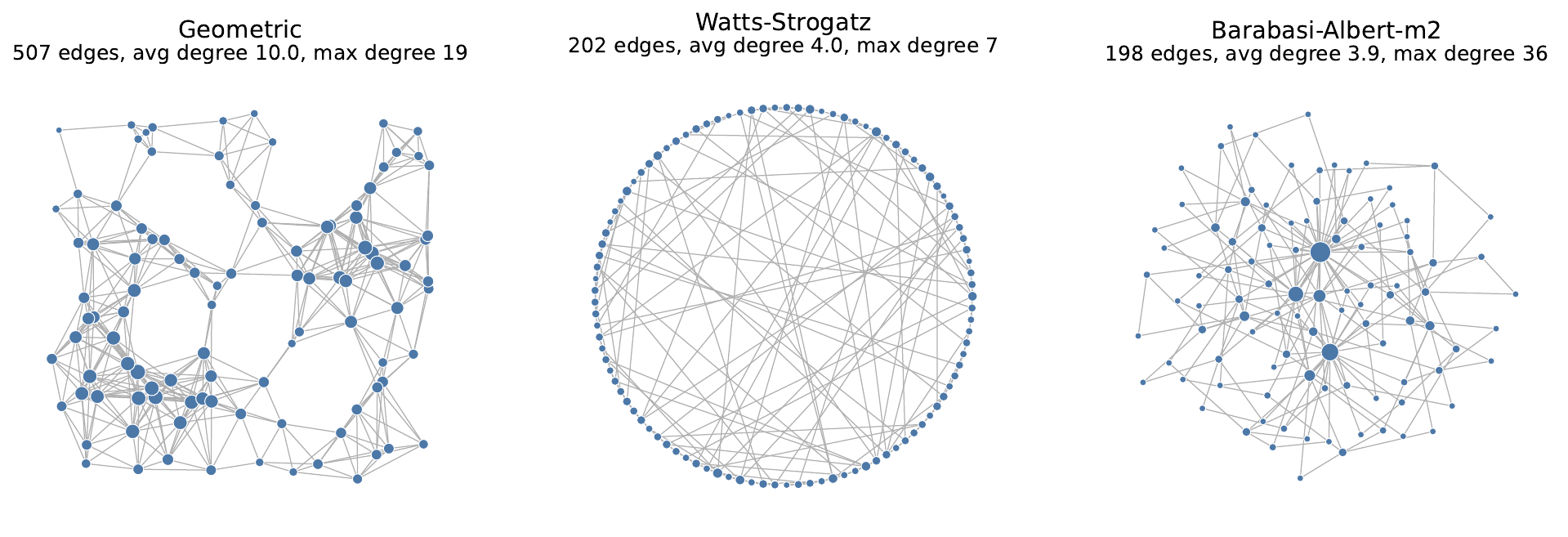}
    \caption{Three random graph models used in our experiments, each with $n=101$ nodes: a random geometric graph (left), a Watts--Strogatz graph with $k=4$ (middle), and a Barab\'asi--Albert graph with $m=2$ (right). Node size is proportional to degree.}
    \label{fig:graph-topology}
\end{figure}

A \emph{geometric} graph places nodes uniformly at random in the unit square and connects any two nodes within distance $r$ of each other; since edges only link nearby nodes, it is strongly clustered, as in wireless sensor networks. A \emph{Watts--Strogatz} graph starts from a ring lattice where each node is connected to its $k$ nearest neighbors, then randomly rewires each edge with probability $\beta$; the rewired edges act as long-range shortcuts, yielding a small-world graph with nearly uniform degrees. A \emph{Barab\'asi--Albert} graph is grown by preferential attachment, each new node connecting to $m$ existing nodes with probability proportional to their degree; this produces a scale-free graph with a power-law degree distribution, where a few hubs have very large degree.

\subsection{Additional Experiments on Larger Graphs}
\label{app:larger-graphs}

We repeat the median estimation experiment of \Cref{sec:synthetic} on larger networks with $n=501$ nodes, using the geometric, Watts--Strogatz and Barab\'asi--Albert models of \Cref{app:topologies}; all other settings are unchanged. \Cref{fig:larger-graphs} shows that Goal-PD converges linearly on all three topologies and, after $10^7$ timesteps, reaches an error roughly three orders of magnitude lower than AsyncADMM and DAPD.

\begin{figure}[H]
    \centering
    \begin{subfigure}{0.32\textwidth}
        \centering
        \includegraphics[width=\linewidth]{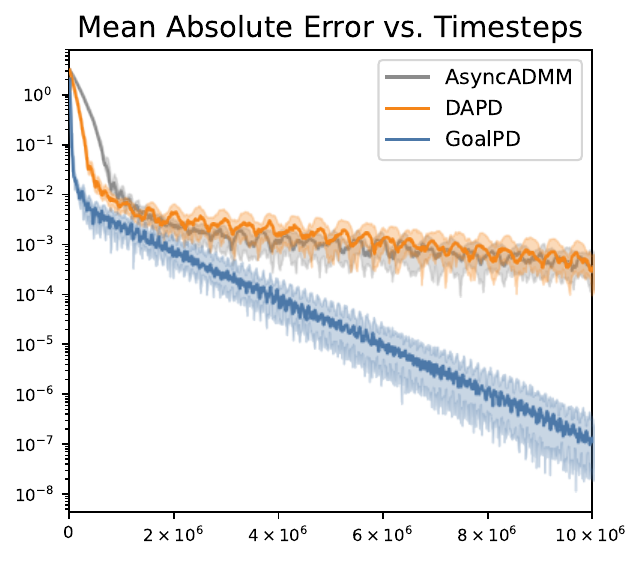}
        \caption{Geometric}
        \label{subfig:larger-geometric}
    \end{subfigure}\hfill
    \begin{subfigure}{0.32\textwidth}
        \centering
        \includegraphics[width=\linewidth]{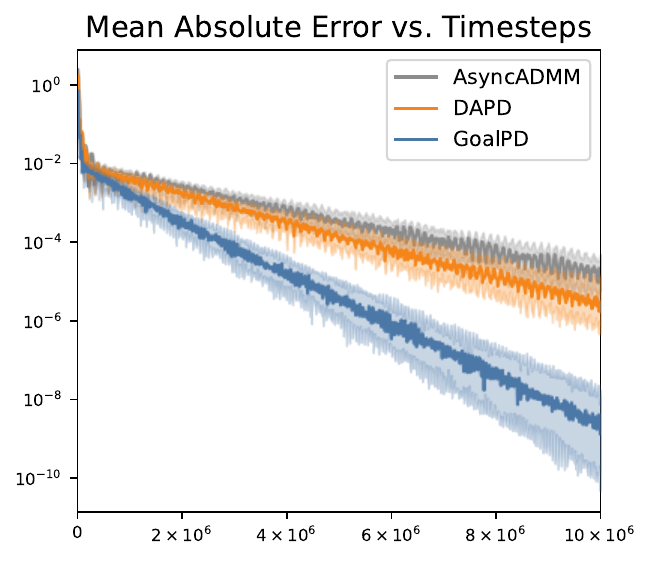}
        \caption{Watts--Strogatz}
        \label{subfig:larger-ws}
    \end{subfigure}\hfill
    \begin{subfigure}{0.32\textwidth}
        \centering
        \includegraphics[width=\linewidth]{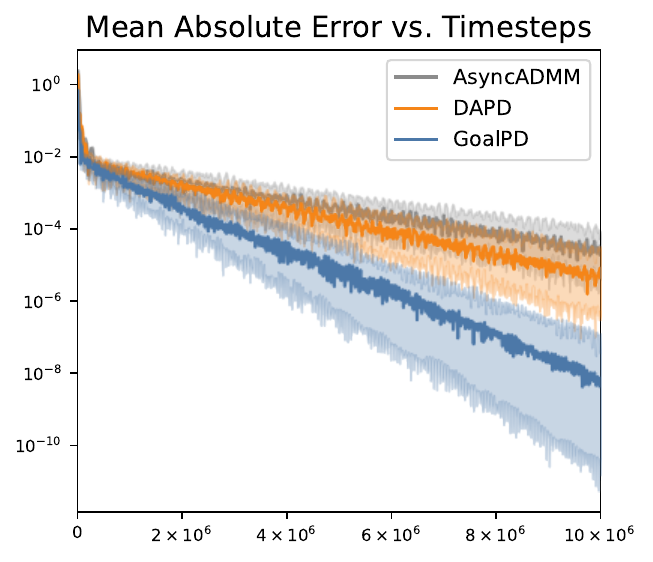}
        \caption{Barab\'asi--Albert ($m=2$)}
        \label{subfig:larger-ba}
    \end{subfigure}
    \caption{Convergence of Goal-PD, AsyncADMM and DAPD for median estimation on three network topologies with $n=501$ nodes, under contaminated Gaussian data. All plots report the absolute error averaged across nodes versus timesteps, as in \Cref{fig:first_fig}, with shaded regions indicating the interquartile range.}
    \label{fig:larger-graphs}
\end{figure}

We further extend the topology comparison of \Cref{subfig:topology} to larger networks, adding a Watts--Strogatz graph with average degree $10$ (Watts-Strogatz-k10) alongside the geometric and Watts--Strogatz ($k=4$) graphs to further illustrate the effect of connectivity. \Cref{fig:topology-size} shows the convergence of Goal-PD for median estimation on these three topologies, for networks of $n=501$, $1001$ and $2001$ nodes. As in \Cref{subfig:topology}, better-connected graphs converge faster: Watts-Strogatz-k10 outperforms Watts--Strogatz, which in turn outperforms the geometric graph, and this ordering is consistent across all three network sizes.

\begin{figure}[H]
    \centering
    \begin{subfigure}{0.32\textwidth}
        \centering
        \includegraphics[width=\linewidth]{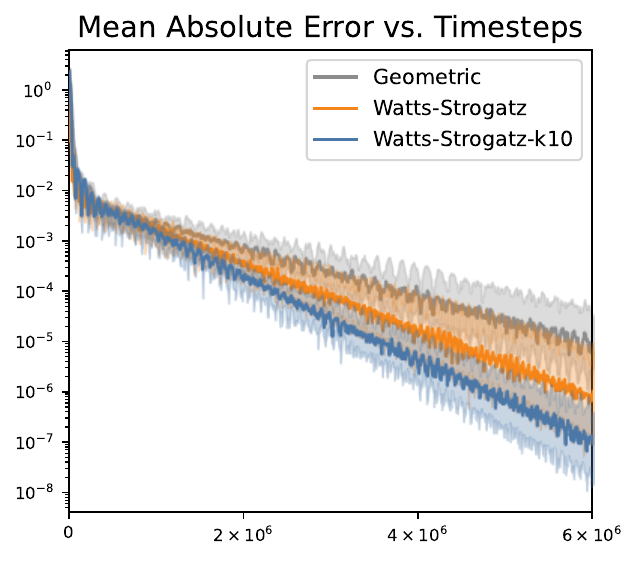}
        \caption{$n=501$}
        \label{subfig:topology-501}
    \end{subfigure}\hfill
    \begin{subfigure}{0.32\textwidth}
        \centering
        \includegraphics[width=\linewidth]{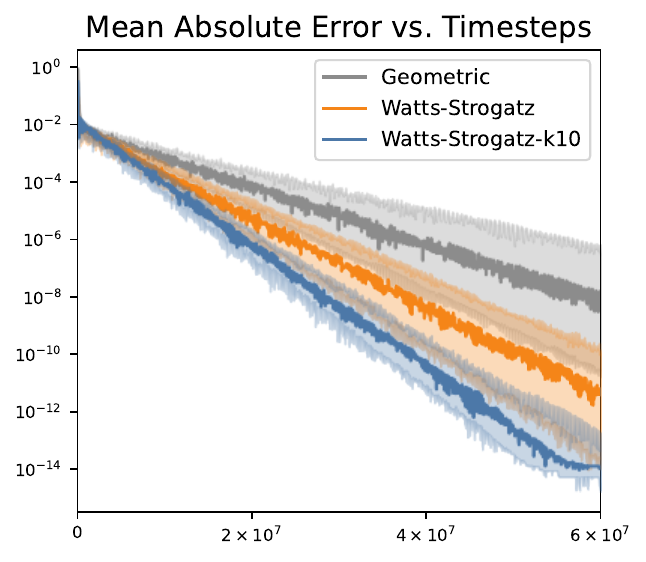}
        \caption{$n=1001$}
        \label{subfig:topology-1001}
    \end{subfigure}\hfill
    \begin{subfigure}{0.32\textwidth}
        \centering
        \includegraphics[width=\linewidth]{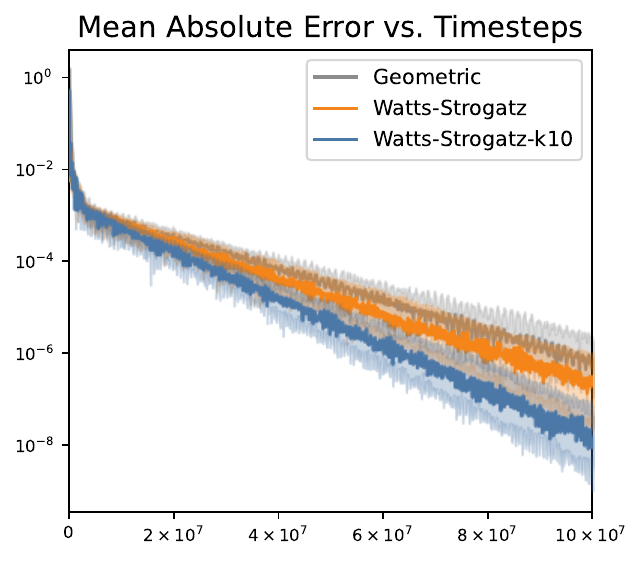}
        \caption{$n=2001$}
        \label{subfig:topology-2001}
    \end{subfigure}
    \caption{Convergence of Goal-PD for median estimation on geometric, Watts--Strogatz and Watts-Strogatz-k10 graphs, under contaminated Gaussian data, for increasing network sizes.}
    \label{fig:topology-size}
\end{figure}

\subsection{Ablation on the Margin Parameter \texorpdfstring{$\gamma$}{gamma}}
\label{app:gamma-ablation}

As shown in \Cref{app:gossip-purecd}, Goal-PD corresponds to the boundary case $\gamma=1$ of a gossip-based specialization of \textsc{Pure-CD}, whose almost-sure convergence guarantee formally requires $\gamma<1$. We empirically assess the effect of this margin by repeating the median estimation experiment of \Cref{sec:synthetic} (same setup as \Cref{subfig:median}, with $n=101$ nodes on a geometric graph) for $\gamma\in\{0.5,0.9,1\}$, under three choices of edge sampling law and step size: uniform edge sampling with the theoretical step size $\sigma_i^{\mathrm{th}}$ (\Cref{subfig:gamma-uniform}), and the degree-weighted edge sampling law $p_e=(1/n)(1/d_i+1/d_j)$ of \Cref{tab:stepsize} with either the theoretical step size $\sigma_i^{\mathrm{th}}$ (\Cref{subfig:gamma-degree-theoretical}) or the practical step size $\hat\sigma_i$ (\Cref{subfig:gamma-degree-practical}). \Cref{fig:gamma-ablation} shows that, in all three settings, convergence is fastest at $\gamma=1$ and slows as $\gamma$ decreases, confirming that despite falling outside \textsc{Pure-CD}'s theoretical guarantee, the boundary choice $\gamma=1$ used by Goal-PD is also the best-performing one in practice.

\begin{figure}[H]
    \centering
    \begin{subfigure}{0.32\textwidth}
        \centering
        \includegraphics[width=\linewidth]{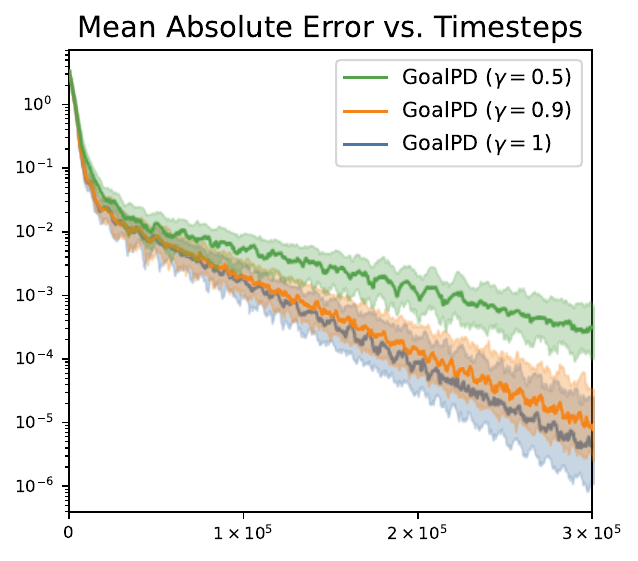}
        \caption{Uniform, $\sigma_i^{\mathrm{th}}$}
        \label{subfig:gamma-uniform}
    \end{subfigure}\hfill
    \begin{subfigure}{0.32\textwidth}
        \centering
        \includegraphics[width=\linewidth]{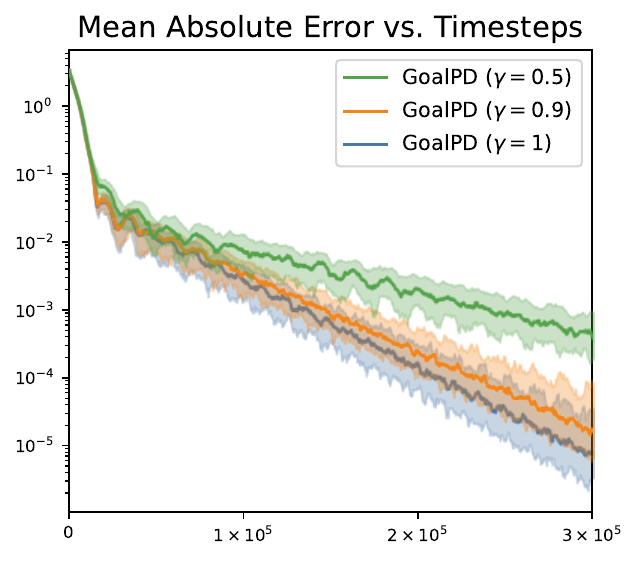}
        \caption{Degree, $\sigma_i^{\mathrm{th}}$}
        \label{subfig:gamma-degree-theoretical}
    \end{subfigure}\hfill
    \begin{subfigure}{0.32\textwidth}
        \centering
        \includegraphics[width=\linewidth]{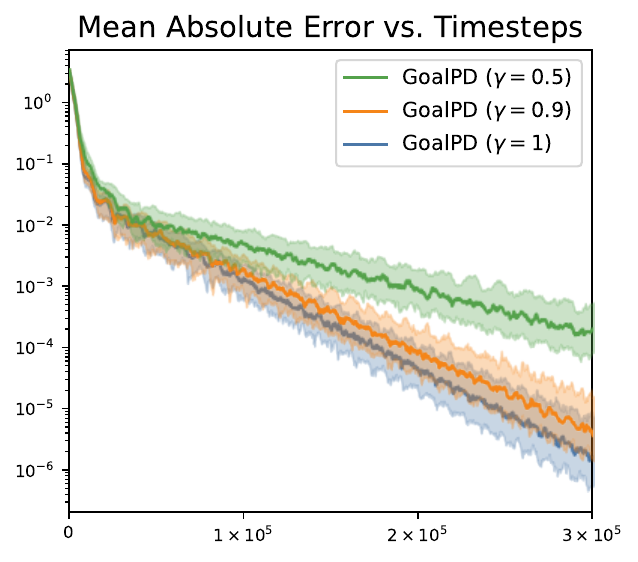}
        \caption{Degree, $\hat\sigma_i$}
        \label{subfig:gamma-degree-practical}
    \end{subfigure}
    \caption{Convergence of Goal-PD for median estimation on a geometric graph with $n=101$ nodes, for margin parameter $\gamma\in\{0.5,0.9,1\}$, under uniform edge sampling with the theoretical step size (a), and degree-weighted edge sampling with the theoretical (b) and practical (c) step sizes.}
    \label{fig:gamma-ablation}
\end{figure}

\subsection{Ablation on the Penalty Parameter \texorpdfstring{$\rho$}{rho}}
\label{app:rho-ablation}

We further study the sensitivity of Goal-PD to the scalar penalty parameter $\rho$ (\Cref{alg:goal-pd}), which is fixed to $\rho=1$ throughout the rest of the paper. We repeat the median estimation experiment of \Cref{sec:synthetic} (same setup as \Cref{subfig:median}, with $n=101$ nodes on a geometric graph) for $\rho\in\{1/10,1/5,1,5,10\}$, again under the three edge sampling and step size combinations of \Cref{app:gamma-ablation}: uniform edge sampling with the theoretical step size $\sigma_i^{\mathrm{th}}$ (\Cref{subfig:rho-uniform}), and the degree-weighted edge sampling law with either the theoretical step size $\sigma_i^{\mathrm{th}}$ (\Cref{subfig:rho-degree-theoretical}) or the practical step size $\hat\sigma_i$ (\Cref{subfig:rho-degree-practical}). \Cref{fig:rho-ablation} shows that $\rho=1$ converges the fastest and most robustly across all three settings. Taking $\rho>1$ makes the early iterations unstable, visible as a plateau followed by a sharp drop in the error curve, and this instability worsens as $\rho$ increases. Conversely, taking $\rho<1$ leads to slower convergence and larger variance across trials, an effect that also becomes more pronounced as $\rho$ decreases.

\begin{figure}[H]
    \centering
    \begin{subfigure}{0.32\textwidth}
        \centering
        \includegraphics[width=\linewidth]{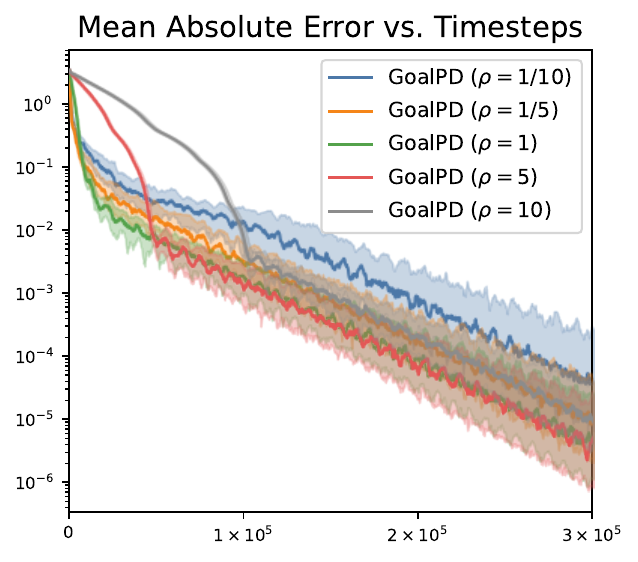}
        \caption{Uniform, $\sigma_i^{\mathrm{th}}$}
        \label{subfig:rho-uniform}
    \end{subfigure}\hfill
    \begin{subfigure}{0.32\textwidth}
        \centering
        \includegraphics[width=\linewidth]{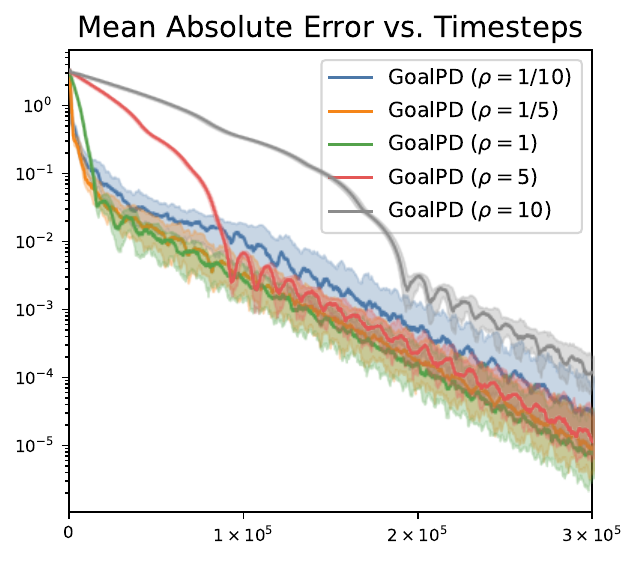}
        \caption{Degree, $\sigma_i^{\mathrm{th}}$}
        \label{subfig:rho-degree-theoretical}
    \end{subfigure}\hfill
    \begin{subfigure}{0.32\textwidth}
        \centering
        \includegraphics[width=\linewidth]{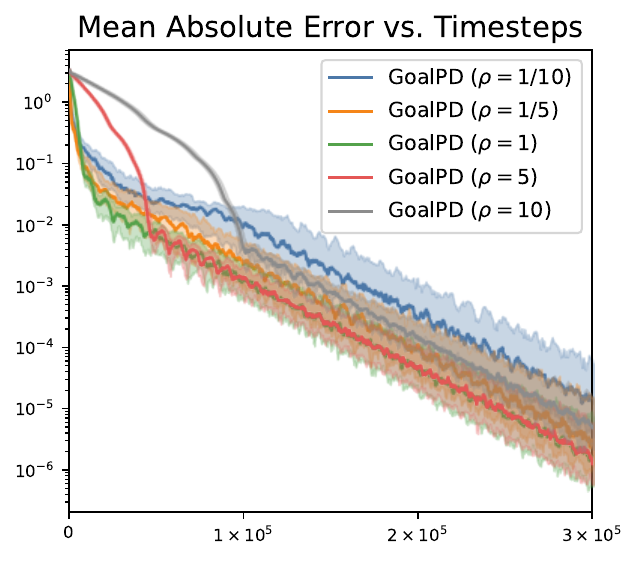}
        \caption{Degree, $\hat\sigma_i$}
        \label{subfig:rho-degree-practical}
    \end{subfigure}
    \caption{Convergence of Goal-PD for median estimation on a geometric graph with $n=101$ nodes, for penalty parameter $\rho\in\{1/10,1/5,1,5,10\}$, under uniform edge sampling with the theoretical step size (a), and degree-weighted edge sampling with the theoretical (b) and practical (c) step sizes.}
    \label{fig:rho-ablation}
\end{figure}

We then compare the sensitivity of Goal-PD, AsyncADMM and DAPD to the choice of $\rho$. We repeat the same median estimation experiment under uniform edge sampling, but in each of the $100$ trials we draw $\rho$ uniformly at random from a fixed interval and use this same value for all three methods. We consider the intervals $[0.1,1]$ (\Cref{subfig:rho-methods-01-1}), $[0.5,2]$ (\Cref{subfig:rho-methods-05-2}) and $[0.1,10]$ (\Cref{subfig:rho-methods-01-10}). \Cref{fig:rho-methods} shows that the convergence of Goal-PD is very robust to the choice of $\rho$: its error curve remains largely unaffected across the three intervals, reaching an error of order $10^{-5}$ after $3\times10^5$ timesteps in every case. In contrast, AsyncADMM and DAPD are much more sensitive to $\rho$ and become dramatically slower when larger values are allowed: on $[0.1,10]$, AsyncADMM does not get below an error of $10^{-1}$, and DAPD stays around $10^{-2}$ with a much wider spread across trials.

\begin{figure}[H]
    \centering
    \begin{subfigure}{0.32\textwidth}
        \centering
        \includegraphics[width=\linewidth]{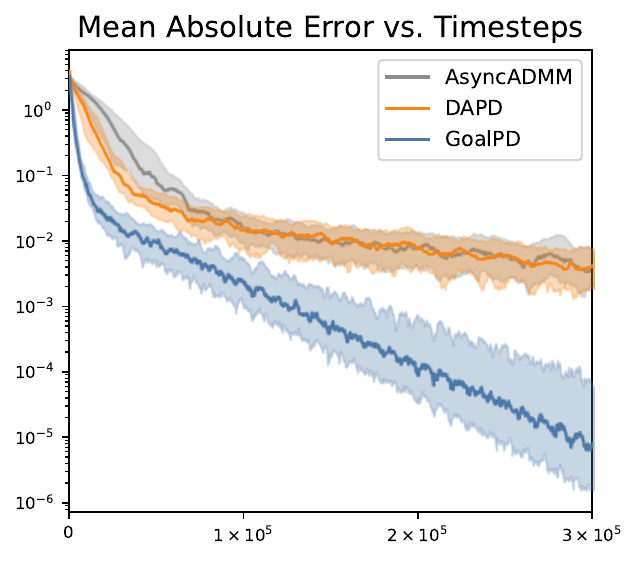}
        \caption{$\rho\sim\mathcal U([0.1,1])$}
        \label{subfig:rho-methods-01-1}
    \end{subfigure}\hfill
    \begin{subfigure}{0.32\textwidth}
        \centering
        \includegraphics[width=\linewidth]{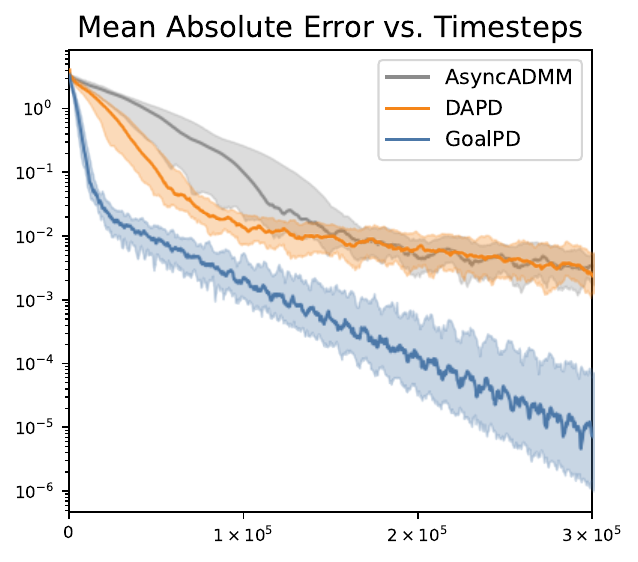}
        \caption{$\rho\sim\mathcal U([0.5,2])$}
        \label{subfig:rho-methods-05-2}
    \end{subfigure}\hfill
    \begin{subfigure}{0.32\textwidth}
        \centering
        \includegraphics[width=\linewidth]{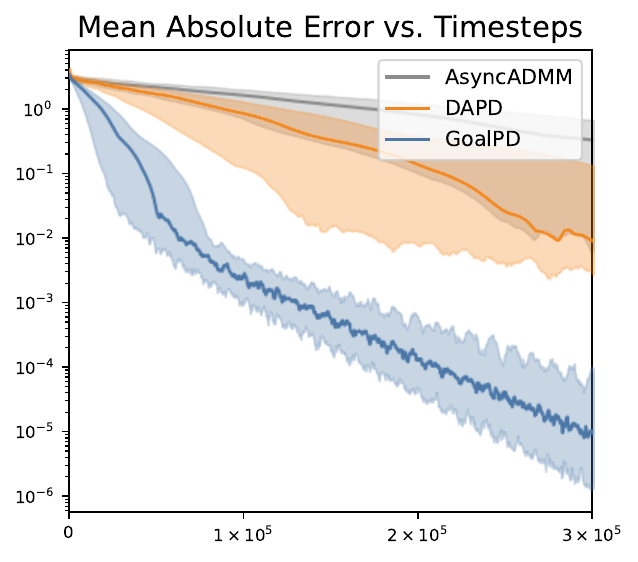}
        \caption{$\rho\sim\mathcal U([0.1,10])$}
        \label{subfig:rho-methods-01-10}
    \end{subfigure}
    \caption{Convergence of Goal-PD, AsyncADMM and DAPD for median estimation on a geometric graph with $n=101$ nodes under uniform edge sampling, where in each trial the penalty parameter $\rho$ is drawn uniformly at random from $[0.1,1]$ (a), $[0.5,2]$ (b) and $[0.1,10]$ (c), and shared by all three methods.}
    \label{fig:rho-methods}
\end{figure}

\subsection{Convergence under the Theoretical and Practical Step Sizes}
\label{app:sigma-ablation}

\Cref{tab:stepsize} compares the theoretical step size $\sigma_i^{\mathrm{th}}:=p_{\min}/(\rho\pi_i)$ and the practical step size $\hat\sigma_i:=1/(\rho d_i)$.

\begin{figure}[H]
    \centering
    \begin{subfigure}{0.32\textwidth}
        \centering
        \includegraphics[width=\linewidth]{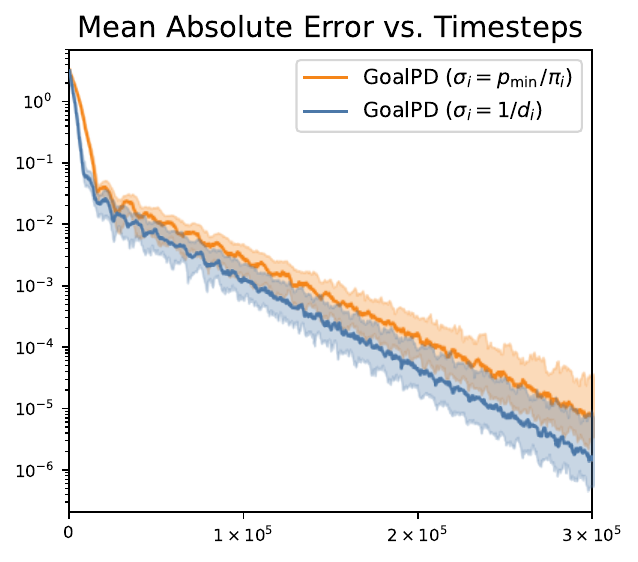}
        \caption{Geometric}
        \label{subfig:sigma-geometric}
    \end{subfigure}\hfill
    \begin{subfigure}{0.32\textwidth}
        \centering
        \includegraphics[width=\linewidth]{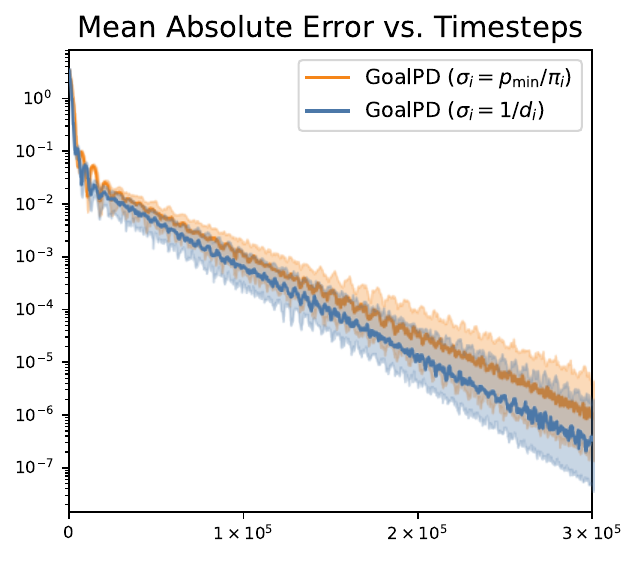}
        \caption{Watts--Strogatz}
        \label{subfig:sigma-ws}
    \end{subfigure}\hfill
    \begin{subfigure}{0.32\textwidth}
        \centering
        \includegraphics[width=\linewidth]{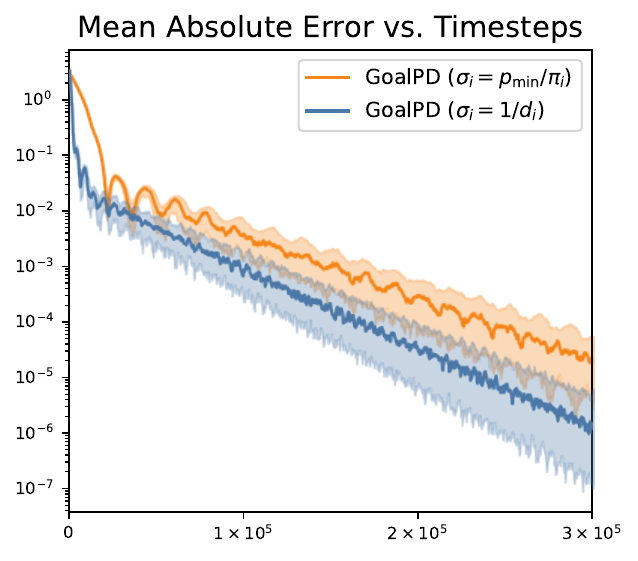}
        \caption{Barab\'asi--Albert ($m=2$)}
        \label{subfig:sigma-ba}
    \end{subfigure}
    \caption{Convergence of Goal-PD for median estimation under the theoretical step size $\sigma_i^{\mathrm{th}}:=p_{\min}/(\rho\pi_i)$ and the practical step size $\hat\sigma_i:=1/(\rho d_i)$, with the degree-weighted edge sampling law of \Cref{tab:stepsize}, on three network topologies with $n=101$ nodes.}
    \label{fig:sigma-ablation}
\end{figure}

\Cref{fig:sigma-ablation} compares the impact of the theoretical and practical step sizes across the three network topologies, each with $n=101$ nodes: geometric (\Cref{subfig:sigma-geometric}), Watts--Strogatz (\Cref{subfig:sigma-ws}), and Barab\'asi--Albert ($m=2$) (\Cref{subfig:sigma-ba}). The practical step size $\hat\sigma_i$ converges faster than the theoretical step size $\sigma_i^{\mathrm{th}}$ on all three topologies, consistent with the final-iterate comparison of \Cref{tab:stepsize}, and this gain is even more significant on the geometric and Barab\'asi--Albert graphs, which are known to have high degree heterogeneity, than on the Watts--Strogatz graph, whose degrees are nearly uniform.

\subsection{Recovery of Pairwise Gossip for Mean Estimation}
\label{app:pairwise-gossip-ablation}

\Cref{prop:pairwise-averaging} shows that, for the mean-estimation instance $f_i(\bx)=\|\bx-\ba_i\|^2/2$ and $\rho=1$, Goal-PD reduces exactly to the classical pairwise-averaging gossip algorithm, for any choice of $\sigma_i>0$. \Cref{fig:mean-rho-ablation} illustrates this equivalence empirically, showing the convergence of Goal-PD to the mean, for $\rho\in\{1/10,1/2,1,2,10\}$, against classical pairwise gossip, on a cycle (\Cref{subfig:mean-rho-cycle}), a geometric graph (\Cref{subfig:mean-rho-geometric}), and a Watts--Strogatz graph (\Cref{subfig:mean-rho-ws}), each with $n=101$ nodes.

\begin{figure}[H]
    \centering
    \begin{subfigure}{0.32\textwidth}
        \centering
        \includegraphics[width=\linewidth]{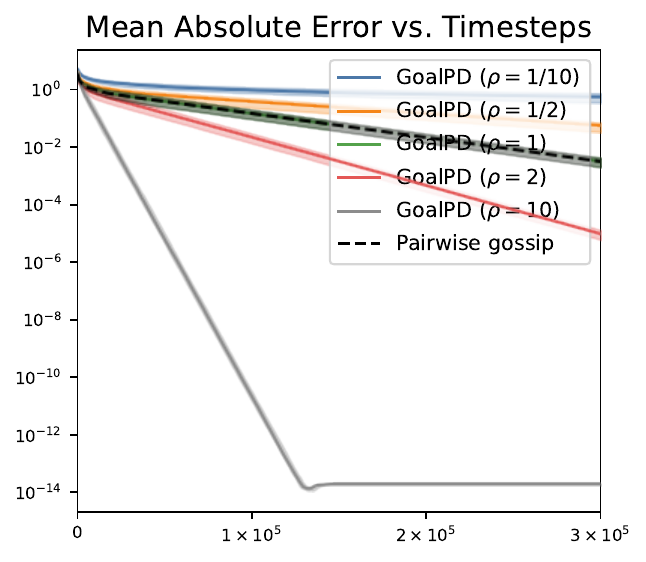}
        \caption{Cycle}
        \label{subfig:mean-rho-cycle}
    \end{subfigure}\hfill
    \begin{subfigure}{0.32\textwidth}
        \centering
        \includegraphics[width=\linewidth]{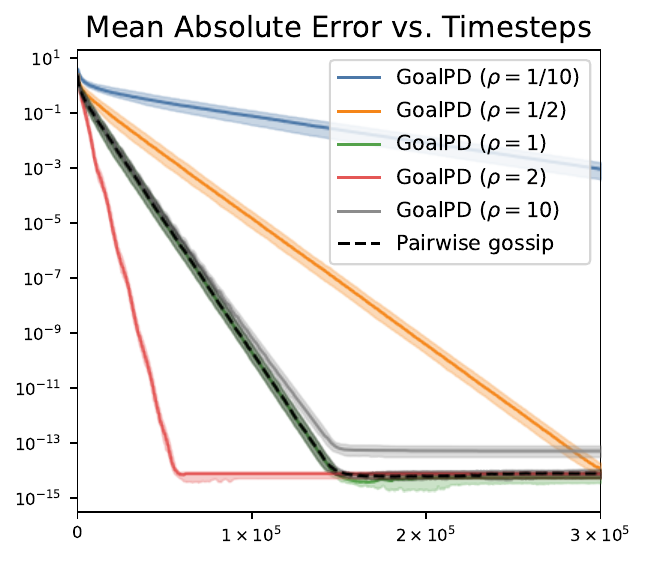}
        \caption{Geometric}
        \label{subfig:mean-rho-geometric}
    \end{subfigure}\hfill
    \begin{subfigure}{0.32\textwidth}
        \centering
        \includegraphics[width=\linewidth]{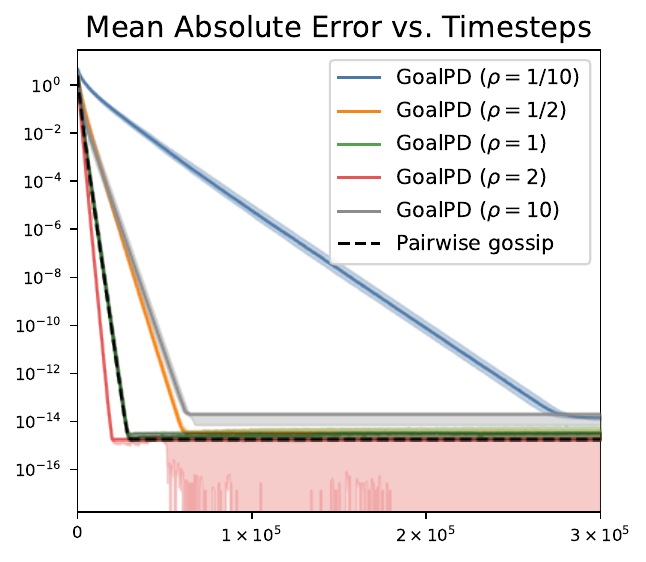}
        \caption{Watts--Strogatz}
        \label{subfig:mean-rho-ws}
    \end{subfigure}
    \caption{Convergence of Goal-PD to the mean, for penalty parameter $\rho\in\{1/10,1/2,1,2,10\}$, against classical pairwise gossip, on three network topologies with $n=101$ nodes.}
    \label{fig:mean-rho-ablation}
\end{figure}

As predicted by \Cref{prop:pairwise-averaging}, the $\rho=1$ curve of Goal-PD coincides exactly with classical pairwise gossip on all three topologies, while every other value of $\rho$ departs from it, confirming that $\rho=1$ is essential for this equivalence to hold.

%% file: new_files/H_appendix.tex
\section{Additional Experiments for \texorpdfstring{Section~\ref{sec:wine}}{Section 4.2}}
\label{app:wine-experiments}

This section provides additional experimental details and results for the Wine Quality dataset experiments of \Cref{sec:wine}.

\subsection{Proximal Operator of the Geometric Median}
\label{app:geometric-median-prox}

We derive the closed-form proximal operator used for the geometric median task of \Cref{sec:wine}. Writing $\ba_i\in\mathbb R^{11}$ for the feature vector held by node $i$, the corresponding local objective is $f_i(\bx)=\|\bx-\ba_i\|$, i.e., $f_i=g(\cdot-\ba_i)$ for $g:=\|\cdot\|$ the Euclidean norm. We first recall the proximal operator of $g$~\citep{parikh2014proximal}. For $\gamma>0$ and $\bv\in\mathbb R^{11}$,
\[
\prox_{\gamma,g}(\bv) = \Big(1-\frac\gamma{\|\bv\|}\Big)_+\bv = \begin{cases}\Big(1-\dfrac\gamma{\|\bv\|}\Big)\bv, & \|\bv\|\ge\gamma \\ 0, & \|\bv\|<\gamma \, .\end{cases}
\]
By the precomposition property of the proximal operator~\citep{parikh2014proximal}, since $f_i(\bx)=g(\bx-\ba_i)$, we have $\prox_{\gamma,f_i}(\bz)=\ba_i+\prox_{\gamma,g}(\bz-\ba_i)$ for every $\bz\in\mathbb R^{11}$, which gives the closed-form expression
\[
\prox_{\gamma,f_i}(\bz) = \ba_i + \Big(1-\frac{\gamma}{\|\bz-\ba_i\|}\Big)_+(\bz-\ba_i) \enspace.
\]

\subsection{Proximal Operator of the Median Regression Loss}
\label{app:median-regression-prox}

To account for an intercept, we augment each feature vector $\ba_i\in\mathbb R^{11}$ with a constant coordinate, $\tilde\ba_i:=(\ba_i^\top,1)^\top\in\mathbb R^{12}$, and write the regression coefficient as $\bx=(\bw^\top,c)^\top\in\mathbb R^{12}$, where $\bw\in\mathbb R^{11}$ collects the slopes and $c\in\mathbb R$ is the intercept, so that $\inprod{\tilde\ba_i}{\bx}=\inprod{\ba_i}{\bw}+c$. The local objective of node $i$ is then $f_i(\bx)=L_{1/2}\big(b_i-\inprod{\tilde\ba_i}{\bx}\big)$ with $L_{1/2}=\tfrac12|\cdot|$. Since $f_i=g\circ\mathcal A$ for the affine map $\mathcal A(\bx)=b_i-\inprod{\tilde\ba_i}{\bx}$ and $g=L_{1/2}$ the pinball loss, which is piecewise linear-quadratic~\citep{rockafellar1998variational}, and since piecewise linear-quadratic functions are closed under precomposition with an affine mapping~\citep{rockafellar1998variational}, $f_i$ is itself piecewise linear-quadratic. We derive $\prox_{\gamma,f_i}$ from the rule for composition with an affine mapping~\citep[Theorem~6.15]{beck2017first}: if $f(\bx)=g(\mathcal A(\bx))$ with $\mathcal A\mathcal A^\top=\alpha I$ for some $\alpha>0$, then $\prox_f(\bz)=\bz+\frac1\alpha\mathcal A^\top\big(\prox_{\alpha g}(\mathcal A(\bz))-\mathcal A(\bz)\big)$. Since $L_{1/2}$ is symmetric, $\gamma f_i(\bx)=\gamma L_{1/2}\big(\inprod{\tilde\ba_i}{\bx}-b_i\big)$, so we apply the theorem with $f=\gamma f_i$, $\mathcal A(\bx)=\inprod{\tilde\ba_i}{\bx}$, $g(u)=\gamma L_{1/2}(u-b_i)$ and $\alpha=\mathcal A\mathcal A^\top=\|\tilde\ba_i\|^2=\|\ba_i\|^2+1$. The function $\alpha g(u)=\frac{\gamma\alpha}{2}|u-b_i|$ is a translated absolute value, whose proximal operator is soft-thresholding around $b_i$, $\prox_{\alpha g}(u)=u-\operatorname{clip}\big(u-b_i,-\frac{\gamma\alpha}2,\frac{\gamma\alpha}2\big)$, where $\operatorname{clip}(u,l,h):=\min\{h,\max\{l,u\}\}$. Substituting $u=\inprod{\tilde\ba_i}{\bz}$ and using $\prox_{\gamma,f_i}=\prox_{\gamma f_i}$, we obtain, for every $\bz\in\mathbb R^{12}$,
\[
\prox_{\gamma,f_i}(\bz) = \bz + \frac{1}{\alpha}\,\operatorname{clip}\Big(b_i-\inprod{\tilde\ba_i}{\bz},\, -\frac{\gamma\alpha}{2},\, \frac{\gamma\alpha}{2}\Big)\,\tilde\ba_i \enspace ,
\]
with $\alpha=\|\tilde\ba_i\|^2=\|\ba_i\|^2+1$.

\subsection{Effect of Data Corruption on the Mean and Median Estimators}
\label{app:wine-pca}

To illustrate why robust estimators are relevant for the Wine Quality experiments of \Cref{sec:wine}, \Cref{fig:wine_pca} visualizes the dataset through its two-dimensional PCA projection, for three levels of corruption $\varepsilon\in\{0,0.2,0.4\}$, where $\varepsilon$ denotes the fraction of corrupted samples. Given $n=6497$ samples, we draw $\lfloor\varepsilon n\rfloor$ samples uniformly at random without replacement and multiply their feature vectors by $5$, leaving the remaining samples unchanged. Each panel shows the uncorrupted and corrupted data points, together with the empirical mean, the coordinatewise median, the geometric median, and the mean of the uncorrupted data. As $\varepsilon$ increases, the empirical mean is pulled increasingly far away from the mean of the uncorrupted data, towards the corrupted points, whereas the coordinatewise and geometric medians remain much closer to it.

\begin{figure}[H]
    \centering
    \begin{subfigure}{0.32\textwidth}
        \centering
        \includegraphics[width=\linewidth]{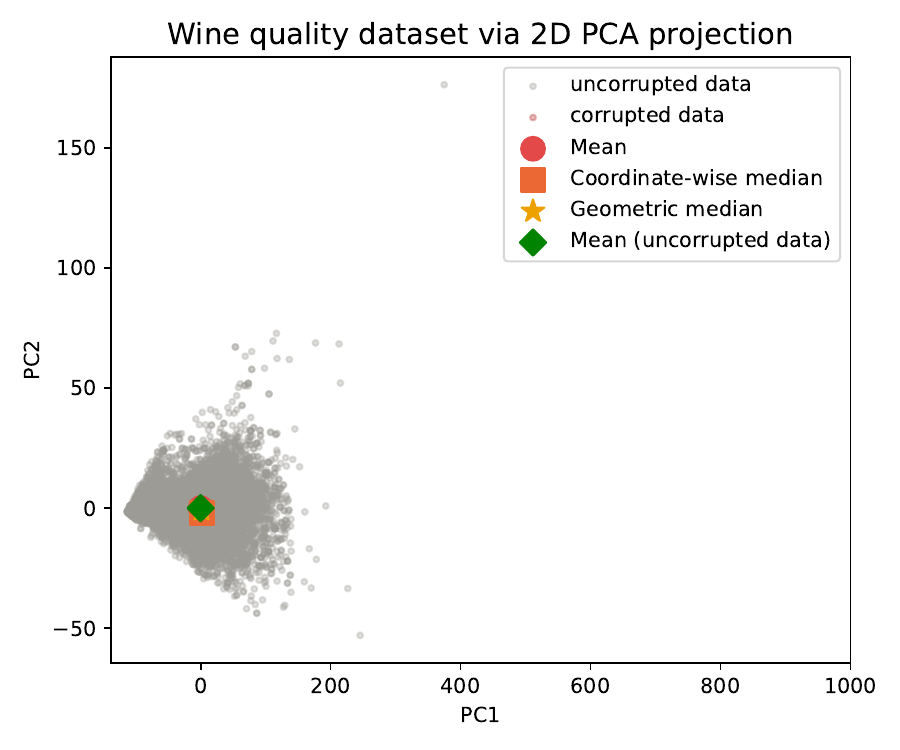}
        \caption{$\varepsilon=0$}
    \end{subfigure}\hfill
    \begin{subfigure}{0.32\textwidth}
        \centering
        \includegraphics[width=\linewidth]{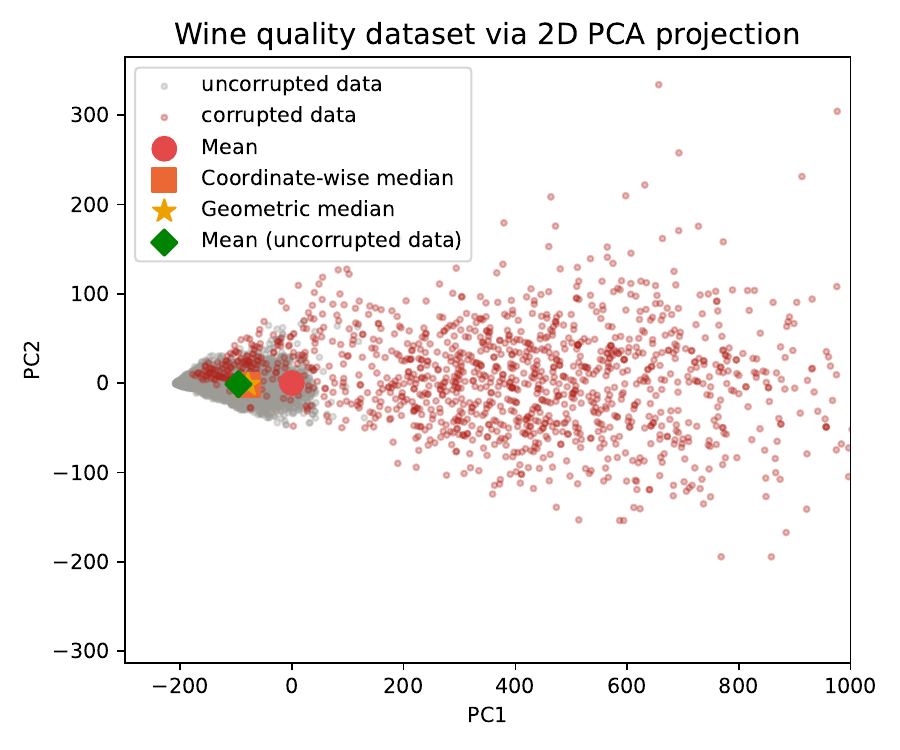}
        \caption{$\varepsilon=0.2$}
    \end{subfigure}\hfill
    \begin{subfigure}{0.32\textwidth}
        \centering
        \includegraphics[width=\linewidth]{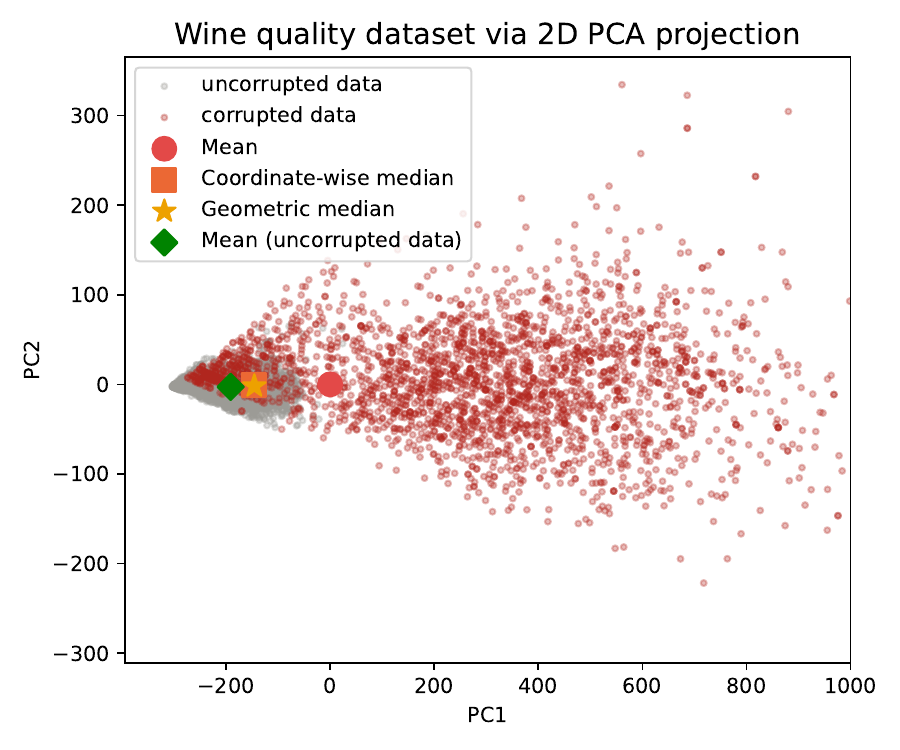}
        \caption{$\varepsilon=0.4$}
    \end{subfigure}
    \caption{Two-dimensional PCA projection of the Wine Quality dataset for increasing corruption levels $\varepsilon$, where a fraction $\varepsilon$ of the samples, drawn uniformly at random, has its feature vector multiplied by $5$.}
    \label{fig:wine_pca}
\end{figure}

\subsection{Sensitivity to the Step Size $\rho$ under Corruption}
\label{app:wine-rho}

\Cref{fig:wine_rho} reports the convergence of Goal-PD on the coordinatewise median task of the Wine Quality dataset for $\rho\in\{0.5,1,2\}$, under the three corruption levels $\varepsilon\in\{0,0.2,0.4\}$ of \Cref{app:wine-pca}.

\begin{figure}[H]
    \centering
    \begin{subfigure}{0.32\textwidth}
        \centering
        \includegraphics[height=4.2cm, width=\linewidth]{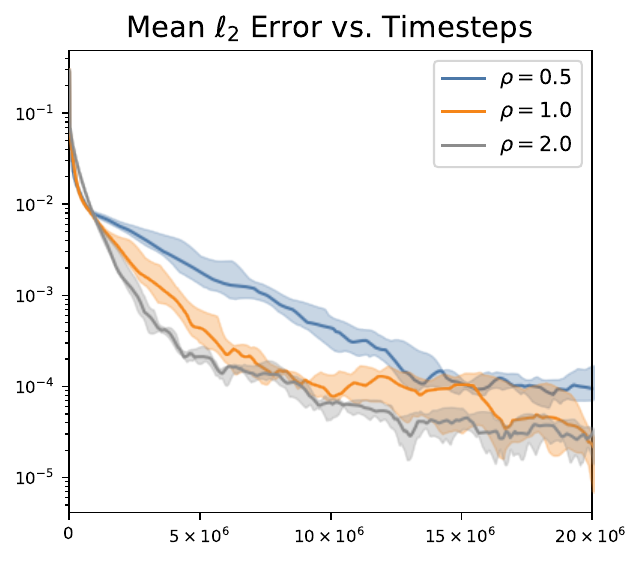}
        \caption{$\varepsilon=0$}
    \end{subfigure}\hfill
    \begin{subfigure}{0.32\textwidth}
        \centering
        \includegraphics[height=4.2cm, width=\linewidth]{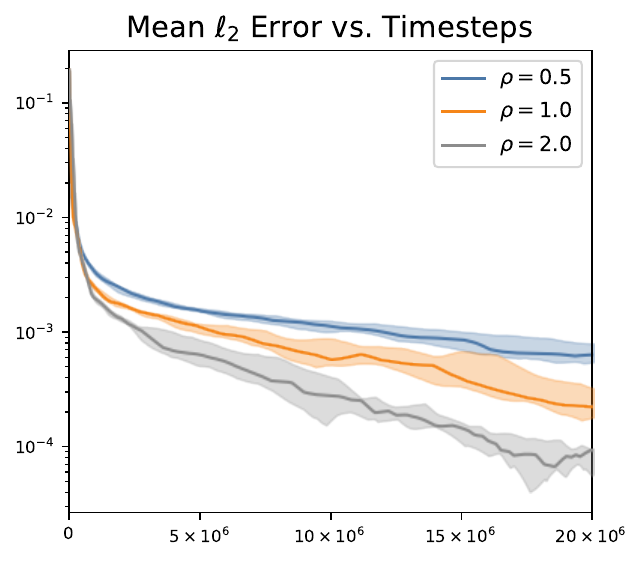}
        \caption{$\varepsilon=0.2$}
    \end{subfigure}\hfill
    \begin{subfigure}{0.32\textwidth}
        \centering
        \includegraphics[height=4.2cm, width=\linewidth]{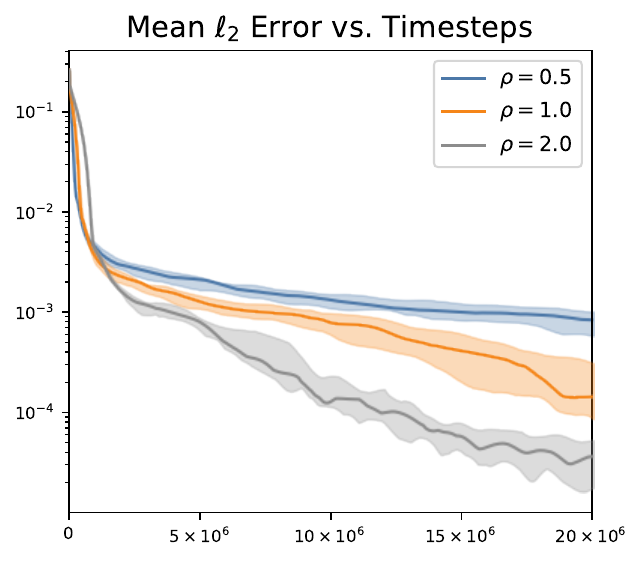}
        \caption{$\varepsilon=0.4$}
    \end{subfigure}
    \caption{Convergence of Goal-PD for the coordinatewise median on the Wine Quality dataset for $\rho\in\{0.5,1,2\}$, under increasing corruption levels $\varepsilon$. All plots report the mean $\ell_2$ error versus timesteps.}
    \label{fig:wine_rho}
\end{figure}

\Cref{fig:wine_rho_geo} reports the same ablation for the geometric median.

\begin{figure}[H]
    \centering
    \begin{subfigure}{0.32\textwidth}
        \centering
        \includegraphics[height=4.2cm, width=\linewidth]{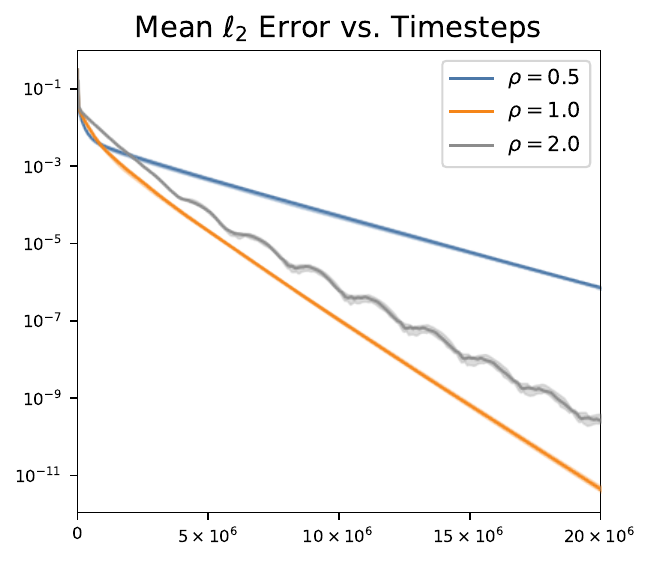}
        \caption{$\varepsilon=0$}
    \end{subfigure}\hfill
    \begin{subfigure}{0.32\textwidth}
        \centering
        \includegraphics[height=4.2cm, width=\linewidth]{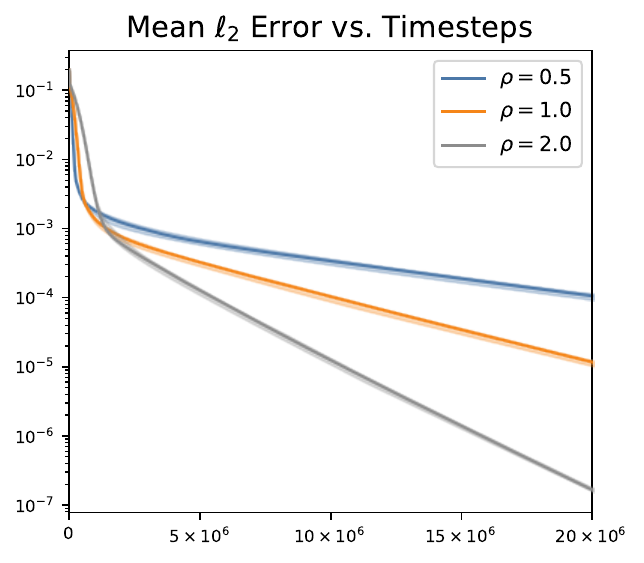}
        \caption{$\varepsilon=0.2$}
    \end{subfigure}\hfill
    \begin{subfigure}{0.32\textwidth}
        \centering
        \includegraphics[height=4.2cm, width=\linewidth]{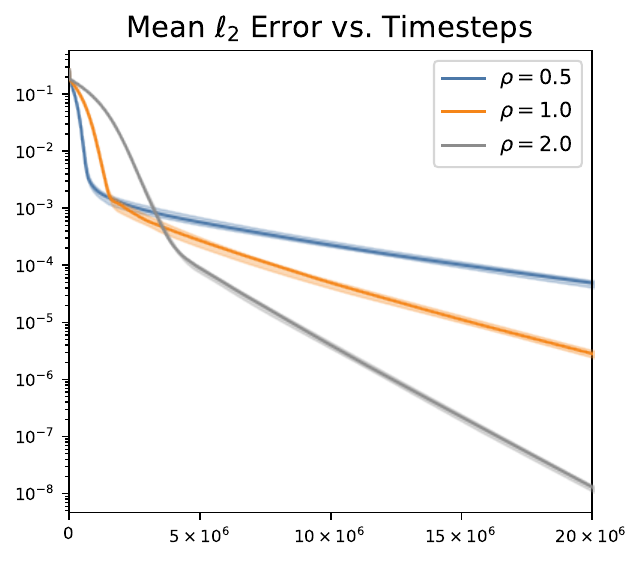}
        \caption{$\varepsilon=0.4$}
    \end{subfigure}
    \caption{Convergence of Goal-PD for the geometric median on the Wine Quality dataset for $\rho\in\{0.5,1,2\}$, under increasing corruption levels $\varepsilon$. All plots report the mean $\ell_2$ error versus timesteps.}
    \label{fig:wine_rho_geo}
\end{figure}

%% file: references.bib
@article{he20121,
  title={On the O(1/n) convergence rate of the Douglas--Rachford alternating direction method},
  author={He, Bingsheng and Yuan, Xiaoming},
  journal={SIAM Journal on Numerical Analysis},
  volume={50},
  number={2},
  pages={700--709},
  year={2012},
  publisher={SIAM}
}

@incollection{robinson2009some,
  title={Some continuity properties of polyhedral multifunctions},
  author={Robinson, Stephen M},
  booktitle={Mathematical Programming at Oberwolfach},
  pages={206--214},
  year={2009},
  publisher={Springer}
}

@book{beck2017first,
  title={First-order methods in optimization},
  author={Beck, Amir},
  year={2017},
  publisher={SIAM}
}

@article{koenker2001quantile,
  title={Quantile regression},
  author={Koenker, Roger and Hallock, Kevin F},
  journal={Journal of economic perspectives},
  volume={15},
  number={4},
  pages={143--156},
  year={2001},
  publisher={American Economic Association}
}

@misc{winequality186,
  author       = {Cortez, Paulo and Cerdeira, A. and Almeida, F. and Matos, T. and Reis, J.},
  title        = {Wine Quality},
  year         = {2009},
  howpublished = {UCI Machine Learning Repository},
  note         = {DOI: https://doi.org/10.24432/C56S3T}
}

@book{newman2018networks,
  title={Networks},
  author={Newman, Mark},
  year={2018},
  publisher={Oxford university press}
}

@article{anderson1985eigenvalues,
  title={Eigenvalues of the Laplacian of a graph},
  author={Anderson Jr, William N and Morley, Thomas D},
  journal={Linear and multilinear algebra},
  volume={18},
  number={2},
  pages={141--145},
  year={1985},
  publisher={Taylor \& Francis}
}

@article{zhang2011laplacian,
  title={The Laplacian eigenvalues of graphs: a survey},
  author={Zhang, Xiao-Dong},
  journal={arXiv preprint arXiv:1111.2897},
  year={2011}
}

@article{latafat2019new,
  title={A new randomized block-coordinate primal-dual proximal algorithm for distributed optimization},
  author={Latafat, Puya and Freris, Nikolaos M and Patrinos, Panagiotis},
  journal={IEEE Transactions on Automatic Control},
  volume={64},
  number={10},
  pages={4050--4065},
  year={2019},
  publisher={IEEE}
}

@book{rockafellar1997convex,
  title={Convex analysis},
  author={Rockafellar, R Tyrrell},
  volume={28},
  year={1997},
  publisher={Princeton university press}
}

@book{rockafellar1998variational,
  title={Variational analysis},
  author={Rockafellar, R Tyrrell and Wets, Roger JB},
  year={1998},
  publisher={Springer}
}

@article{combettes2015stochastic,
  title={Stochastic quasi-Fej{\'e}r block-coordinate fixed point iterations with random sweeping},
  author={Combettes, Patrick L and Pesquet, Jean-Christophe},
  journal={SIAM Journal on Optimization},
  volume={25},
  number={2},
  pages={1221--1248},
  year={2015},
  publisher={SIAM}
}

@article{vu2013splitting,
  title={A splitting algorithm for dual monotone inclusions involving cocoercive operators},
  author={V{\~u}, B{\`{\u a}}ng C{\^o}ng},
  journal={Advances in Computational Mathematics},
  volume={38},
  number={3},
  pages={667--681},
  year={2013},
  publisher={Springer}
}

@article{condat2013primal,
  title={A primal--dual splitting method for convex optimization involving Lipschitzian, proximable and linear composite terms},
  author={Condat, Laurent},
  journal={Journal of optimization theory and applications},
  volume={158},
  number={2},
  pages={460--479},
  year={2013},
  publisher={Springer}
}

@inproceedings{alacaoglu2020random,
  title={Random extrapolation for primal-dual coordinate descent},
  author={Alacaoglu, Ahmet and Fercoq, Olivier and Cevher, Volkan},
  booktitle={International conference on machine learning},
  pages={191--201},
  year={2020},
  organization={PMLR}
}

@incollection{robbins1971convergence,
  title={A convergence theorem for non negative almost supermartingales and some applications},
  author={Robbins, Herbert and Siegmund, David},
  booktitle={Optimizing methods in statistics},
  pages={233--257},
  year={1971},
  publisher={Elsevier}
}

@misc{wei2013o1kconvergenceasynchronousdistributed,
      title={On the O(1/k) Convergence of Asynchronous Distributed Alternating Direction Method of Multipliers}, 
      author={Ermin Wei and Asuman Ozdaglar},
      year={2013},
      eprint={1307.8254},
      archivePrefix={arXiv},
      primaryClass={math.OC},
      url={https://arxiv.org/abs/1307.8254}, 
}

@article{minsker2015geometric,
  title={Geometric median and robust estimation in Banach spaces},
  author={Minsker, Stanislav},
  journal={Bernoulli},
  year={2015}
}

@article{duchi2011dual,
  title={Dual averaging for distributed optimization: Convergence analysis and network scaling},
  author={Duchi, John C and Agarwal, Alekh and Wainwright, Martin J},
  journal={IEEE Transactions on Automatic control},
  volume={57},
  number={3},
  pages={592--606},
  year={2011},
  publisher={IEEE}
}

@article{ayadi2017outlier,
  title={Outlier detection approaches for wireless sensor networks: A survey},
  author={Ayadi, Aya and Ghorbel, Oussama and Obeid, Abdulfattah M and Abid, Mohamed},
  journal={Computer Networks},
  volume={129},
  pages={319--333},
  year={2017},
  publisher={Elsevier}
}

@book{huber2011robust,
  title={Robust statistics},
  author={Huber, Peter J and Ronchetti, Elvezio M},
  year={2011},
  publisher={John Wiley \& Sons}
}

@article{boyd2006randomized,
  title={Randomized gossip algorithms},
  author={Boyd, Stephen and Ghosh, Arpita and Prabhakar, Balaji and Shah, Devavrat},
  journal={IEEE transactions on information theory},
  volume={52},
  number={6},
  pages={2508--2530},
  year={2006},
  publisher={IEEE}
}

@article{van2025robust,
  title={Robust distributed estimation: Extending gossip algorithms to ranking and trimmed means},
  author={Van Elst, Anna and Colin, Igor and Cl{\'e}men{\c{c}}on, Stephan},
  journal={Advances in Neural Information Processing Systems},
  volume={38},
  pages={26493--26529},
  year={2026}
}

@article{shi2014linear,
  title={On the linear convergence of the ADMM in decentralized consensus optimization},
  author={Shi, Wei and Ling, Qing and Yuan, Kun and Wu, Gang and Yin, Wotao},
  journal={IEEE Transactions on Signal Processing},
  volume={62},
  number={7},
  pages={1750--1761},
  year={2014},
  publisher={IEEE}
}

@article{yang2022survey,
  title={A survey of ADMM variants for distributed optimization: Problems, algorithms and features},
  author={Yang, Yu and Guan, Xiaohong and Jia, Qing-Shan and Yu, Liang and Xu, Bolun and Spanos, Costas J},
  journal={arXiv preprint arXiv:2208.03700},
  year={2022}
}

@article{bianchi2015coordinate,
  title={A coordinate descent primal-dual algorithm and application to distributed asynchronous optimization},
  author={Bianchi, Pascal and Hachem, Walid and Iutzeler, Franck},
  journal={IEEE Transactions on Automatic Control},
  volume={61},
  number={10},
  pages={2947--2957},
  year={2015},
  publisher={IEEE}
}

@article{nedic2009distributed,
  title={Distributed subgradient methods for multi-agent optimization},
  author={Nedic, Angelia and Ozdaglar, Asuman},
  journal={IEEE Transactions on automatic control},
  volume={54},
  number={1},
  pages={48--61},
  year={2009},
  publisher={IEEE}
}

@article{shi2015proximal,
  title={A proximal gradient algorithm for decentralized composite optimization},
  author={Shi, Wei and Ling, Qing and Wu, Gang and Yin, Wotao},
  journal={IEEE Transactions on Signal Processing},
  volume={63},
  number={22},
  pages={6013--6023},
  year={2015},
  publisher={IEEE}
}

@inproceedings{wei2012distributed,
  title={Distributed alternating direction method of multipliers},
  author={Wei, Ermin and Ozdaglar, Asuman},
  booktitle={2012 IEEE 51st IEEE Conference on Decision and Control (CDC)},
  pages={5445--5450},
  year={2012},
  organization={IEEE}
}

@inproceedings{iutzeler2013asynchronous,
  title={Asynchronous distributed optimization using a randomized alternating direction method of multipliers},
  author={Iutzeler, Franck and Bianchi, Pascal and Ciblat, Philippe and Hachem, Walid},
  booktitle={52nd IEEE conference on decision and control},
  pages={3671--3676},
  year={2013},
  organization={IEEE}
}

@article{boyd2011distributed,
  title={Distributed optimization and statistical learning via the alternating direction method of multipliers},
  author={Boyd, Stephen and Parikh, Neal and Chu, Eric and Peleato, Borja and Eckstein, Jonathan and others},
  journal={Foundations and Trends{\textregistered} in Machine learning},
  volume={3},
  number={1},
  pages={1--122},
  year={2011},
  publisher={Now Publishers, Inc.}
}

@article{parikh2014proximal,
  title={Proximal algorithms},
  author={Parikh, Neal and Boyd, Stephen and others},
  journal={Foundations and trends{\textregistered} in Optimization},
  volume={1},
  number={3},
  pages={127--239},
  year={2014},
  publisher={Now Publishers, Inc.}
}

@book{wasserman2013all,
  title={All of statistics: a concise course in statistical inference},
  author={Wasserman, Larry},
  year={2013},
  publisher={Springer Science \& Business Media}
}

@article{iutzeler2017distributed,
  title={Distributed computation of quantiles via ADMM},
  author={Iutzeler, Franck},
  journal={IEEE Signal Processing Letters},
  volume={24},
  number={5},
  pages={619--623},
  year={2017},
  publisher={IEEE}
}
